\documentclass{article} % For LaTeX2e
\usepackage{iclr2027_conference,times}

\usepackage{amsmath,amsfonts,bm}

\def\eqref#1{equation~\ref{#1}}
\def\1{\bm{1}}

\DeclareMathAlphabet{\mathsfit}{\encodingdefault}{\sfdefault}{m}{sl}
\SetMathAlphabet{\mathsfit}{bold}{\encodingdefault}{\sfdefault}{bx}{n}

\usepackage{hyperref}
\usepackage{url}

\title{Hybrid Methods for Robust Tabular Data Imputation}

\author{Jinwei Li \\
Department of Data Science\\
Friedrich-Alexander-Universität Erlangen-Nürnberg\\
Erlangen, Germany\\
\texttt{jinwei.li@fau.de}
\And
Michelle Bruch \\
Department of Data Science\\
Friedrich-Alexander-Universität Erlangen-Nürnberg\\
Erlangen, Germany\\
\texttt{michelle.bruch@fau.de}
\AND
Daniel Tenbrinck \\
Department of Data Science\\
Friedrich-Alexander-Universität Erlangen-Nürnberg\\
Erlangen, Germany\\
\texttt{daniel.tenbrinck@fau.de}
}

\usepackage[utf8]{inputenc} % allow utf-8 input
\usepackage[T1]{fontenc}    % use 8-bit T1 fonts
\usepackage{hyperref}       % hyperlinks
\usepackage{url}            % simple URL typesetting
\usepackage{booktabs}       % professional-quality tables
\usepackage{amsfonts}       % blackboard math symbols
\usepackage{nicefrac}       % compact symbols for 1/2, etc.
\usepackage{microtype}      % microtypography
\usepackage{xcolor}         % colors
\usepackage{amsmath,amssymb,amsthm}

\newtheorem{theorem}{Theorem}

\usepackage{bm}
\usepackage{amsmath,amssymb,amsthm}
\usepackage{graphicx}
\usepackage{subcaption}
\usepackage{booktabs}
\usepackage{algorithm}
\usepackage{algpseudocode}
\usepackage{enumitem}
\newtheorem{Proposition}{Proposition}
\newtheorem{corollary}{Corollary}
\usepackage{wrapfig}
\usepackage[nameinlink,capitalise]{cleveref}
\iclrfinalcopy % Uncomment for camera-ready version, but NOT for submission.
\begin{document}

\maketitle

\begin{abstract}
Missing data are a fundamental challenge in statistical analysis and machine learning, as the choice of imputation method substantially impacts downstream inference.
In this work, we propose two hybrid imputation methods called NuclearForest and SoftForest, which combine nuclear-norm-based low-rank initialization using Singular Value Thresholding (SVT) and SoftImpute, respectively, with a non-iterative Random Forest refinement. For the SVT-based component, we further introduce an adaptive step-size rule, prove adaptive step-size bounds, and establish convergence for the corresponding zero-initialized iteration. The low-rank initialization provides a structured warm start that captures the global covariance patterns in the data, while the subsequent Random Forest step recovers residual nonlinear signals encoding local dependencies.
We conduct an extensive benchmark on diverse datasets from different application domains, comparing the proposed methods with seven established imputation methods under the Missing Completely at Random (MCAR), Missing at Random (MAR), and Missing Not at Random (MNAR) mechanisms across varying missingness rates. Our results demonstrate that NuclearForest and SoftForest match or exceed the imputation fidelity of state-of-the-art iterative methods such as MissForest, while significantly reducing computational cost. In particular, they achieve speedups of approximately \(5.81\times\) and \(9.52\times\) over MissForest by replacing iterative cycles with a single refinement step. Our approach effectively exploits the low-rank structure of real-world tabular data and accommodates mixed-type variables, providing an efficient and robust solution for data imputation in bioinformatics, economics, and beyond.
\end{abstract}

\section{Introduction}
\paragraph{Motivation}
Missing values are ubiquitous in real-world biomedical datasets, and data imputation is a critical prerequisite in the preprocessing pipeline. 
Consequently, the reliability of downstream analyses, including statistical testing and classification using machine learning methods, depends directly on the quality of the imputed data.
However, state-of-the-art Random-Forest-based imputation methods, such as MissForest \citep{stekhoven2012missforest}, require repeated fitting of predictive models until convergence, resulting in substantial computational cost.
This computational complexity limits the practicality of iterative Random-Forest-based imputation methods when imputation must be performed repeatedly across experimental conditions.
We introduce NuclearForest and SoftForest, two hybrid imputation methods that combine low-rank matrix completion with non-iterative Random Forest refinement, substantially reducing runtime while preserving high-quality imputations on tabular datasets.
\paragraph{Background}
Missing-data imputation can be approached from several complementary paradigms.
Classical theory distinguishes between missing completely at random (MCAR), missing at random (MAR), and missing not at random (MNAR), since the validity of an imputation strategy depends on the mechanism that generated the missing entries \citep{rubin1976inference}.
Practical imputation methods span simple univariate replacement, model-based multiple imputation such as MICE \citep{vanbuuren2011mice}, and non-parametric machine-learning approaches such as MissForest \citep{stekhoven2012missforest}. 
A complementary line of work exploits global correlation structures through low-rank matrix completion and nuclear-norm regularization, including Singular Value Thresholding (SVT) \citep{cai2010singular} and SoftImpute \citep{mazumder2010spectral}.

MissForest, proposed by \citet{stekhoven2012missforest}, is a random-forest-based imputation method that has been shown to perform well on biological datasets, outperforming imputation methods such as kNN by modeling complex interactions and nonlinear relationships.
Subsequent studies have shown that random-forest-based imputation can achieve strong imputation accuracy in datasets with nonlinear relationships and mixed variable types, with particularly favorable performance reported under MCAR/MAR missingness in metabolomics benchmarks \citep{tang2017random,wei2018missing}.
We therefore choose random forests for the refinement step because our target setting is tabular imputation rather than representation learning from large unstructured data. For medium-sized tabular datasets, Random-Forest-based models remain state-of-the-art and outperform neural-network baselines even under extensive hyperparameter-search budgets \citep{grinsztajn2022tree}. 
Moreover, random forests are a standard non-parametric regression tool and often perform reasonably well under default hyperparameter settings \citep{breiman2001random,probst2019hyperparameters}.
For missing-data imputation, this property is particularly valuable: random-forest-based models can capture nonlinear effects and feature interactions while accommodating both numerical and categorical variables, which motivates their use in MissForest \citep{stekhoven2012missforest}.
However, MissForest often is computationally expensive, as it repeatedly fits variable-wise random forests until convergence.
The simple univariate initialization used before the iterative updates may provide a limited starting point, since it does not explicitly impose low-rank or other global structural assumptions on the completed matrix.

%The simple univariate initialization used before the iterative updates may provide an oversimplified starting point, making subsequent iterative refinement more challenging.

%Since SVT is designed to minimize the \textbf{nuclear} norm, the hybrid method combining SVT and random \textbf{forest} is called \textbf{NuclearForest}. Moreover, the combination of \textbf{Soft}Impute with random \textbf{forest} is named \textbf{SoftForest}. 

The SVT algorithm~\citep{cai2010singular} is a classical first-order method for nuclear-norm-based matrix completion.
Its appeal lies in a simple iterative structure, where each iteration alternates between singular-value shrinkage and a residual update.
However, because standard SVT relies on fixed algorithmic parameters, including the singular-value threshold and the step size, its practical convergence speed can be sensitive to these choices and may require many iterations when they are not well matched to the data.
One response in the literature is to adapt the threshold during the iteration \citep{zarmehi2017adaptive}.

\paragraph{Approximate low-rank structure in real-world datasets.}
Many real-world datasets exhibit an approximately low-rank structure, as their variation is often governed by a relatively small number of latent factors.
This is particularly relevant for omics data, here referring to tabular measurements of molecular abundances, where correlated biological pathways and shared regulatory mechanisms induce dependencies among measured features \citep{hilafu2020sparse}.
Therefore, methods based on low-rank matrix completion provide a natural framework for omics data imputation by exploiting latent structure, but they may not fully capture nonlinear dependencies in the data.

\paragraph{Missingness mechanisms in MS-omics data.} Mass-spectrometry (MS) omics data commonly contain missing values, and the missingness mechanism can affect imputation performance and downstream analysis \citep{sethi2016omics,wei2018missing}.
%To obtain ground truth for an extensive benchmark of imputation methods, we generate missing values in the metabolomics dataset under three mechanisms: MCAR, MAR, and MNAR. Since MCAR and MAR are often difficult to distinguish in MS-based metabolomics, previous benchmarks evaluate them jointly using random masking. We follow this setting and uniformly mask observed entries to generate MCAR/MAR missingness \citep{wei2018missing}.
Using the metabolomics dataset from \citet{wei2018missing}, we benchmark imputation methods under two artificial missingness settings: a joint MCAR/MAR condition and a separate MNAR condition. Because MCAR and MAR are often difficult to distinguish in MS-based metabolomics, the joint MCAR/MAR setting follows the benchmark design of \citet{wei2018missing} and is generated by uniformly masking observed entries.
In an additional experiment using a housing dataset \citep{housing_price_kaggle}, MCAR and MAR are evaluated separately to further assess the behavior of the imputation methods under distinct missingness mechanisms.
MNAR is modeled as left-censored missingness, where low-abundance features fall below the limit of detection or quantification.
Accordingly, we generate MNAR missingness by selecting features and removing observations below variable-specific quantile cutoffs \citep{wei2018missing}.

\paragraph{Benchmark methods and evaluation metrics.} In this paper, real-world data from different applications will be imputed and evaluated using nine different imputation methods: MissForest, kNN, Mean, Median, Half-min, SVT, SoftImpute, and our proposed methods NuclearForest and SoftForest.
We evaluate imputation quality from multiple complementary perspectives, covering normalized root mean squared error (NRMSE), global sample structure, univariate group-level signals, distributional similarity, and downstream predictive utility.
Specifically, we use NRMSE and an NRMSE-based sum of ranks (SOR) score for masked-entry reconstruction; PCA/PLS Procrustes analysis for structural preservation in reduced-dimensional spaces; Student's \(t\)-test followed by Pearson correlation analysis for preservation of univariate group differences; and Gower's distance together with downstream predictive \(R^2\) degradation to assess distributional and task-level effects.

\paragraph{Contributions}

\begin{enumerate}[leftmargin=*,itemsep=2pt]

\item \textbf{Adapted SVT.} We introduce an efficient adaptive SVT variant
(Algorithm~\ref{alg:svt-me}) for imputation that replaces zero
initialization with column-mean warm starting and modulates the step size according to the relative observed-entry reconstruction error, rather than relying on a fixed step size \(\delta\). 
Specifically,
\(\delta_k\) is contracted by a factor of \(0.9\) when the error increases
and expanded by a factor of \(1.05\), up to \(2p\), when the error decreases,
where \(p = |\Omega|/(n_1n_2)\) is the observed-entry ratio. Under a finite-contraction condition, the adaptive step sizes remain bounded away from zero and satisfy \(\delta_k<2\). Proposition~\ref{prop:1} proves these bounds, while Corollary~\ref{cor:adaptive-svt-zero-init} gives convergence and an \(O(T^{-1})\) rate for the corresponding zero-initialized iteration. In addition, an ablation study separately evaluates column-mean warm starting and adaptive step-size modulation, indicating that both components contribute to improved imputation quality. See Appendix~\ref{app:ablation}.

\item \textbf{Hybrid imputation algorithms.}
We propose NuclearForest and SoftForest, two hybrid imputation methods that combine low-rank initialization with a single subsequent random-forest refinement step. NuclearForest uses the proposed adaptive SVT variant as its initialization module, whereas SoftForest uses SoftImpute as the corresponding low-rank initializer. Across the evaluated real-data settings, both methods achieve competitive imputation quality relative to iterative random-forest baselines while reducing runtime by approximately \(5.81\times\) for NuclearForest and \(9.52\times\) for SoftForest in one representative experiment.

\item \textbf{Comprehensive real-data benchmark.}
We provide a broad empirical evaluation of nine imputation methods on metabolomics and housing datasets under MCAR, MAR, and MNAR missingness mechanisms. The benchmark compares methods across reconstruction, structural, statistical, and downstream predictive metrics, allowing us to assess not only masked-entry reconstruction error but also the preservation of low-dimensional sample structure, group-level signals, and predictive utility.

\end{enumerate}

%\paragraph{the OLD version of Background}
%In many real-world datasets, variables exhibit underlying latent dependencies. Our methods utilize this implicit low-rank structure to achieve more efficient data imputation. To balance both global structure and local accuracy, we propose a hybrid, two-step imputation framework that combines low-rank matrix completion (Singular Value Thresholding (SVT) \citep{cai2010singular} / SoftImpute \citep{hastie2014matrix}) with a single feature-specific random forest refinement, capturing the nonlinear structure in the dataset. 

\section{Methodology}

\paragraph{Matrix completion via nuclear norm minimization.}

Let \(M\in\mathbb{R}^{n_1\times n_2}\) be a partially observed matrix, and an entry of \(M\) is indexed by a pair \((i,j)\), where
\(i\in\{1,\ldots,n_1\}\) is the row index and
\(j\in\{1,\ldots,n_2\}\) is the column index. Let \(\Omega \subseteq \{1,\ldots,n_1\}\times\{1,\ldots,n_2\}\) denote the index set of observed entries. The sampling operator \(\mathcal{P}_{\Omega}\) is defined entrywise by \(
\bigl(\mathcal{P}_{\Omega}(X)\bigr)_{ij}=
\begin{cases}
X_{ij}, & (i,j)\in\Omega,\\
0, & (i,j)\notin\Omega.
\end{cases}\).
The matrix completion problem aims to recover a complete matrix \(X\) that agrees with the partially observed matrix \(M\) on the observed index set \(\Omega\). Since infinitely many matrices can satisfy these observed constraints, recovery is ill posed without additional a priori assumptions about the expected solution. A standard assumption is that the underlying complete matrix has a low-rank structure, meaning that much of its variation can be represented in a low-dimensional linear subspace.
This assumption leads to the following rank-minimization formulation:
\begin{equation}
\label{eq:matrix_completion_np}
\min_{X \in  \mathbb{R}^{n_1 \times n_2}}\, \operatorname{rank}(X)
\qquad
\text{s. t.}
\qquad
\mathcal{P}_\Omega(X)=\mathcal{P}_\Omega(M).
\end{equation}
Since rank minimization is nonconvex and NP-hard \citep{rechtGuaranteedMinimumrankSolutions2010}, the optimization problem \eqref{eq:matrix_completion_np} is commonly replaced by a convex relaxation \citep{candes2012exact} 
\begin{equation}
\label{eq:matrix_completion_nuclear}
\min_{X \in \mathbb{R}^{n_1 \times n_2}} \; \|X\|_*
\qquad
\text{s. t.}
\qquad
\mathcal{P}_\Omega(X)=\mathcal{P}_\Omega(M),
\end{equation}
where \(\|X\|_*=\sum_i \sigma_i(X)\) denotes the nuclear norm, i.e.,
the sum of the singular values of \(X\). This leads to a convex formulation of the matrix completion problem based on nuclear norm minimization.
Solutions to this relaxed problem can be efficiently computed via singular value thresholding methods.

\paragraph{Singular Value Thresholding (SVT).}
The SVT algorithm was introduced by \citet{cai2010singular} as a first-order method for solving nuclear-norm minimization problems. 
In the matrix completion setting, instead of treating the constrained nuclear-norm problem \eqref{eq:matrix_completion_nuclear} directly, the method is derived from the closely related regularized problem
\[
\min_X \;
\tau \|X\|_* + \frac12 \|X\|_F^2
\qquad
\text{s. t.}
\qquad
\mathcal{P}_\Omega(X)=\mathcal{P}_\Omega(M),
\]
where \(\tau>0\) is a threshold parameter and \(\|\cdot\|_F\) denotes the Frobenius norm.
For large values of \(\tau\), the minimizer of this problem approaches the minimum-Frobenius norm solution of the corresponding nuclear-norm minimization problem. To derive the SVT algorithm, one introduces the following Lagrangian
\(
\mathcal{L}(X,Y)
=
\tau\|X\|_* + \frac12\|X\|_F^2
+
\langle Y,\mathcal{P}_\Omega(M-X)\rangle,
\)
where \(Y\) is the dual variable associated with the equality constraint. 
Following the Uzawa interpretation of \citet{cai2010singular}, this yields an iterative scheme that alternates between minimizing the Lagrangian with respect to the primal variable \(X\) and performing a dual gradient step in the dual variable \(Y\). Starting from \(Y^0=0\), the SVT iteration is given by
\[
X^k = \mathcal{S}_\tau(Y^{k-1}),
\qquad
Y^k = Y^{k-1} + \delta_k\,\mathcal{P}_\Omega(M-X^k),
\]
where \(\delta_k>0\) denotes the step-size parameter. In the classical
SVT implementation of \citet{cai2010singular}, this step size is typically
chosen to be constant, i.e., \(\delta_k=\delta\). The operator \(\mathcal{S}_\tau\) denotes the proximal operator of the nuclear norm given by
\(
\mathcal{S}_\tau(Y)
:=
\operatorname{prox}_{\tau\|\cdot\|_*}(Y)
=
\arg\min_X
\left\{
\tau\|X\|_*
+
\frac12\|X-Y\|_F^2
\right\}.
\)
If \(Y = U\Sigma V^T\) is a singular value decomposition of \(Y\), with \(\Sigma=\operatorname{diag}(\sigma_1,\dots,\sigma_r),\) where \(\sigma_1,\dots,\sigma_r\) are the singular values of \(Y\), then \(\mathcal{S}_\tau\) acts by soft-thresholding the singular values through: 
\begin{equation}
\label{eq:svt}
\mathcal{S}_\tau(Y)
=
U\,\operatorname{diag}\big((\sigma_i-\tau)^+\big)\,V^T,
\qquad
(a)^+ := \max(a,0).
\end{equation}
Hence, each SVT iteration alternates between two operations: singular value thresholding, which promotes a low-rank structure, and a dual update on the observed entries, which drives the reconstruction toward consistency with the available data.

\paragraph{Improved SVT variant.}
In this work, we propose an improved variant of the classical SVT method with two key modifications.
First, instead of initializing the missing entries as zero, we use a column-mean warm start, which provides an inexpensive data-dependent initialization that preserves the observed entries and places the initial completion on the empirical feature scale, thereby reducing the burden on subsequent low-rank recovery and random-forest refinement.
Second, we replace the empirical fixed step size
\(\delta=1.2/p\) (where \(p=\frac{|\Omega|}{n_1n_2}\) is the observed-entry rate of the matrix) used in the original SVT implementation with an adaptive rule: \(\delta_0=1.2p\), followed by multiplicative expansion or contraction according to the observed-entry residual, with a cap at \(\delta_{\max}=2\). 
Under the finite-contraction assumption, the adaptive step sizes remain bounded away from zero. In our experiments, this condition is empirically satisfied, as the residual-based rule does not trigger contraction steps on the evaluated real datasets.
Thus, the adaptive step sizes remain within the standard SVT admissible range \citep{cai2010singular}. Proposition~\ref{prop:1} gives the step-size bounds, and Corollary~\ref{cor:adaptive-svt-zero-init} gives the zero-initialized convergence result.
Instead of using an adaptive threshold as proposed by Zarmehi and Marvasti~\citep{zarmehi2017adaptive}, we keep \(\tau\) fixed, matching the original SVT parameter settings. In all numerical experiments, we compute the SVT updates using a full singular value decomposition.

\paragraph{SoftImpute.}

In contrast to the SVT algorithm, SoftImpute, introduced by \citet{mazumder2010spectral}, approaches matrix completion through the following penalized nuclear norm problem
\begin{equation}
\label{eq:softimpute}
\min_X \;
\frac12 \|\mathcal{P}_\Omega(M-X)\|_F^2 + \lambda \|X\|_*,
\end{equation}
where \(\lambda>0\) is a fixed regularization parameter.
This objective is convex: the first term penalizes reconstruction error on the observed entries, while the nuclear-norm term promotes low-rank structure through the same convex surrogate of the rank function used in nuclear-norm matrix completion.
Therefore, SoftImpute is closely related to the SVT formulation, but replaces the hard constraint
\(\mathcal{P}_\Omega(X)=\mathcal{P}_\Omega(M)\)
with a penalized reconstruction term on the observed entries.

To solve this minimization problem, SoftImpute proceeds as follows.
Given a current iterate \(X^k\), the method defines
\[
W^k
=
\mathcal{P}_\Omega(M) + \mathcal{P}_\Omega^\perp(X^k),
\]
where \(\mathcal{P}_\Omega^\perp\) denotes the complementary projection onto the unobserved entries.
This leads to the following equivalent reformulation:
\[
W^k
=
X^k + \mathcal{P}_\Omega(M-X^k).
\]
Thus, \(W^k\) coincides with the observed data on \(\Omega\) and with the current iterate on the complementary index set.
The next iterate is then defined by
\(
X^{k+1} = S_{\lambda}(W^k),
\)
where \(S_\lambda\) denotes the singular value soft-thresholding operator introduced in \eqref{eq:svt}.
Hence, each iteration of SoftImpute consists of an imputation step in which the missing entries are filled-in using the current iterate, followed by a singular value thresholding step that promotes a low-rank structure.
The parameter \(\lambda\) determines the amount of shrinkage applied to the singular values.
In contrast to SVT, SoftImpute is formulated entirely in terms of the primal variable.
Let \(X^\star\) denote a solution of the penalized nuclear norm minimization problem.
Then the objective values generated by SoftImpute decrease monotonically and converge to the optimal value.
Moreover, the convergence rate is of order \(\mathcal{O}(1/k)\).

\paragraph{Random forest regression.}
As a random-forest-based iterative imputation baseline, we use a MissForest-style procedure in which each incomplete column is treated as a supervised regression problem.
Let \(X^{(t)}\in\mathbb{R}^{n_1\times n_2}\) denote the completed matrix at iteration \(t\).
At each iteration, incomplete columns are visited sequentially. For an incomplete column \(j\), a random forest regressor \(f_j^{(t)}\) is trained on the observed rows
\(\{i:(i,j)\in\Omega\}\), using the current imputed predictor vectors \(X^{(t)}_{i,-j}\) and the observed responses \(M_{ij}\). 
Here, \(X^{(t)}_{i,-j}\) denotes the row-\(i\) predictor vector formed by all columns except the target column \(j\).
The trained model is then used to update the missing entries in column \(j\):
\(
X_{ij}^{(t+1)} =
f_j^{(t)}\!\left(X^{(t)}_{i,-j}\right), (i,j)\notin\Omega,
\)
while observed entries are kept fixed.
This round-robin procedure is repeated until either a maximum number of iterations is reached or the change between consecutive completed matrices falls below a predefined tolerance.

%The R package \texttt{missForest} follows the original \citet{stekhoven2012missforest} stopping rule: after each iteration, the difference between consecutive imputations is monitored, and for continuous-only data the iteration stops once this difference increases for the first time; for mixed-type data, the corresponding continuous and categorical differences are considered separately. When this stopping rule is triggered, \texttt{missForest} returns the previous imputation matrix. 

\paragraph{NuclearForest and SoftForest.}
We propose two hybrid imputation methods that combine low-rank completion with nonlinear random-forest refinement.
NuclearForest first obtains a low-rank initialization using SVT, whereas SoftForest uses SoftImpute.
This design changes the role of the random forest compared with fully iterative random-forest imputation.
Rather than starting from a simple univariate initialization and repeatedly updating incomplete features until convergence, as in MissForest-style methods, our methods first construct a structured low-rank completion and then apply a single random-forest refinement pass.
For each incomplete feature, a random-forest regressor is trained on the rows where that feature is observed, using the corresponding matrix completed by SVT or SoftImpute as the predictor space.
The fitted model is then applied once to the missing entries of that feature, while the originally observed entries remain fixed.
The single-pass design is deliberate: the low-rank stage already captures global correlation structure, so the random forest is used primarily as a nonlinear residual corrector that refines the low-rank estimate by modeling interaction effects and nonlinear dependencies.
Additional refinement passes are possible and may slightly improve accuracy in some settings, but they increase computational cost and provide only marginal further gains in our experiments.
We therefore use the single-pass random-forest refinement to emphasize the accuracy--efficiency trade-off of the proposed hybrid framework.

%\textbf{state in the abstract: wanna have global structure, then local nonlinear by regression model. save time+combine the global with local. USE KEY WORDS HYBRID. no formular is needed here. for this reason we call it nuclearforest and softforest.}

\section{Experiments and results}

%We will use uniform random masking for MCAR \citep{wei2018missing}. MAR occurs rarely, but we will show the result of MAR in the housing data set. It can be implemented by using a logistic model on another random column excluding the missing column itself. \circled{3} MNAR missing values were introduced by quantile cut-off for the full dataset \citep{lazar2016accounting}\citep{Webb-Robertson2015}. (THESE CITATIONS USED PROTEOMICS DATASET, MIGHT BE INAPPROPRIATE FOR OUR SETTING)

In our numerical experiments, we benchmark nine imputation methods on two datasets under three missingness mechanisms.
For each setting, missing entries are regenerated independently over 10 repeated experimental runs, ensuring different missingness masks across runs while preserving reproducibility. Neural-network-based imputation methods are not included as primary baselines because the focus of this study is on efficient classical and hybrid tabular imputers. This choice is consistent with recent tabular-learning benchmarks showing that tree-based models remain highly competitive and can outperform neural-network baselines on many medium-sized tabular datasets \citep{grinsztajn2022tree}. 
On the publicly available metabolomics dataset \citep{wei2018missing}, we use the four evaluation metrics adopted by \citet{wei2018missing}.
Since MAR is commonly approximated by MCAR in metabolomics benchmarks, we evaluate MCAR/MAR via uniform random masking and MNAR via left-censored missingness.
On the housing dataset, we use the same four metrics together with three additional task-specific metrics. For the MAR setting, we generated missing values using a conditional logistic masking mechanism, where the missingness probability of each variable was modeled as a function of other variables rather than the variable's own value \citep{schouten2018generating}.
Runtime is reported as wall-clock time in seconds and summarized as mean ± standard deviation (std) over repeated runs under the same local computational environment.
All experiments were run on macOS 14.1 using an Apple M3 Pro processor with 11 CPU cores and 18 GB of RAM. The code of this benchmark experiment as well as of the two proposed imputation algorithms will be publicly released upon acceptance to ensure reproducibility.

\paragraph{Ablations and SVD implementation in the SVT.}

\begin{wrapfigure}[10]{r}{0.35\textwidth}
    \centering
    
    \includegraphics[width=0.30\textwidth]{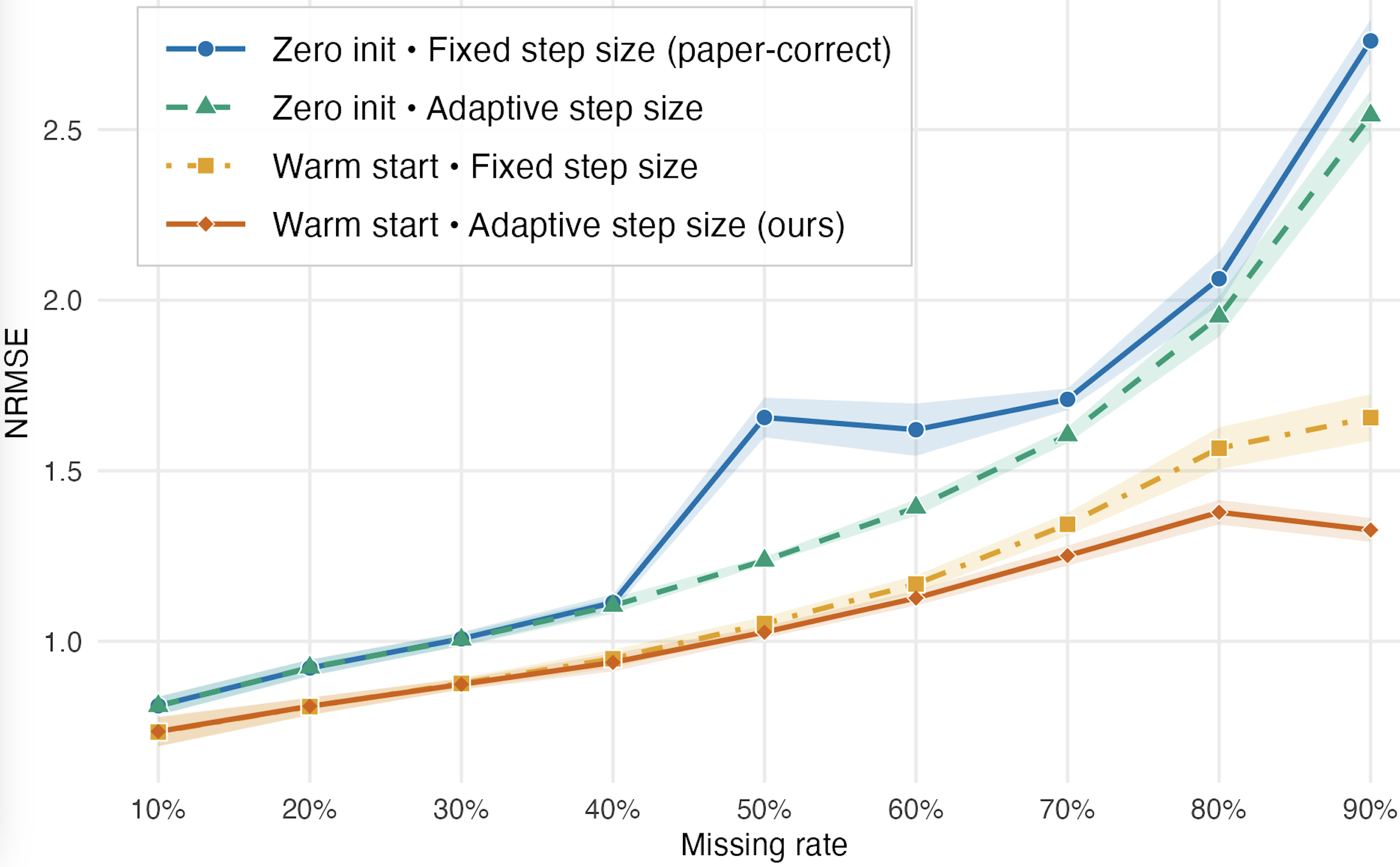}
    \caption{Ablation study.}
    \label{fig:wrapped_result}
    
\end{wrapfigure}

We evaluate the contribution of each design choice by isolating
(i) zero initialization versus column-mean warm start and
(ii) fixed versus adaptive step size.
Across the four metrics
(NRMSE, PCA Procrustes distance, \(p\)-value correlation, and PLS Procrustes
distance), both components improve imputation quality, with the full method
performing best overall under MCAR/MAR missingness.
\cref{fig:wrapped_result} summarizes the NRMSE ablation, while the
complete ablation results are reported in Appendix~\ref{app:ablation}.
We also compare dense full SVD with partial SVD.
The original SVT implementation of \citet{cai2010singular} exploits the
sparsity of the dual iterate \(Y^k\) and computes only the dominant
singular values needed by the shrinkage operator, using PROPACK/Lanczos
bidiagonalization.
This is appropriate for large sparse matrix completion
problems.
In contrast, our experimental matrices are dense after initialization. On the metabolomics dataset (\(198\times131\)),
the full-SVD implementation achieved a \(2.89\times\) speedup over the
partial-SVD SVT variant, with indistinguishable imputation error (Appendix~\ref{app:partial-full-svd} ).
For this reason, we use the full SVD in all experiments.
For large sparse datasets, one may instead use partial SVD approximations, which are commonly employed to reduce the cost of singular-value shrinkage when only the leading singular components are needed.
\begin{figure}[h]
    \centering

    \begin{subfigure}[b]{0.48\textwidth}
        \centering
        \includegraphics[width=\linewidth,height=0.27\textheight,keepaspectratio]{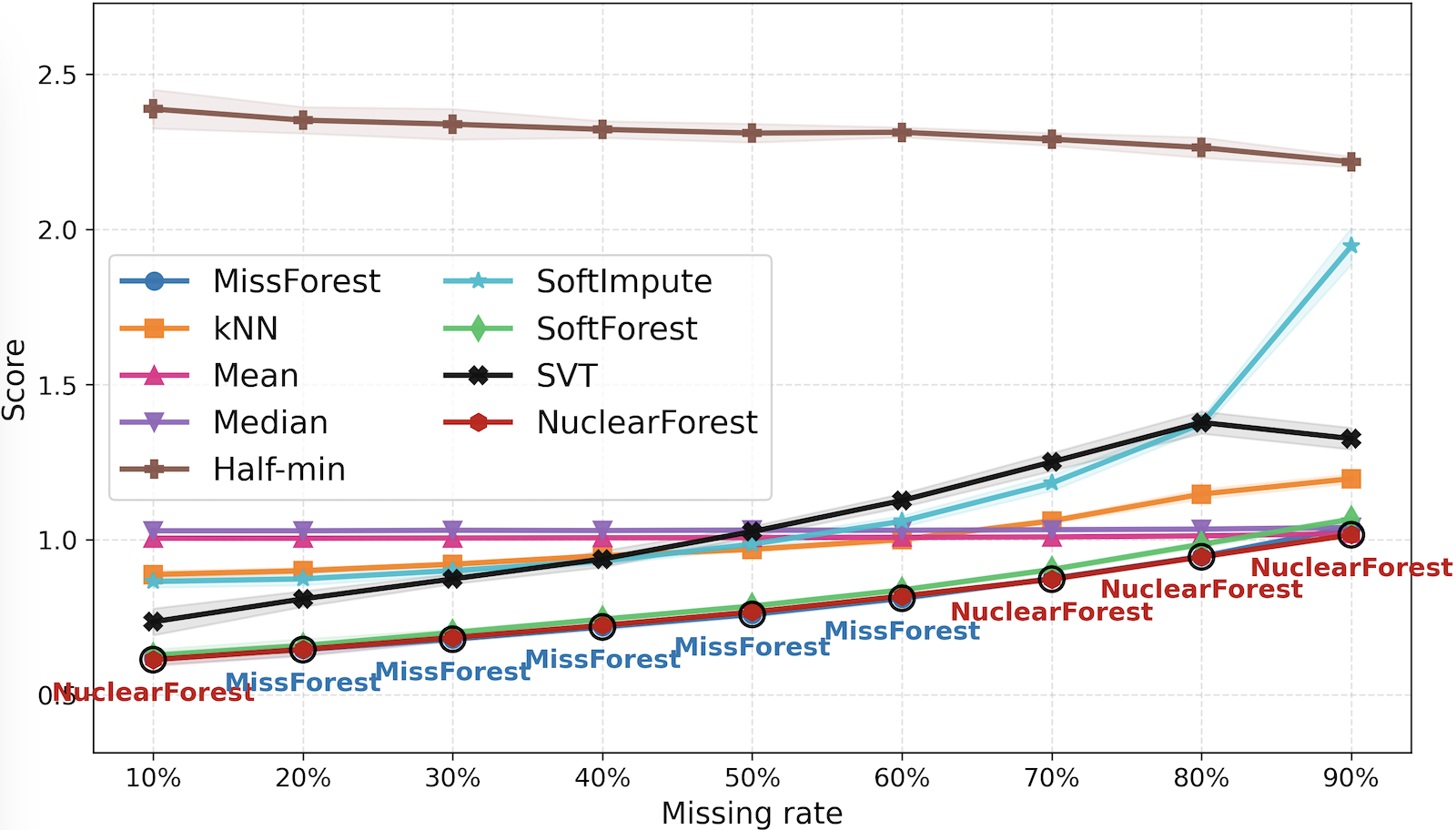}
        \caption{NRMSE}
        \label{fig:sub11}
    \end{subfigure}
    \hfill
    \begin{subfigure}[b]{0.48\textwidth}
        \centering
        \includegraphics[width=\linewidth,height=0.27\textheight,keepaspectratio]{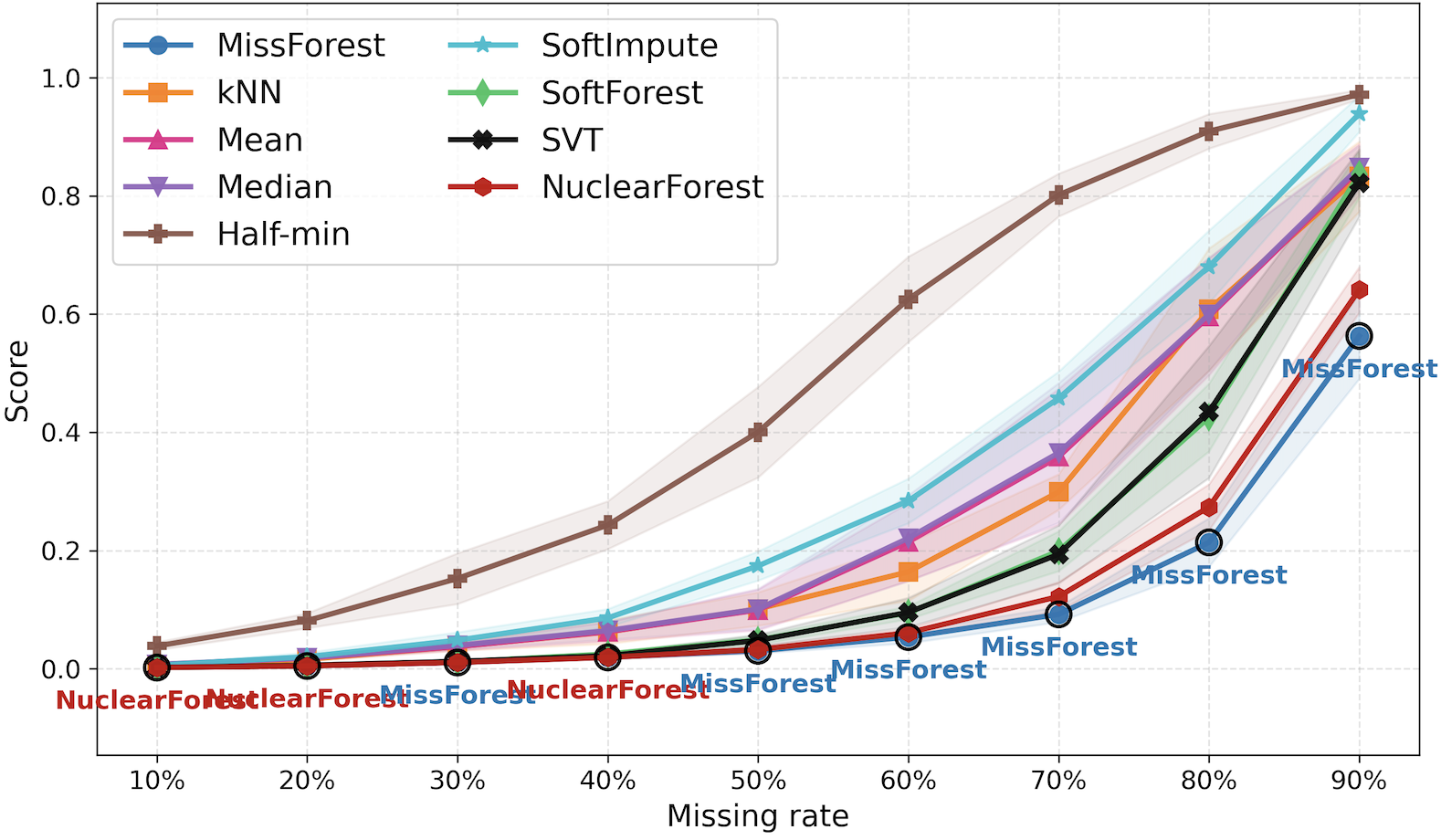}
        \caption{PCA Procrustes distance}
        \label{fig:sub12}
    \end{subfigure}

    \centering
    \begin{subfigure}[b]{0.48\textwidth}
        \centering
        \includegraphics[width=\linewidth,height=0.27\textheight,keepaspectratio]{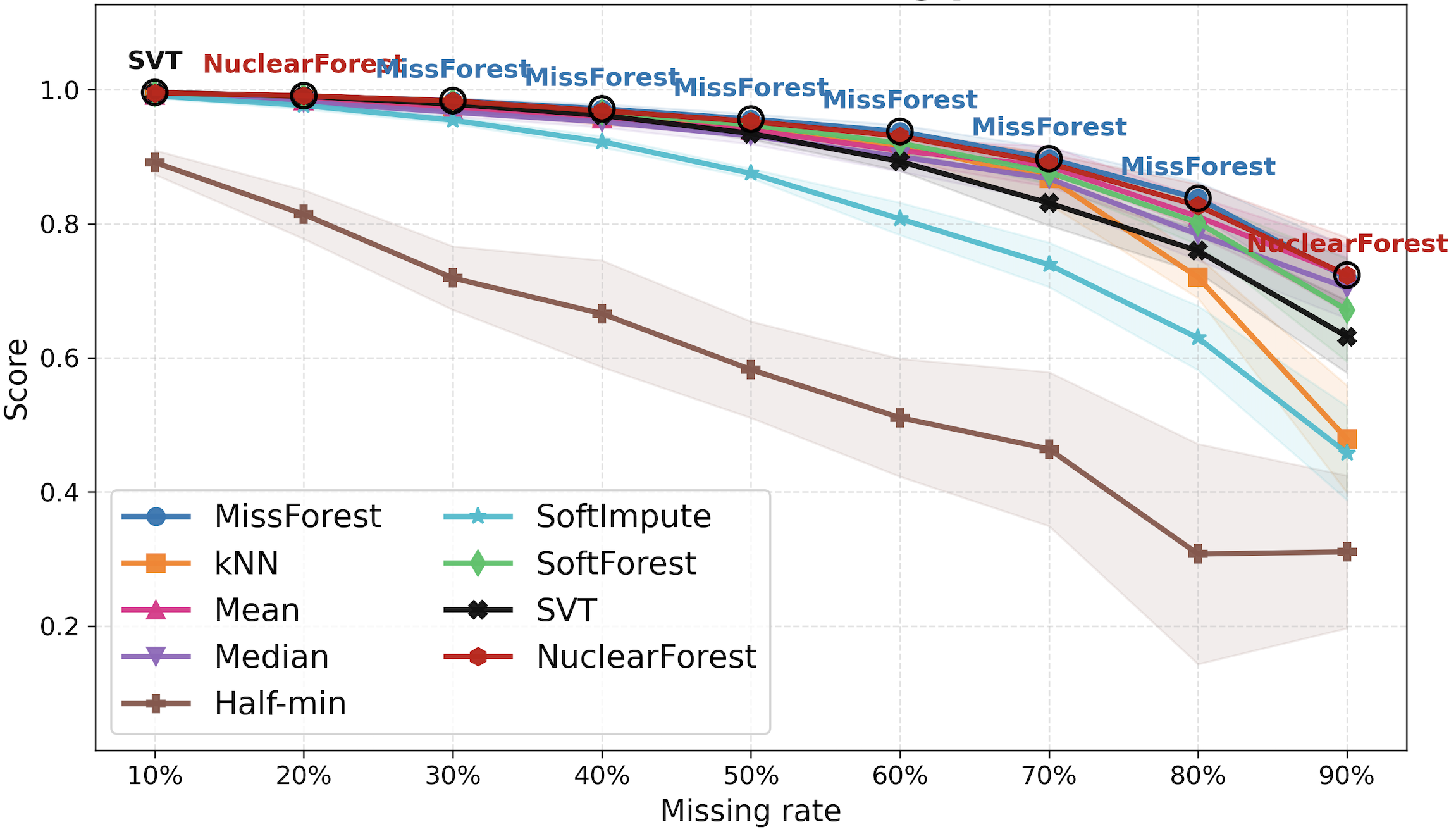}
        \caption{Pearson log-\(p\) correlation}
        \label{fig:sub13}
    \end{subfigure}
    \hfill
    \begin{subfigure}[b]{0.48\textwidth}
        \centering
        \includegraphics[width=\linewidth,height=0.27\textheight,keepaspectratio]{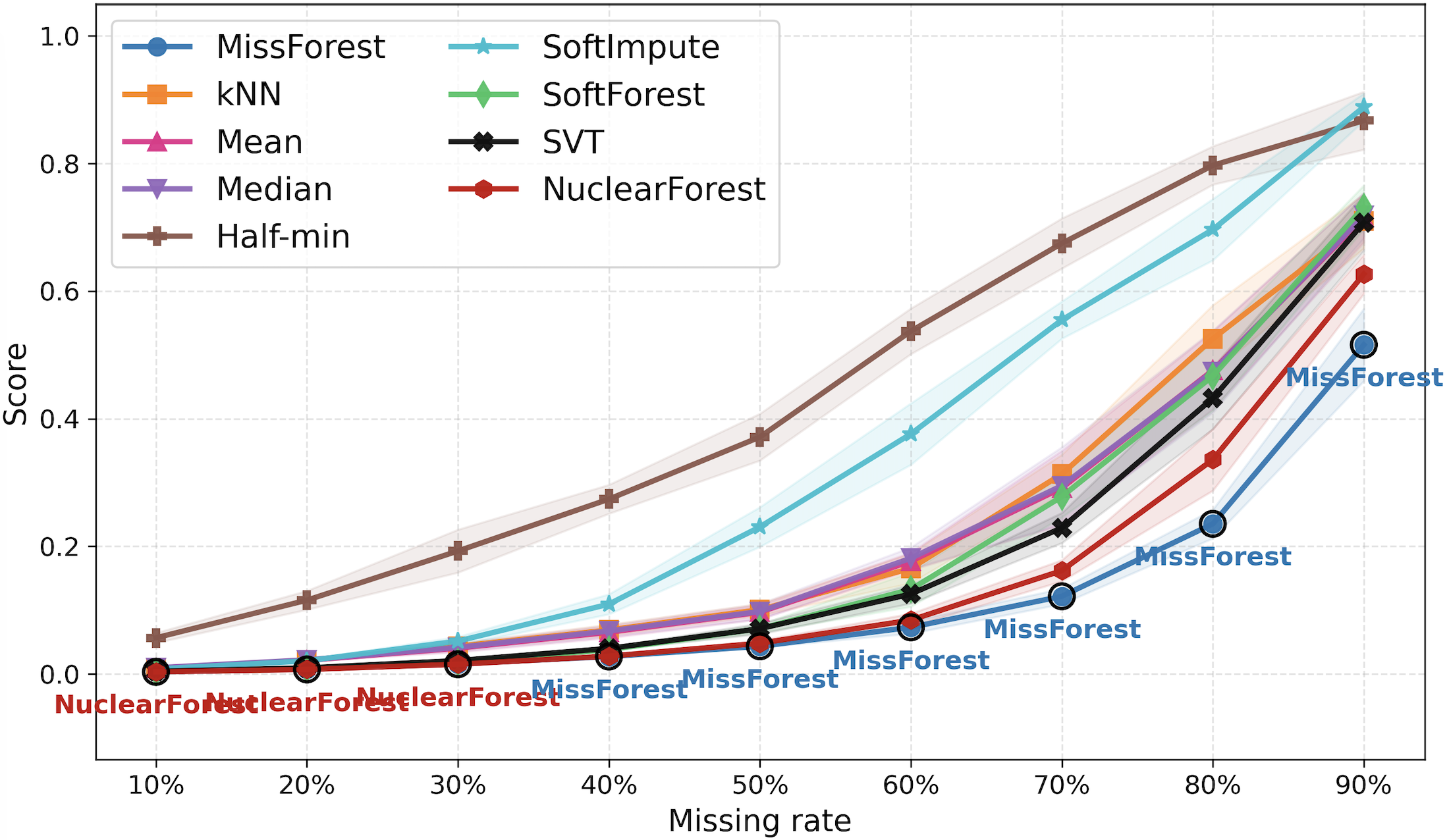}
        \caption{PLS Procrustes distance}
        \label{fig:sub14}
    \end{subfigure}

    \caption{Comparison of imputation methods on the metabolomics dataset under MCAR/MAR.}
    \label{fig:naturemcar}
\end{figure}

\paragraph{Metabolomics dataset under MCAR/MAR missingness.}
NRMSE is computed on masked entries after column-wise \(z\)-score standardization of both the ground-truth and imputed matrices, using means and standard deviations estimated from the complete ground-truth matrix. This transformation enables an unbiased comparison when using NRMSE by standardizing variable scales, preventing variables with larger values from dominating the evaluation \citep{wei2018missing}.
As shown in \cref{fig:sub11}, the imputation quality of MissForest, NuclearForest, and SoftForest is very close to each other with respect to the NRMSE metric.
As can be observed, MissForest performs slightly better for missing rates between 20\% and 60\%.
Next, PCA is performed using the first two principal components, as they capture the greatest variance in the data \citep{wei2018missing}.
Moreover, the symmetric Procrustes sum of squared errors is computed to compare the distribution.
In \cref{fig:sub12}, NuclearForest yields the lowest PCA Procrustes distance at 10\%, 20\%, and 40\% missingness.
Following \citet{wei2018missing}, we evaluate whether imputation preserves univariate group-difference signals by conducting Welch's two-sample \(t\)-tests for each variable between groups in the complete and imputed data, and then computing the Pearson correlation between the resulting log-transformed \(p\)-values.
As shown in \cref{fig:sub13}, NuclearForest achieves the highest Pearson log-\(p\) correlation values at 20\% and 90\% missingness and remains competitive across the remaining missing rates, where MissForest obtains the best results.
PLS Procrustes analysis is used to quantify the structural distortion.
NuclearForest achieves the lowest PLS Procrustes distance for missing rates between 10\% and 30\% in \cref{fig:sub14}. As shown in \cref{tab:runtime_resultsbio}, the hybrid methods (SoftForest and NuclearForest) remain significantly faster, providing \textasciitilde9.52\texttimes{} and \textasciitilde5.81\texttimes{} speedups, respectively.
Further runtime analysis is provided in Appendix~\ref{app:runtime_scaling}.

\begin{table}[h]
\centering

\caption{Runtime comparison for metabolomics dataset, reported in seconds as mean ± std.}

\label{tab:runtime_resultsbio}
\begin{tabular}{lccc}
\toprule
 Mechanism & MissForest

 & SoftForest 

 & NuclearForest \\
\midrule
MCAR/MAR
& 77.787 ± 21.426 

& 8.171 ± 1.931 

& 13.389 ± 2.365 \\

MNAR
& 89.438 ± 22.845 

& 4.344 ± 2.137 

& 9.797 ± 2.381 \\
\bottomrule
\end{tabular}
\end{table}

\begin{figure}[H]
    \centering

    \begin{subfigure}[b]{0.49\textwidth}
        \centering
        \includegraphics[width=\textwidth]{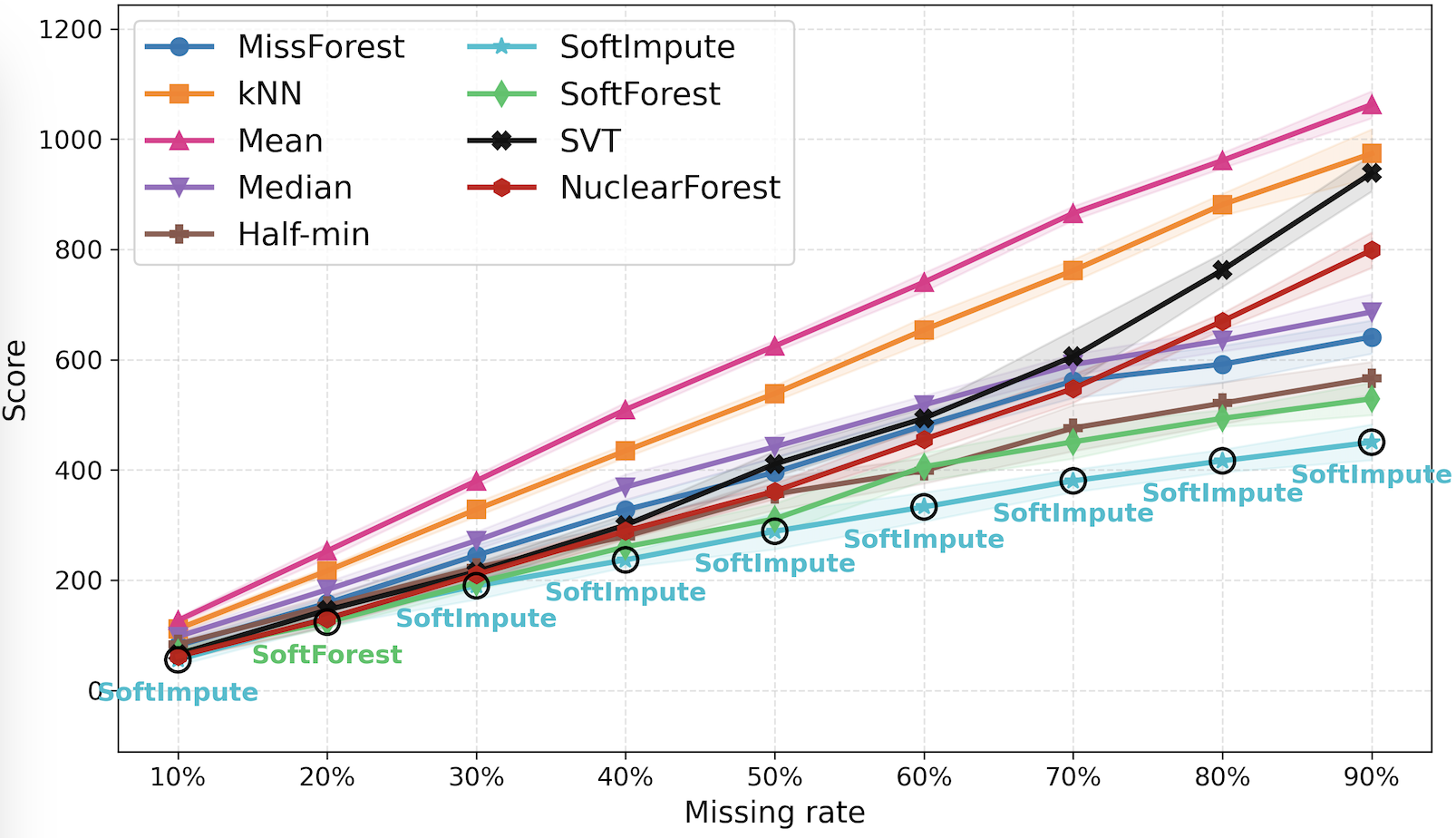}
        \caption{SOR}
        \label{fig:sub21}
    \end{subfigure}
    \hfill
    \begin{subfigure}[b]{0.49\textwidth}
        \centering
        \includegraphics[width=\textwidth]{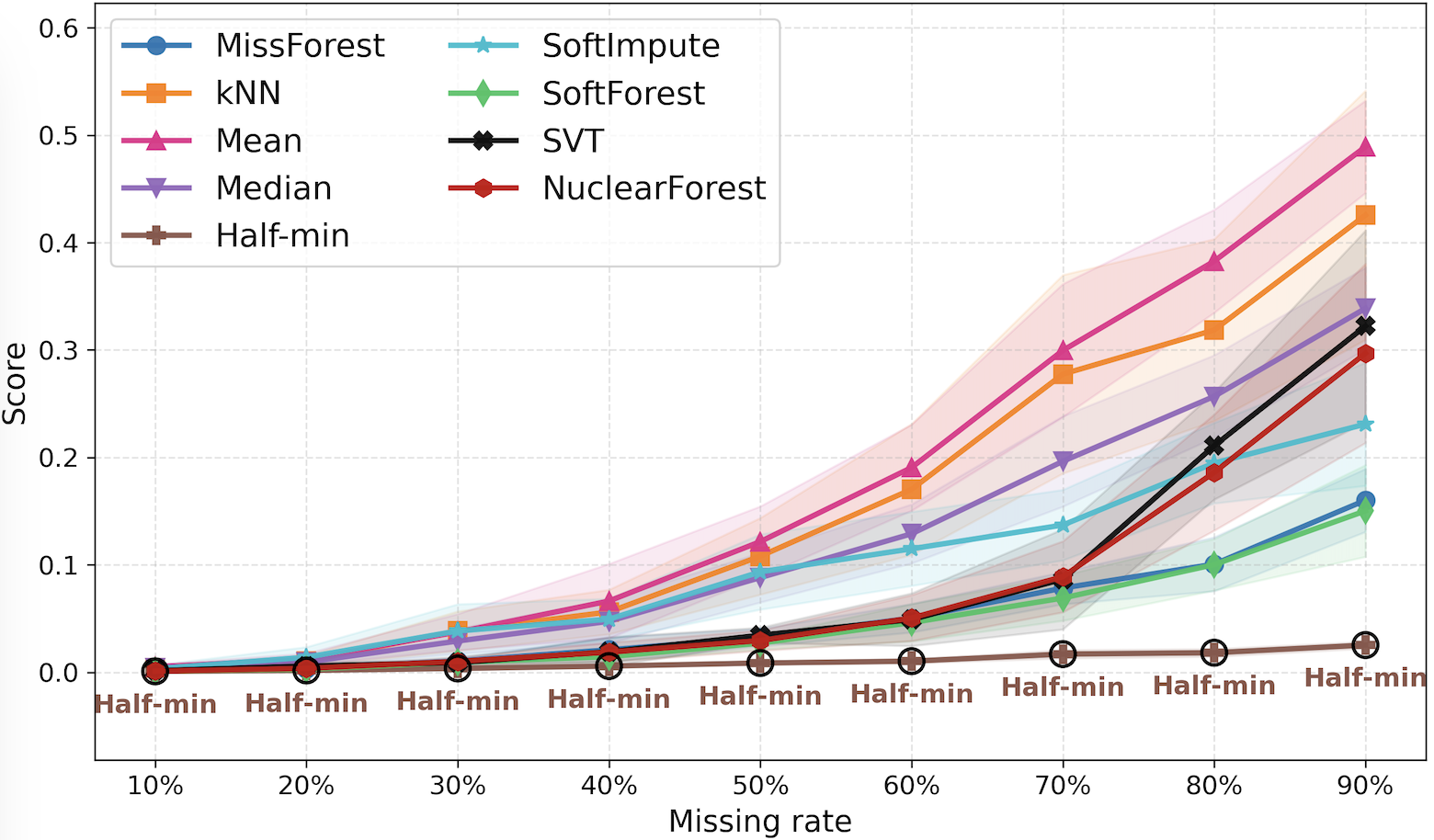}
        \caption{PCA Procrustes distance}
        \label{fig:sub22}
    \end{subfigure}

    \begin{subfigure}[b]{0.49\textwidth}
        \centering
        \includegraphics[width=\textwidth]{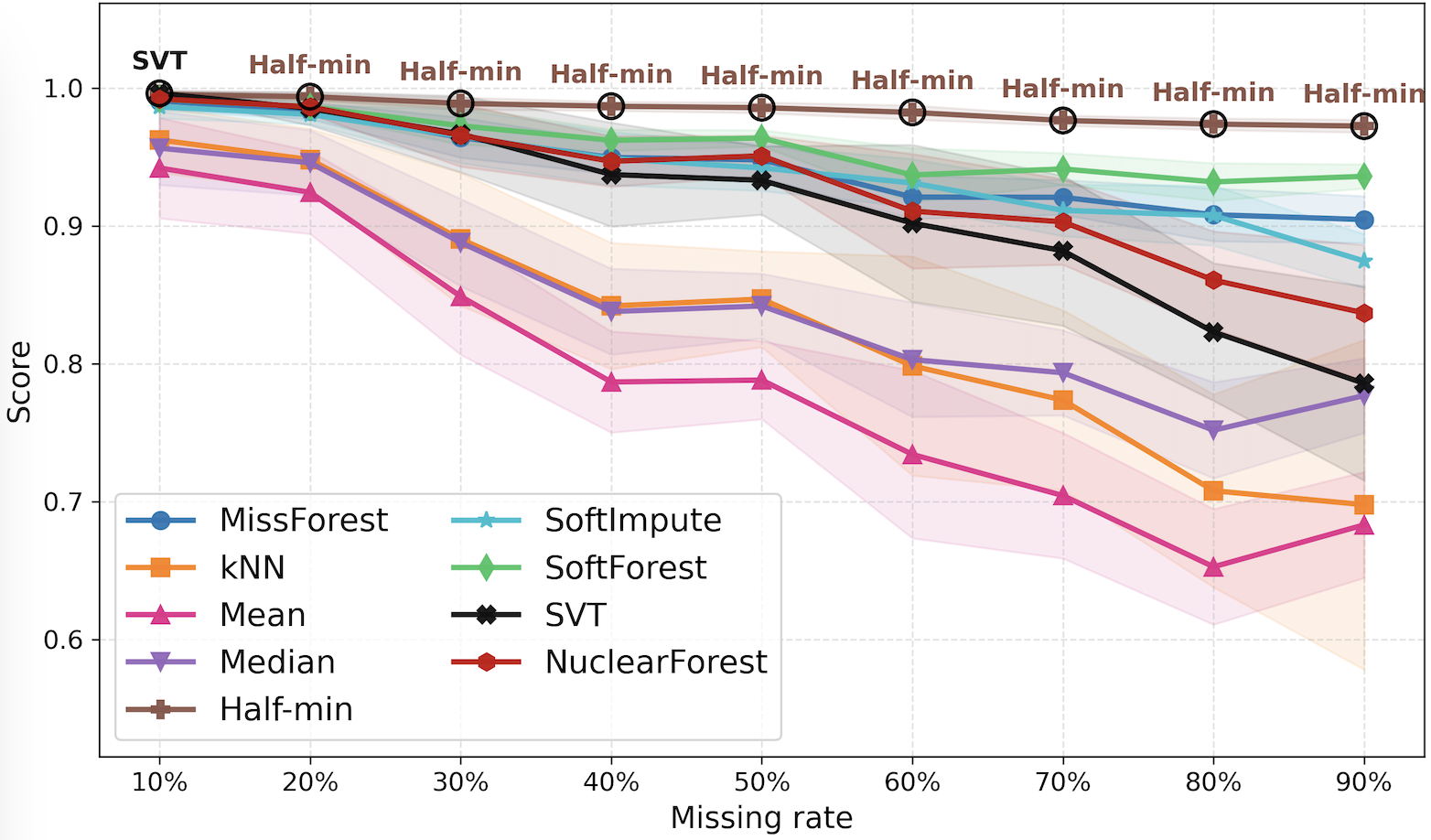}
        \caption{Pearson Log-\(p\) correlation}
        \label{fig:sub23}
    \end{subfigure}
    \hfill
    \begin{subfigure}[b]{0.49\textwidth}
        \centering
        \includegraphics[width=\textwidth]{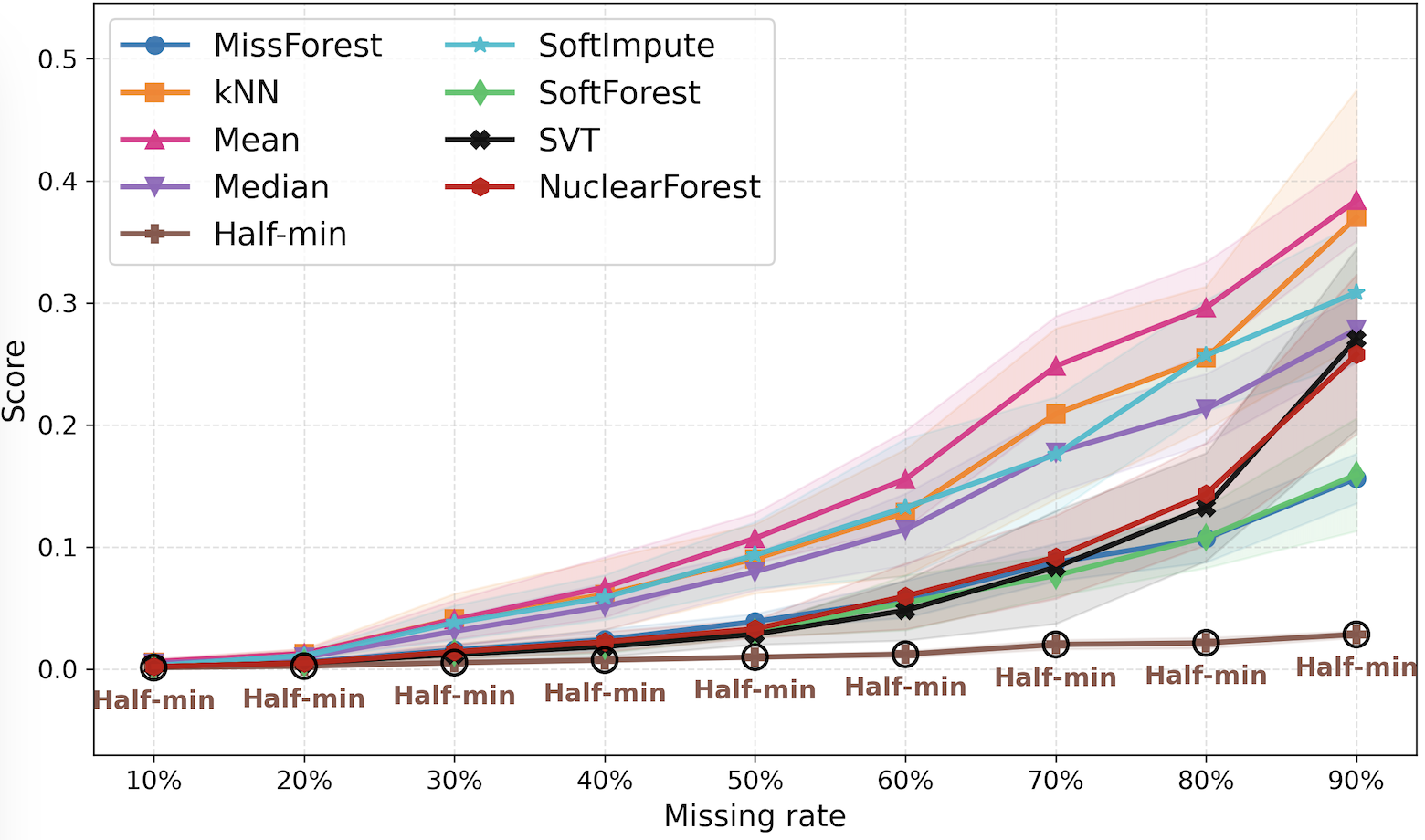}
        \caption{PLS Procrustes distance}
        \label{fig:sub24}
    \end{subfigure}
    \caption{Comparison of the imputation methods using metabolomics dataset under MNAR}    
    \label{fig:naturemnar}
\end{figure}

\paragraph{Metabolomics dataset under MNAR missingness.}
A non-parametric method called NRMSE-based sum of ranks (SOR) is used to evaluate the imputation error for the skewness of the MNAR distribution \citep{wei2018missing}.
\cref{fig:sub21} indicates that SoftImpute demonstrates a clear advantage in this metric, achieving the lowest values across all missingness levels except at 20\% missingness, where SoftForest performs best.
As shown in \cref{fig:naturemnar} \subref{fig:sub22}--\subref{fig:sub24}, the left-censored Half-min imputation performs best for three other metrics, yielding results similar to those reported by \citet{wei2018missing}.
This behavior is expected under the MNAR design, where missing values are generated from low-abundance entries below feature-specific thresholds.
Since Half-min replaces missing entries with a small value, its inductive bias matches the left-censored missingness mechanism.

\begin{figure}[t]
    \centering

    \begin{subfigure}[t]{0.32\textwidth}
        \centering
        \includegraphics[width=\textwidth]{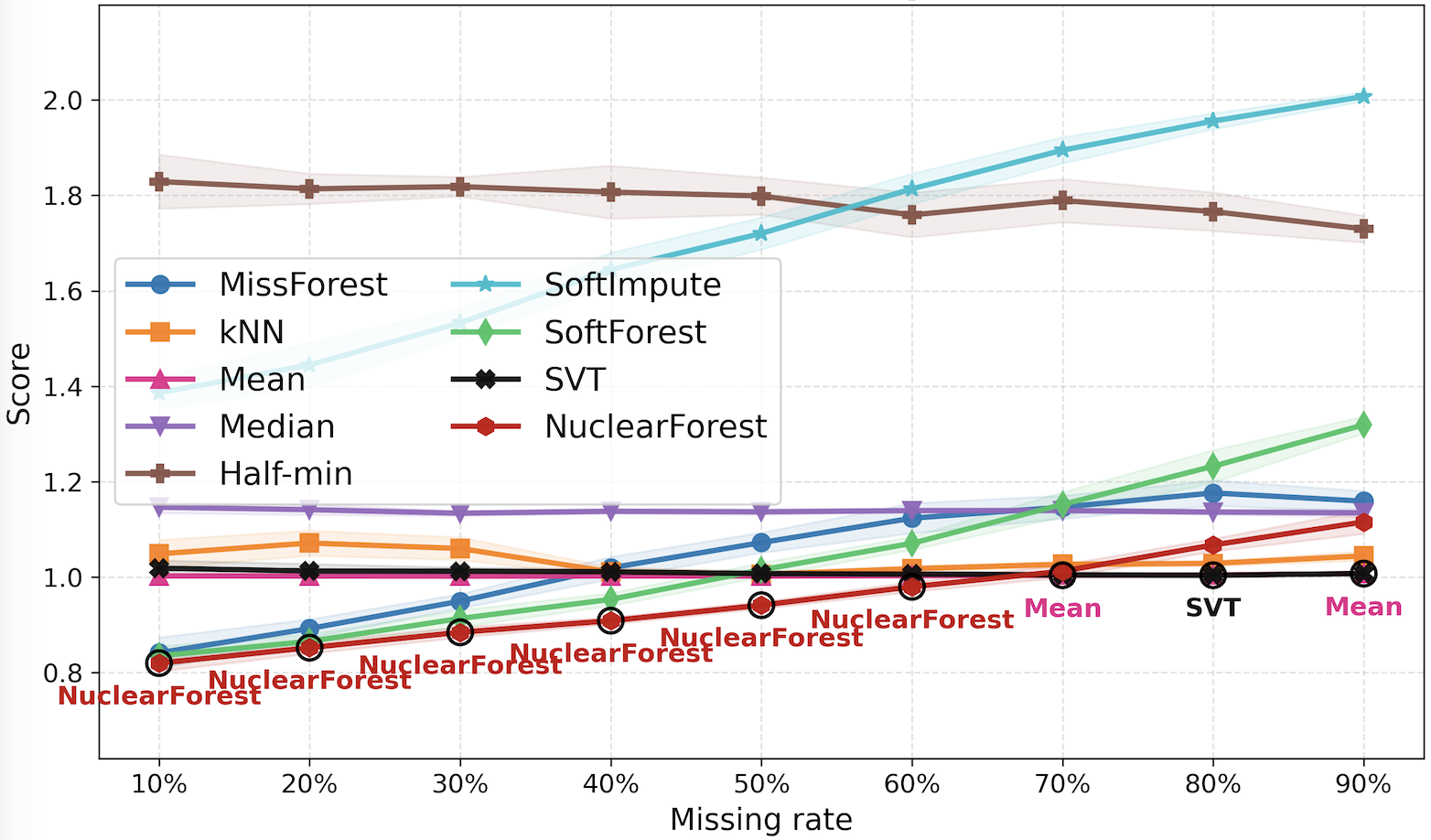}
        \caption{NRMSE}
        \label{fig:sub31}
    \end{subfigure}
    \hfill
    \begin{subfigure}[t]{0.32\textwidth}
        \centering
        \includegraphics[width=\textwidth]{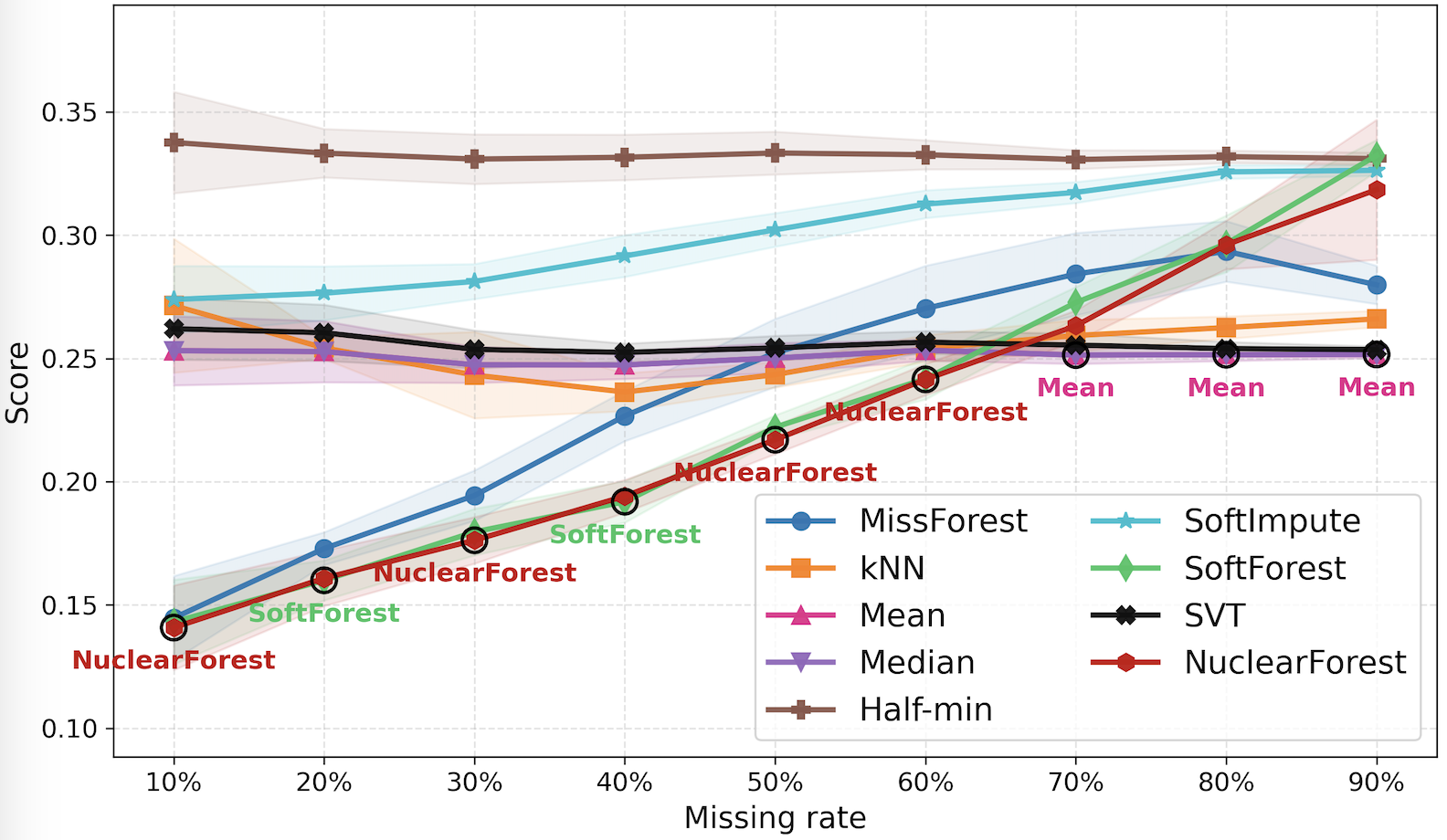}
        \caption{PFC}
        \label{fig:sub32}
    \end{subfigure}
    \hfill
    \begin{subfigure}[t]{0.32\textwidth}
        \centering
        \includegraphics[width=\textwidth]{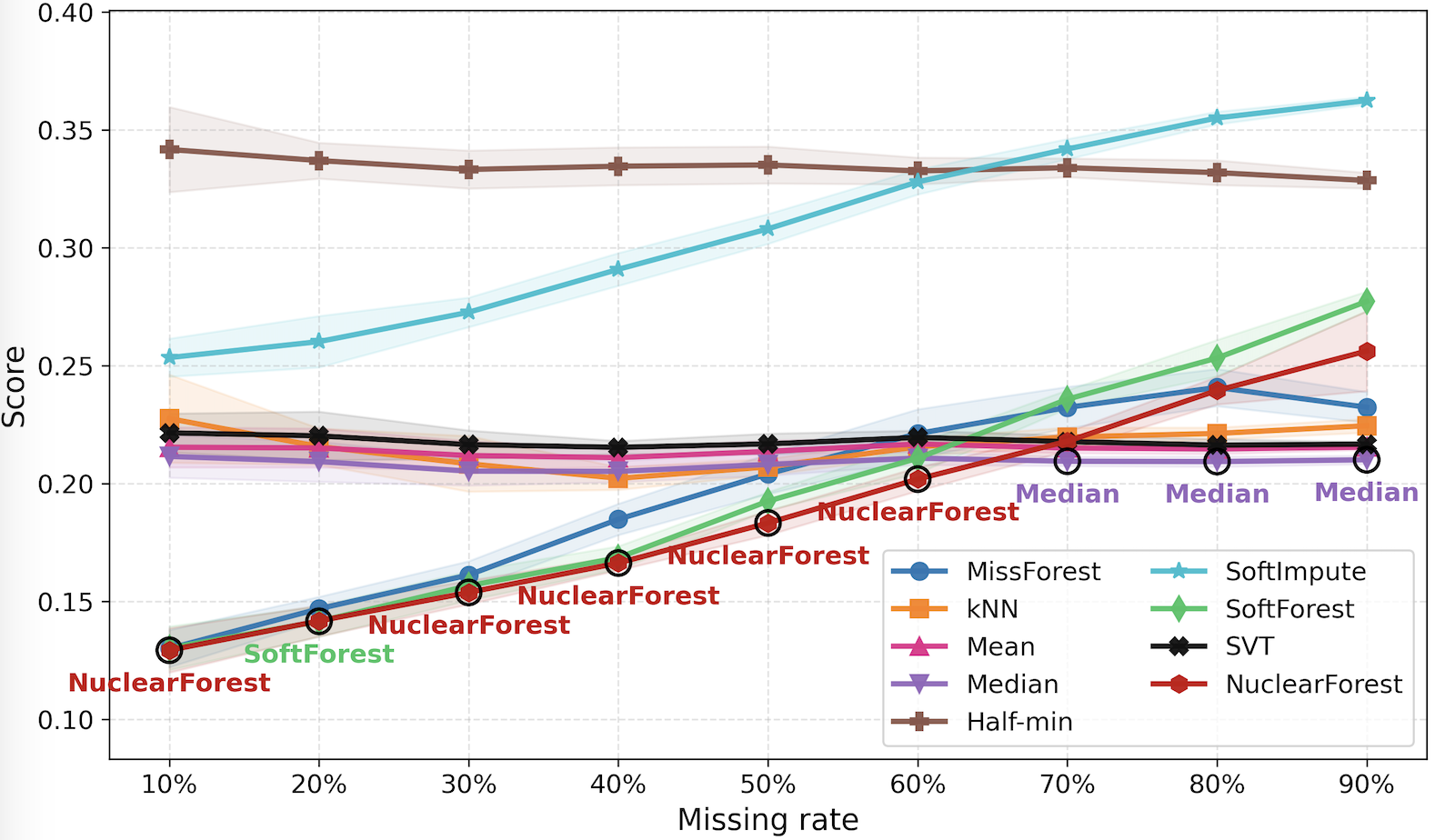}
        \caption{Gower's distance}
        \label{fig:sub33}
    \end{subfigure}

    \begin{subfigure}[t]{0.32\textwidth}
        \centering
        \includegraphics[width=\textwidth]{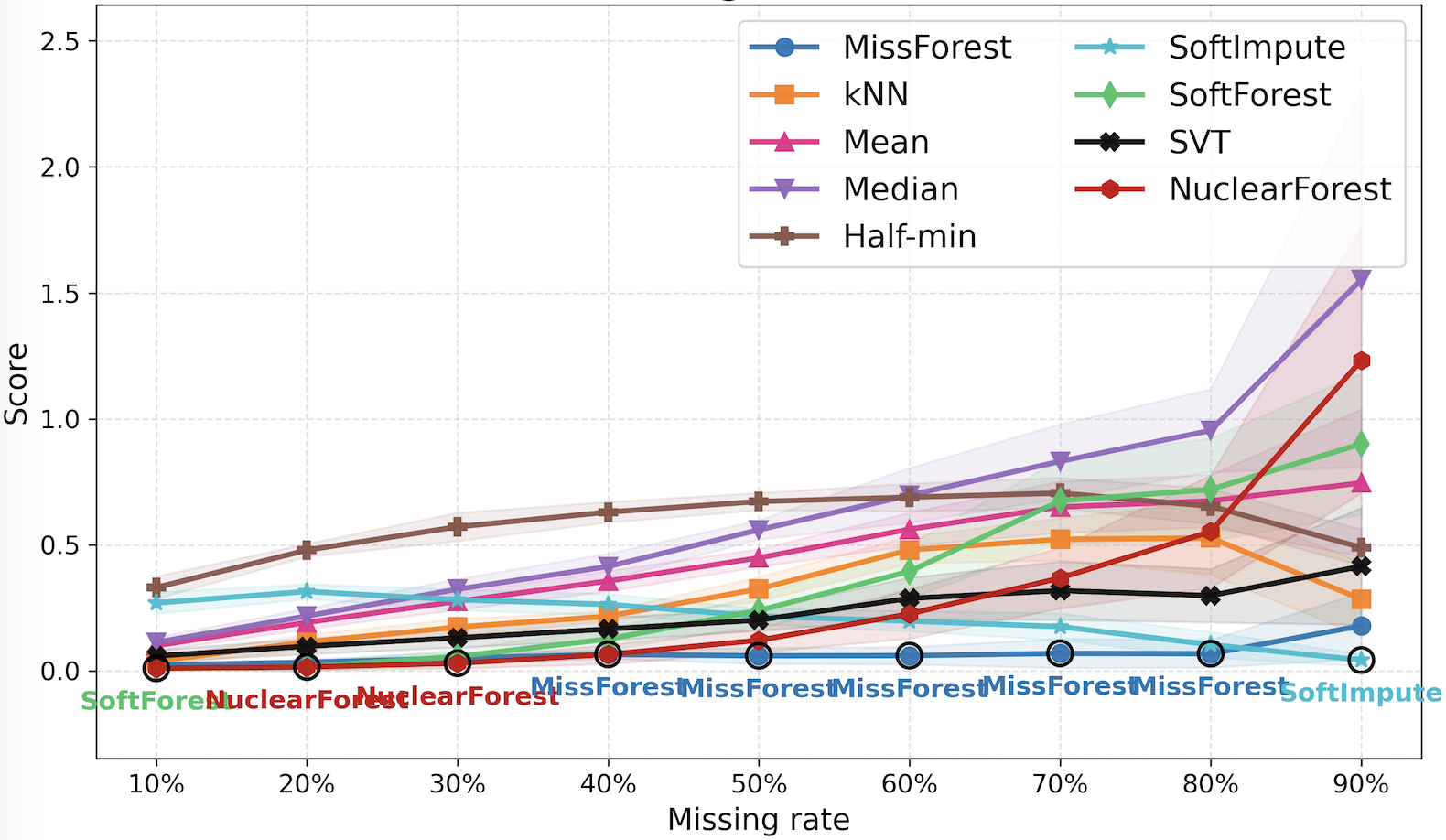}
        \caption{Predictive R\textsuperscript{2} degradation}
        \label{fig:sub34}
    \end{subfigure}
    \hfill
    \begin{subfigure}[t]{0.32\textwidth}
        \centering
        \includegraphics[width=\textwidth]{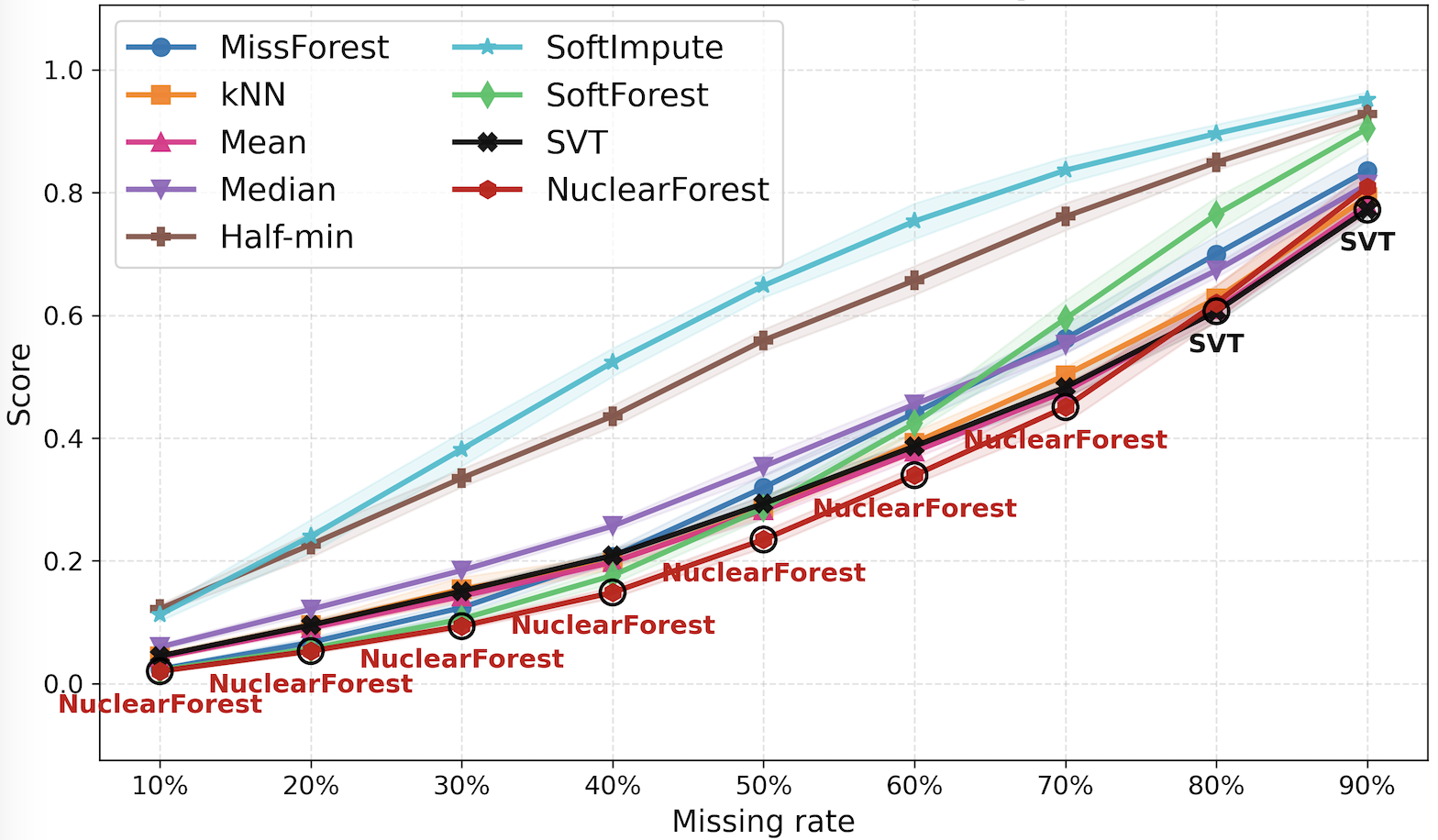}
        \caption{PCA Procrustes distance}
        \label{fig:sub35}
    \end{subfigure}
    \hfill
    \begin{subfigure}[t]{0.32\textwidth}
        \centering
        \includegraphics[width=\textwidth]{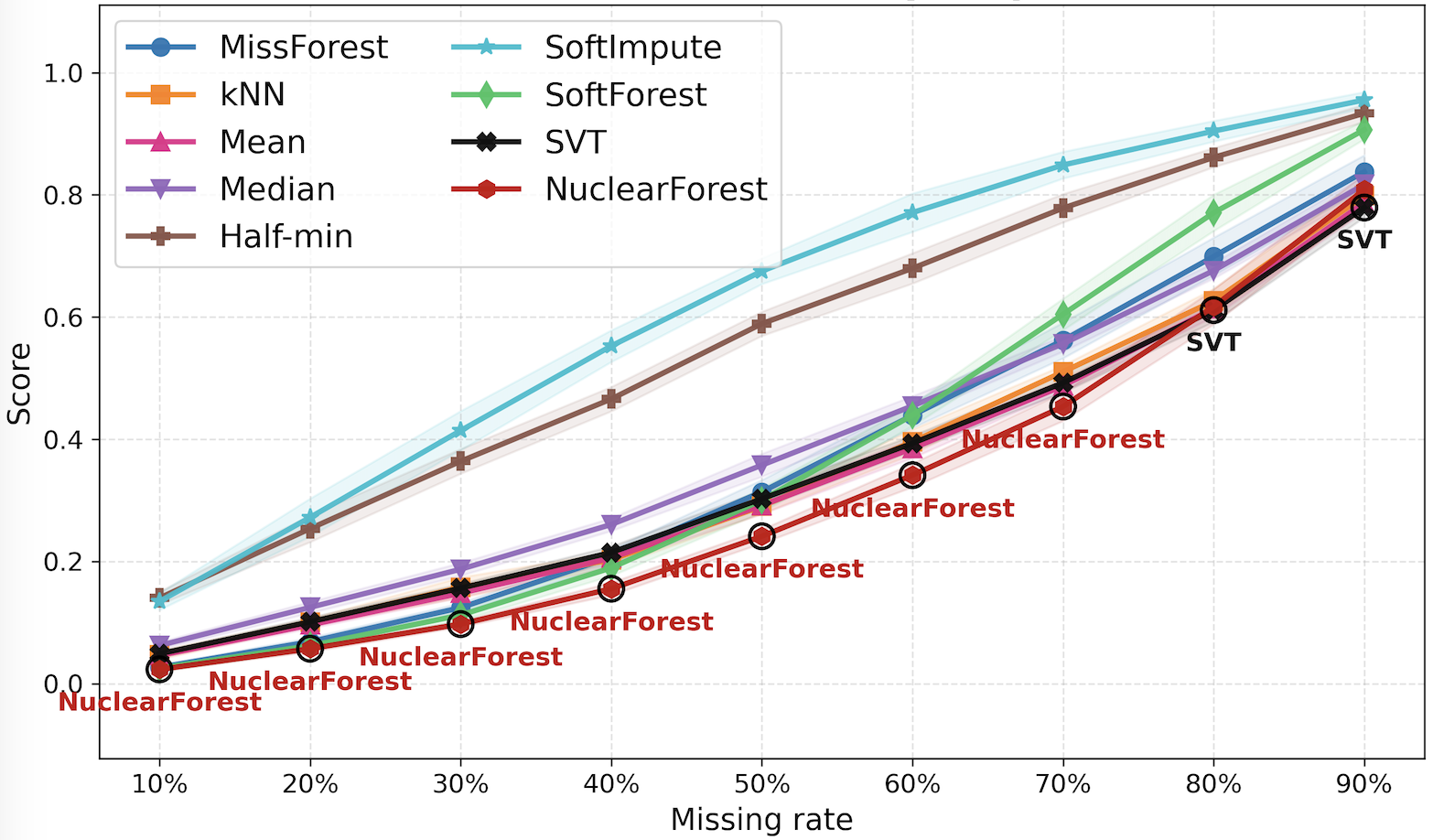}
        \caption{PLS Procrustes distance}
        \label{fig:sub36}
    \end{subfigure}

    \begin{subfigure}[t]{0.32\textwidth}
        \centering
        \includegraphics[width=\textwidth]{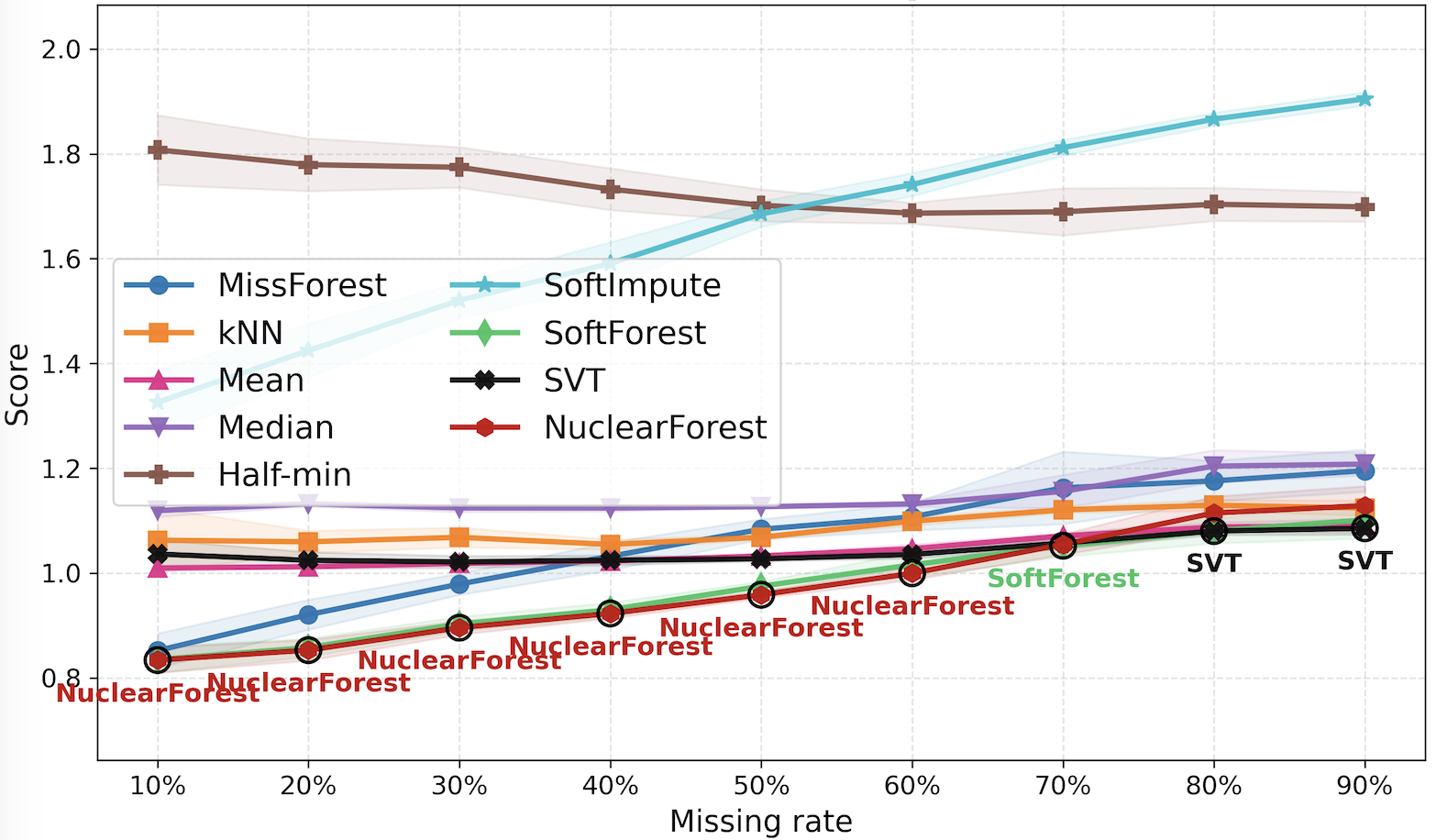}
        \caption{NRMSE}
        \label{fig:sub37}
    \end{subfigure}
    \hfill
    \begin{subfigure}[t]{0.32\textwidth}
        \centering
        \includegraphics[width=\textwidth]{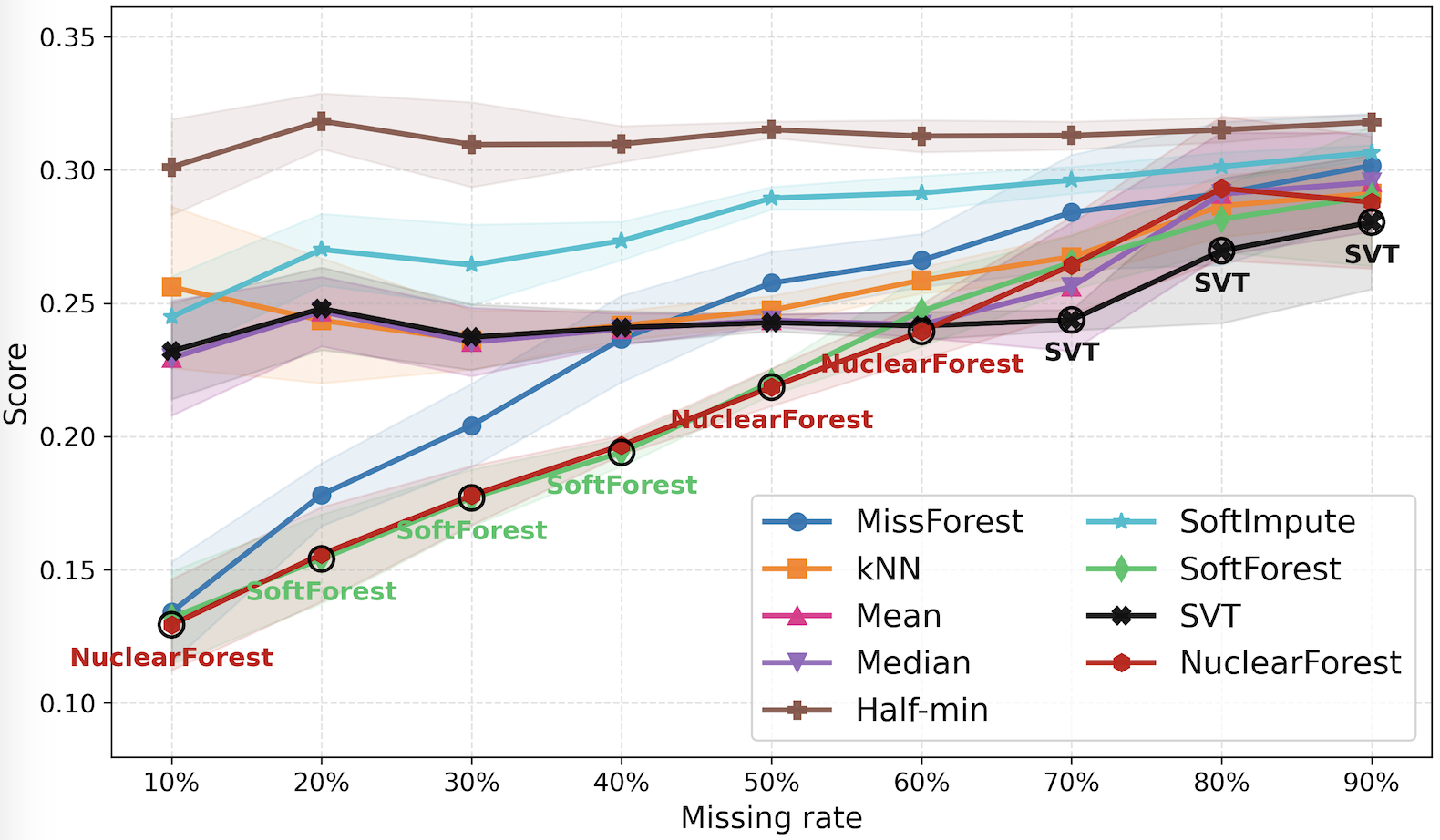}
        \caption{PFC}
        \label{fig:sub38}
    \end{subfigure}
    \hfill
    \begin{subfigure}[t]{0.32\textwidth}
        \centering
        \includegraphics[width=\textwidth]{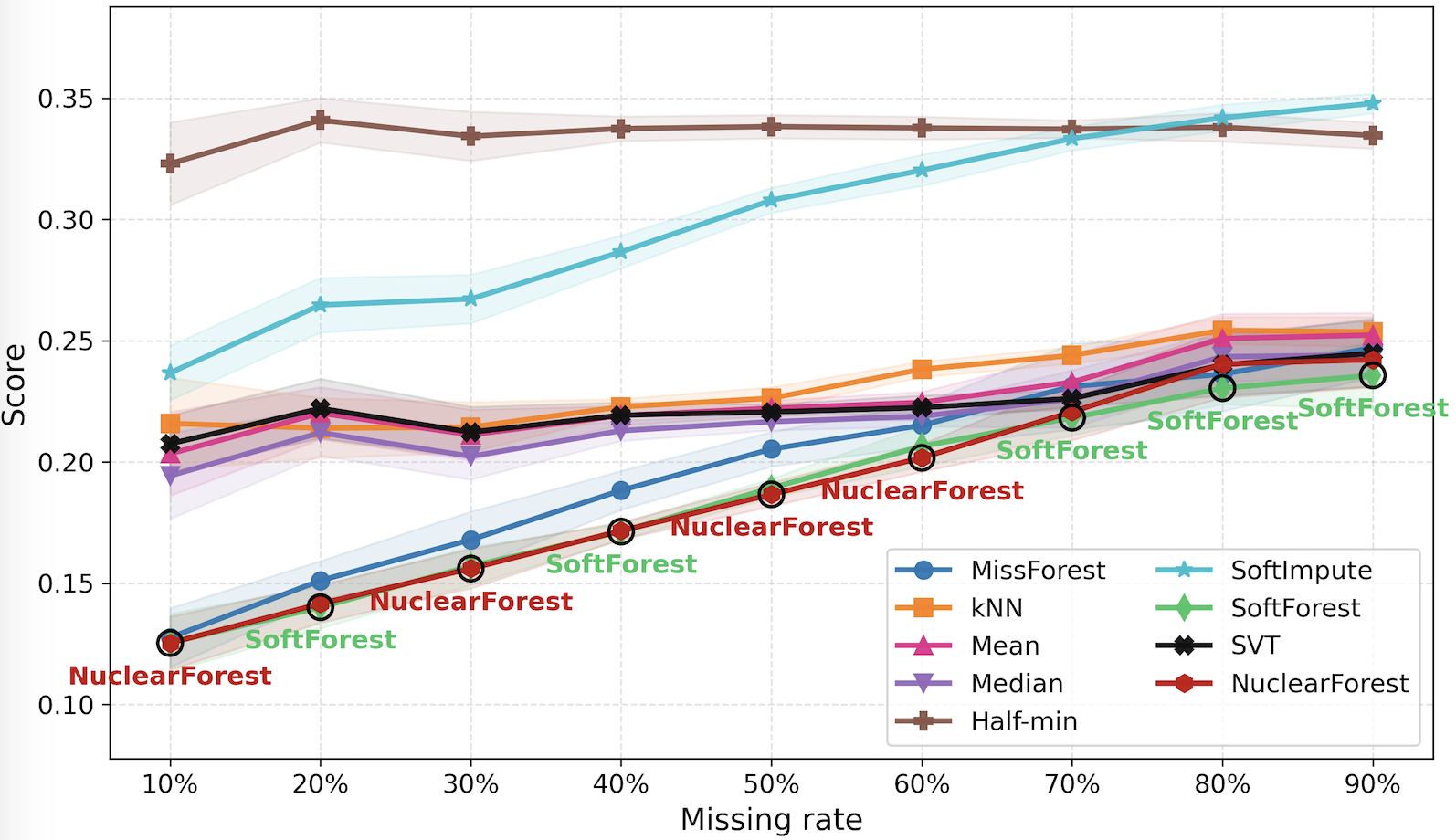}
        \caption{Gower's distance}
        \label{fig:sub39}
    \end{subfigure}

    \begin{subfigure}[t]{0.32\textwidth}
        \centering
        \includegraphics[width=\textwidth]{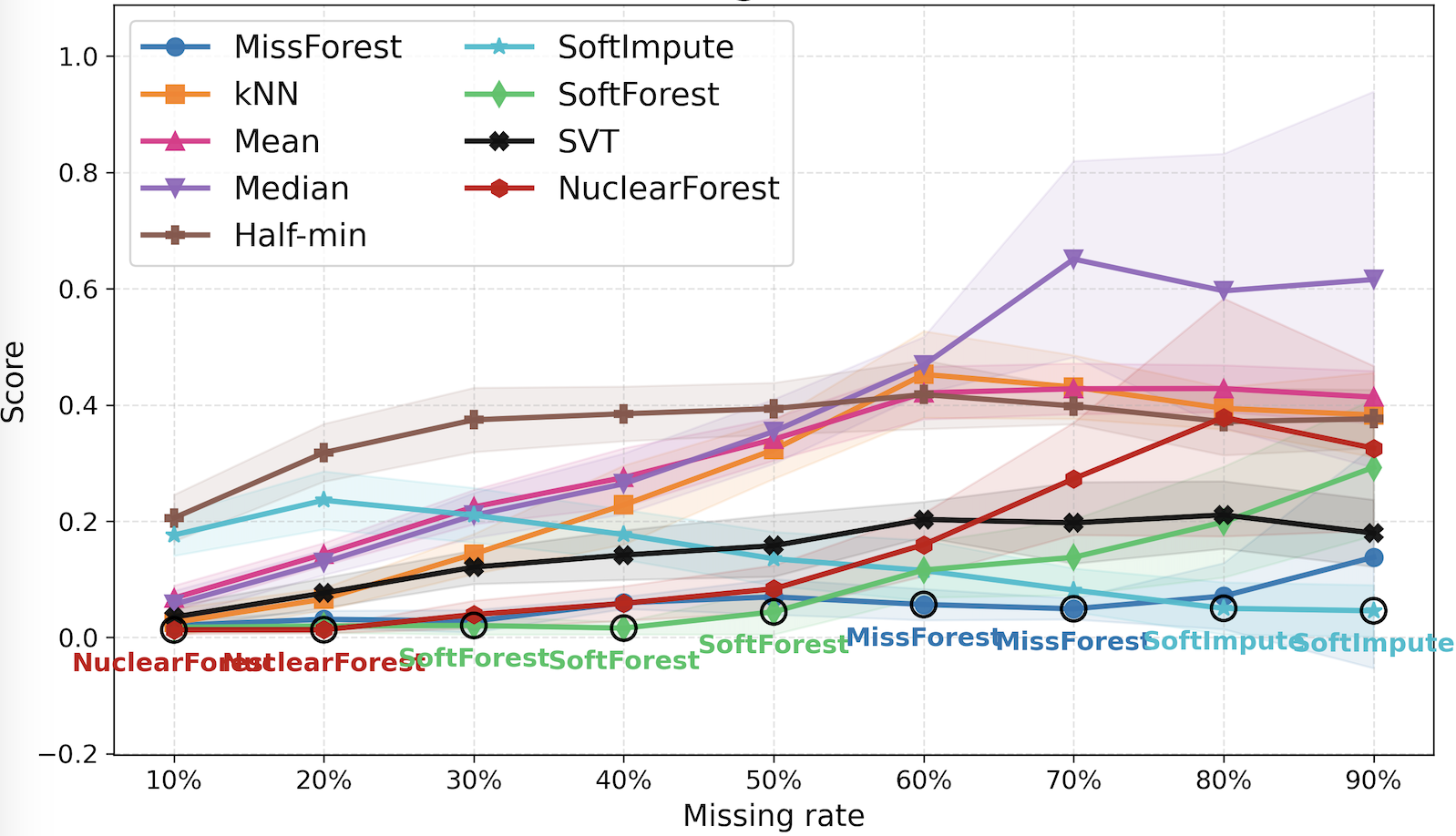}
        \caption{Predictive R\textsuperscript{2} degradation}
        \label{fig:sub40}
    \end{subfigure}
    \hfill
    \begin{subfigure}[t]{0.32\textwidth}
        \centering
        \includegraphics[width=\textwidth]{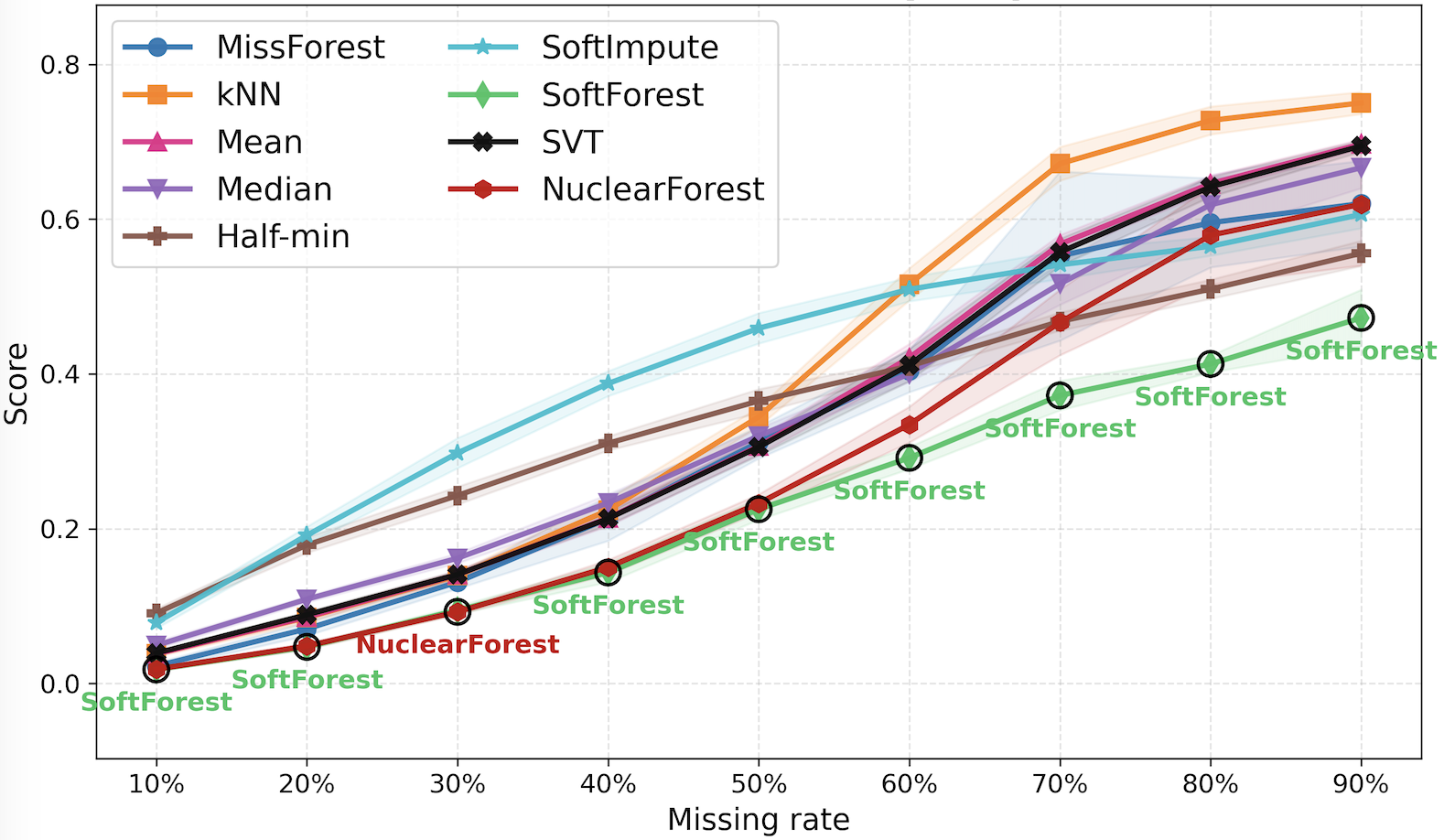}
        \caption{PCA Procrustes distance}
        \label{fig:sub41}
    \end{subfigure}
    \hfill
    \begin{subfigure}[t]{0.32\textwidth}
        \centering
        \includegraphics[width=\textwidth]{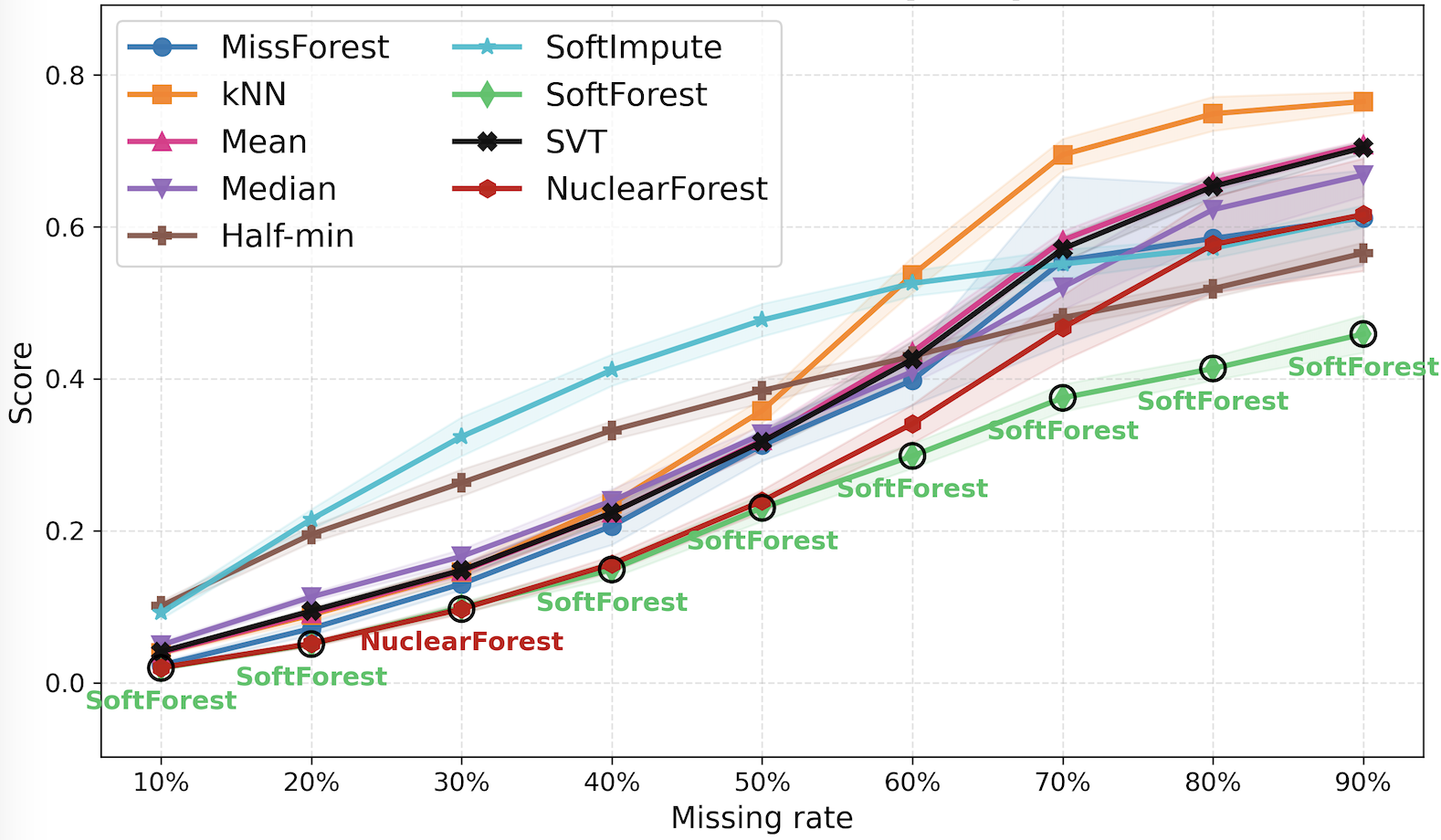}
        \caption{PLS Procrustes distance}
        \label{fig:sub42}
    \end{subfigure}

    \caption{Comparison of imputation methods under MCAR (a--f) and MAR (g--l).}
    \label{fig:housing}
\end{figure}
\begin{table}[t]
\centering
\caption{Runtime comparison on the housing dataset, reported in seconds as mean ± std.}
\label{tab:runtime_results}
\begin{tabular}{lccc}
\toprule
Mechanism 
& MissForest

& SoftForest 

& NuclearForest \\
\midrule
MCAR
& 5.934 ± 0.187 

& 0.664 ± 0.045 

& 0.659 ± 0.047 \\
MAR
& 5.973 ± 0.198 

& 0.661 ± 0.039 

& 0.661 ± 0.034 \\
\bottomrule
\end{tabular}
\end{table}

\paragraph{Housing dataset under MCAR and MAR missingness.}
We further benchmark data imputation on the publicly available Kaggle housing dataset \citep{housing_price_kaggle} under MCAR and MAR mechanisms, having six numerical and seven categorical features.
In terms of estimation error, NRMSE after \(z\)-score transformation is used and for the binary/one hot columns, we calculate the proportion of falsely classified entries (PFC) \citep{stekhoven2012missforest}. \cref{fig:housing} \subref{fig:sub31}--\subref{fig:sub32} show that NuclearForest performs best for missing rates between 10\% and 60\% in terms of NRMSE, and at 10\%, 30\%, and 50\%--60\% in terms of PFC under the MCAR mechanism, while SoftForest achieves the lowest PFC values at 20\% and 40\% missingness.
Under the MAR mechanism, the results in \cref{fig:housing} \subref{fig:sub37}--\subref{fig:sub38} indicate that NuclearForest obtains the lowest NRMSE at 10\%--60\% missingness, while SoftForest performs best at 70\% missingness.
In terms of PFC, NuclearForest achieves the best results at 10\% and 50\%--60\% missingness, whereas SoftForest obtains the lowest values at 20\%--40\% missingness.
For distributional evaluation, Gower's distance measures the dissimilarity for the mixed data types \citep{elbadisy2024imputation,gower1971similarity}.
\cref{fig:sub33} shows that, under the MCAR mechanism, NuclearForest obtains the lowest values at 10\% and 30\%--60\% missingness, while SoftForest achieves the lowest value at 20\% missingness.
Under the MAR mechanism, \cref{fig:sub39} shows that NuclearForest performs best at 10\%, 30\%, and 50\%--60\% missingness, whereas SoftForest obtains the lowest values at 20\%, 40\% and 70\%--90\% missingness.
Downstream price prediction evaluates the utility of the imputed data for regression tasks.
This follows a prediction-oriented evaluation perspective in supervised learning with missing values \citep{josse2024consistency,lemorvan2021imputation,lemorvan2024imputation}, where missing-value handling is assessed by its impact on downstream predictive performance rather than only by reconstruction error. We therefore compute Downstream Predictive \(R^2\) Degradation by training a ridge regression model to predict housing prices from all remaining features.
Overall, these results suggest that the most effective method for preserving downstream predictive performance depends on the severity of missingness. Hybrid methods are particularly competitive at low to moderate missingness levels, whereas MissForest and SoftImpute become more favorable under higher missingness.
To maintain consistency across datasets, we additionally evaluate imputation performance using global and group-level structure preservation metrics, including PCA and PLS Procrustes distance, and Pearson correlation of log-transformed \(p\)-values \citep{wei2018missing}.
These metrics are not commonly used for evaluating regression tasks, where predictive performance (e.g., regression error) is typically the primary criterion, originating instead from omics-based studies where preserving multivariate structure and statistical inference is essential.
Specifically, we divide the dataset into three groups of price tercile labels to assess the imputation methods.
\cref{fig:housing} \subref{fig:sub35}--\subref{fig:sub36} and \cref{fig:housing} \subref{fig:sub41}--\subref{fig:sub42} show the results of PCA/PLS Procrustes distance and the experimental results with respect to the log-transformed \(p\)-values is given in Appendix \ref{app:additionalplot} \cref{fig:housinglogp}. Under the MCAR mechanism, NuclearForest obtains the lowest PCA and PLS Procrustes distance at 10\%--70\% missingness.
At 80\%--90\% missingness, SVT achieves the best result for both PCA and PLS Procrustes distances.
When missingness follows the MAR mechanism, SoftForest obtains the best results across most missingness levels, while NuclearForest performs best at 30\% missingness.
The Pearson log-\(p\) correlation results in \cref{fig:housinglogp}\subref{fig:subo41}--\subref{fig:subo42} show that, under MCAR, median imputation achieves the best performance across most missingness rates, with the exception of the 10\%--20\% setting.
For evaluation under MAR, SoftForest performs best at 50\%--70\% missing rates.

\section{Conclusion}
We proposed NuclearForest and SoftForest, two hybrid imputation methods that combine low-rank initialization with a single random-forest refinement pass. 
NuclearForest is built on an improved SVT variant with warm starting and adaptive step sizes.
For this adaptive SVT component, we prove step-size bounds and a convergence result for the corresponding zero-initialized iteration, and our ablation studies show that both the warm start and adaptive update contribute to improved imputation quality under the MCAR/MAR missingness setting considered in our experiments.
Across the metabolomics and housing benchmarks, the proposed hybrid methods offer a favorable accuracy--efficiency trade-off.
On the former, NuclearForest remains competitive with MissForest across missingness levels and is particularly strong at low missing rates, while MissForest achieves the best performance in several medium- and high-missingness settings.
Importantly, the proposed hybrids are substantially faster: in the MCAR/MAR metabolomics experiment, SoftForest and NuclearForest achieve approximately \(9.52\times\) and \(5.81\times\) speedups over MissForest, respectively.
A limitation is that these gains do not fully extend to the
left-censored MNAR setting, where Half-min is well aligned with this mechanism.
This suggests that the proposed hybrids are most effective when the missingness pattern preserves enough global correlation structure for informative low-rank initialization.

The advantage of the hybrid strategy is more pronounced on the smaller housing dataset, where NuclearForest and SoftForest perform strongly across many low- to moderate-missingness settings, with especially strong performance on PCA and PLS Procrustes metrics, indicating better preservation of low-dimensional and group-level structure.
These findings suggest that low-rank completion provides an effective structural warm start, while a single random forest refinement step can capture nonlinear feature relationships with lower computational cost than fully iterative random forest imputation. Overall, this supports the proposed hybrid strategy as a practical and computationally efficient approach for preserving structural information during tabular data imputation.
Future work will further characterize the full hybrid procedure theoretically and explore low-rank warm starts for other nonlinear refiners, including neural-network-based imputers.

\subsection*{Acknowledgements}
This work was supported by the Bayerisches Verbundforschungsprogramm (BayVFP) of the Free State of Bavaria under the funding line ``Digitalisierung''. We also thank Gina Pommerenke, Judith Mehler, Johanna Schuerlein and Dr. Josef Scheiber at BioVariance GmbH for valuable biological expertise and discussions that helped inform the interpretation of the biomedical data.

\subsection*{AI use statement}
In this work, we used generative AI tools to improve the grammar and readability of the manuscript, identify potentially relevant literature, and assist with coding. All AI-assisted content was manually reviewed and verified. We take responsibility for the final content of this work.

\bibliography{iclr2027_conference}
\bibliographystyle{iclr2027_conference}

\appendix

\section{Adaptive SVT variant and step-size analysis}
\label{app:svtcomparison}

This section compares the proposed adaptive SVT variant with the standard fixed-step SVT and provides the proof of the step-size boundedness result.

\noindent\paragraph{Adaptive Step Size vs. Fixed Step Size}
The original paper by \citet{cai2010singular} uses a fixed step size \(\delta = 1.2/p\), where \(p\) is the sampling rate. Moreover, they prove convergence for a fixed step size satisfying \(0 < \delta < 2\). In this paper, our SVT uses an adaptive step size that increases when the error is decreasing (\verb+step * 1.05+) and shrinks when the error increases (\verb+step * 0.9+).

In a convergence analysis paper by \citet{LeiZhou2019}, they prove that a necessary and sufficient condition for convergence with respect to the Bregman distance, specifically the step size sequence \(\delta_k\), must satisfy \(\sum_{k=1}^{\infty} \delta_k = \infty\) and the corresponding convergence proof is provided by \citep[Proposition~5]{LeiZhou2019}. The necessity is derived using a lower bound on the decay of the Bregman distance. This provides theoretical justification for employing an adaptive step size, as long as the step sizes do not decay too rapidly (bounded away from zero), ensuring divergence of their cumulative sum.

Recent paper in optimization theory suggests that adaptive step size strategies can accelerate first-order methods beyond what is achievable with constant step sizes. In particular, \citet{Davis2025AdaptiveGD} demonstrate that gradient descent with adaptive step sizes achieves nearly linear convergence under weaker growth conditions where constant step size methods exhibit only sublinear rates. 

Although these results are established for smooth unconstrained optimization (nuclear norm is not smooth), they provide important insight into the role of step size selection in gradient-based algorithms. Since the SVT algorithm can be interpreted as a gradient descent method applied to the dual problem, this perspective motivates the use of adaptive step size strategies within SVT. Moreover, it can improve convergence behavior by applying larger effective updates in the case of slow progress.

\noindent\paragraph{Warm-start initialization (column-mean imputation) vs. cold-start (zero matrix)} Standard SVT updates the matrix starting from zero. Motivated by the warm-start strategy used in SoftImpute \citep{mazumder2010spectral}, we initialize the missing entries with column means before applying the low-rank completion step. This deterministic initialization provides a simple and stable starting point for the iterative updates.

%\paragraph{The nonmonotone schemes} 

%As can be seen in the previous studies, nonmonotone schemes can improve the likelihood of finding a global optimum. Moreover, they can improve convergence speed in cases where a monotone scheme is forced to creep along the bottom of a narrow curved valley \citep{zhang2004nonmonotone}. 

\paragraph{Adaptive step size convergence under finitely many contraction steps}

\paragraph{Linear equality constraints \citep{cai2010singular}.}

Set the objective function
\[
f_\tau(\mathbf{X}) = \tau \|\mathbf{X}\|_* + \frac{1}{2}\|\mathbf{X}\|_F^2
\]
for some fixed \( \tau > 0 \), and consider the following optimization problem:
\begin{equation}
\begin{aligned}
\text{minimize} \quad & f_\tau(\mathbf{X}) \\
\text{subject to} \quad & \mathcal{A}(\mathbf{X}) = \mathbf{b},
\end{aligned}
\tag{3.1}
\end{equation}
where \( \mathcal{A} \) is a linear transformation mapping \( n_1 \times n_2 \) matrices into \( \mathbb{R}^m \) and \( \mathcal{A}^* \) denotes its adjoint. Then the Lagrangian for this problem
is of the form

\[
\mathcal{L}(\mathbf{X}, \mathbf{y}) 
= f_\tau(\mathbf{X}) 
+ \langle \mathbf{y}, \mathbf{b} - \mathcal{A}(\mathbf{X}) \rangle,
\tag{3.2}
\]
where \( \mathbf{X} \in \mathbb{R}^{n_1 \times n_2} \) and \( \mathbf{y} \in \mathbb{R}^m \), 
and starting with \( \mathbf{y}^0 = 0 \), Uzawa's iteration is given by
\begin{equation}
\begin{cases}
\mathbf{X}^k = \mathcal{S}_\tau\!\bigl(\mathcal{A}^*(\mathbf{y}^{k-1})\bigr), \\
\mathbf{y}^k = \mathbf{y}^{k-1} + \delta_k \bigl(\mathbf{b} - \mathcal{A}(\mathbf{X}^k)\bigr).
\end{cases}
\tag{3.3}
\end{equation}

\paragraph{Shrinkage iterations \citep{cai2010singular}.}

For the matrix completion problem, let \( \mathcal{A} = \mathcal{P}_\Omega \), \( b = \mathcal{P}_\Omega(M) \), and note that
\( \|\mathcal{P}_\Omega\| = 1 \).
Then \((3.1)\) reduces to
\begin{equation}
\begin{aligned}
\text{minimize} \quad & \tau \|\mathbf{X}\|_* + \frac{1}{2}\|\mathbf{X}\|_F^2 \\
\text{subject to} \quad & \mathcal{P}_\Omega(\mathbf{X}) = \mathcal{P}_\Omega(\mathbf{M}).
\end{aligned}
\tag{2.8}
\end{equation}

and the corresponding iteration \((3.3)\) becomes
\begin{equation}
\begin{cases}
\mathbf{X}^k = \mathcal{S}_\tau\!\bigl(\mathbf{Y}^{k-1}\bigr), \\
\mathbf{Y}^k = \mathbf{Y}^{k-1} + \delta_k \mathcal{P}_\Omega\!\bigl(\mathbf{M} - \mathbf{X}^k\bigr),
\end{cases}
\tag{2.7}
\end{equation}
where we write \( \mathbf{Y}^k \) in place of \( \mathbf{y}^k \).

\begin{theorem}[SVT convergence {\citep{cai2010singular}}]
Suppose the step sizes obey
\[ 0 < \inf_k \delta_k \le \sup_k \delta_k < 2/\|\mathcal{A}\|^2 \]
Then the sequence \( \{\mathbf{X}^k\} \) obtained via \((3.3)\) converges to the unique solution to \((3.1)\). 
In particular, the sequence \( \{\mathbf{X}^k\} \) obtained via \((2.7)\) converges to the unique solution of \((2.8)\) provided that 
\[ 0 < \inf_k \delta_k \le \sup_k \delta_k < 2 \]
\end{theorem}

\begin{theorem} \citep{LeiZhou2019}
Let \( \{(X^k, Y^k)\}_{k \in \mathbb{N}} \) be produced by \((3.3)\) and \( \mathbf{b}_0 \neq 0 \). 
Here, \(\mathbf{b}_0\) denotes the orthogonal projection of \(\mathbf{b}\) onto the range of \(\mathcal{A}\).
The term \(D_\Psi^{Y^T}(X^\star,X^T)\) denotes the Bregman distance associated
with \(\Psi\), evaluated at \(X^T\) with subgradient \(Y^T\). In general, for
\(\widetilde Y\in\partial\Psi(\widetilde X)\),
\(
D_\Psi^{\widetilde Y}(X,\widetilde X)
=
\Psi(X)-\Psi(\widetilde X)-\langle X-\widetilde X,\widetilde Y\rangle .
\)

Then the following statements hold.
\begin{enumerate}
\item If \( \sup_k \delta_k < \frac{1}{2\|\mathcal{A}\|^2} \), then
\[
\lim_{T \to \infty} D_\Psi^{Y^T}(X^\star, X^T) = 0 
\quad \text{if and only if} \quad 
\sum_{k=1}^{\infty} \delta_k = \infty.
\]

\item If \( \sup_k \delta_k < \frac{2}{\|\mathcal{A}\|^2} \), then
\[
\|X^{T+1} - X^\star\|_F^2 \le \tilde{C} \left[ \sum_{k=1}^{T} \delta_k \right]^{-1}, 
\quad \forall T \in \mathbb{N},
\]
where \( \tilde{C} \) is a constant independent of \( T \).
\end{enumerate}
\end{theorem}

\begin{Proposition}[Adaptive step sizes bounds under condition]
\label{prop:1}
Let \(p:=|\Omega|/(n_1n_2)\) denote the observed-entry ratio.  Assume \(0<p<1\), let \(\delta_0=1.2\,p\) and set
\(\delta_1=\delta_0\). Let
\[
e_k :=
\frac{\|\mathcal{P}_\Omega(X^{k+1}-M)\|_F}
{\|\mathcal{P}_\Omega(M)\|_F}
\]
denote the relative observed-entry reconstruction error, as in
Algorithm~\ref{alg:svt-me}. For \(k\ge 1\), let the step sizes be generated by
\[
\delta_{k+1}=
\begin{cases}
0.9\,\delta_{k}, & \text{if } e_k>e_{k-1},\\[1mm]
\min\,\bigl(1.05\,\delta_{k},\,2p\bigr), & \text{if } e_k\le e_{k-1}.
\end{cases}
\]
Assume further that the total number of contraction steps is finite, namely, there exists
\(B\in\mathbb N_0\) such that at most \(B\) indices \(k\) satisfy \(e_k>e_{k-1}\).
Then the following statements hold:

\begin{enumerate}
\item[\textup{(a)}] For all \(k\ge 0\),
\begin{equation}\label{eq:uniform-bounds}
0.9^{B}\,\delta_0 \;\le\; \delta_k \;\le\; 2p \;<\; 2.
\end{equation}
Hence
\[
0<\inf_k \delta_k \le \sup_k \delta_k <2.
\]

\item[\textup{(b)}] The step sizes have divergent sum:
\begin{equation}\label{eq:divergent-sum}
\sum_{k=0}^\infty \delta_k = \infty.
\end{equation}

\end{enumerate}
\end{Proposition}

\begin{proof} 
The proof is in three parts.

\medskip
\noindent\textbf{Part 1: positivity and upper bound.}
We first show by induction that
\[
0<\delta_k\le 2p
\qquad\forall k\ge 0.
\]
For \(k=0\), this is immediate from \(\delta_0=1.2\,p\) and \(0<p<1\), since
\[
0<\delta_0=1.2\,p\le 2p.
\]
For \(k=1\), this also holds because \(\delta_1=\delta_0\).

Now let \(k\ge 1\) and assume \(0<\delta_{k}\le 2p\). There are two cases.
If \(e_k>e_{k-1}\), then
\[
\delta_{k+1}=0.9\,\delta_{k},
\]
so \(0<\delta_{k+1}<2p\).

If \(e_k\le e_{k-1}\), then
\[
\delta_{k+1}=\min\,\bigl(1.05\,\delta_{k},\,2p\bigr),
\]
which again implies \(0<\delta_{k+1}\le 2p\).

Thus, by induction, \(0<\delta_k\le 2p\) for all \(k\ge 0\).

\medskip
\noindent\textbf{Part 2: uniform positive lower bound under finitely many contraction steps.}
For \(k\ge 1\), define
\[
N_k
:=
\#\{j\in\{1,\ldots,k-1\}: e_j>e_{j-1}\},
\]
with \(N_0=N_1=0\). Thus \(N_k\) counts the number of contraction steps used in generating
\(\delta_k\) from the initial step size. By assumption,
\[
N_k\le B
\qquad\forall k\ge 0.
\]

An expansion step never decreases the step size, while a contraction step multiplies the step size by exactly \(0.9\). Therefore, after generating \(\delta_k\),
the smallest possible value of \(\delta_k\) is obtained by applying the factor \(0.9\) at each of the
\(N_k\) contraction steps and no decrease at the remaining steps. Hence
\[
\delta_k \ge 0.9^{N_k}\delta_0 \ge 0.9^B\delta_0.
\]
Combining this with Part 1 yields
\[
0.9^B\delta_0 \le \delta_k \le 2p<2,
\]
which proves \eqref{eq:uniform-bounds}.

\medskip
\noindent\textbf{Part 3: divergence of the sum.}
From the lower bound just proved,
\[
\delta_k\ge 0.9^B\delta_0>0
\qquad\forall k\ge 0.
\]
Therefore,
\[
\sum_{k=0}^\infty \delta_k
\;\ge\;
\sum_{k=0}^\infty 0.9^B\delta_0
=
\infty,
\]
which proves \eqref{eq:divergent-sum}.
\end{proof}

\begin{corollary}[Convergence of the corresponding zero-initialized adaptive SVT iteration]
\label{cor:adaptive-svt-zero-init}

Consider the standard SVT iteration~(2.7) with zero initialization \(Y^0=0\), and let the step sizes be generated by the same adaptive rule as in Proposition~\ref{prop:1}. Define \(e_k\) analogously for this zero-initialized iteration, and assume that its total number of contraction steps is finite, which means there exists \(B\in\mathbb{N}_0\) such that at most \(B\) indices \(k\) satisfy \(e_k>e_{k-1}\).

Then, by the same argument as in Proposition~\ref{prop:1},
\[
0.9^B\delta_0
\le
\delta_k
\le
2p
<
2,
\]
and hence
\[
0
<
\inf_k\delta_k
\le
\sup_k\delta_k
<
2.
\]
Therefore, by the SVT convergence theorem of \citet{cai2010singular}, the corresponding zero-initialized adaptive SVT iterates converge to a unique solution.
Moreover, Proposition~9 of \citet{LeiZhou2019} provides the following convergence rate bound for the zero-initialized iteration:
\[
D_\Psi^{Y^{T+1}}(X^\ast,X^{T+1})
\le
\widetilde{C}
\left[
\sum_{k=1}^{T}\delta_k
\right]^{-1},
\]
where \(\widetilde{C}\) is independent of \(T\).
Since
\[
\delta_k \ge 0.9^B\delta_0,
\]
we have
\[
\sum_{k=1}^{T}\delta_k
\ge
T\,0.9^B\delta_0,
\]
and therefore
\[
D_\Psi^{Y^{T+1}}(X^\ast,X^{T+1})
\le
\frac{\widetilde{C}}
{T\,0.9^B\delta_0}
=
O(T^{-1}).
\]
\end{corollary}

\section{Algorithms}

\begin{algorithm}[H]
\caption{Singular Value Thresholding (SVT) \citep{cai2010singular}}
\label{alg:svt}
\begin{algorithmic}[1]
\Require Partially observed matrix \(M \in \mathbb{R}^{n_1 \times n_2}\), observed index set \(\Omega\), threshold \(\tau\), step size \(\delta\), tolerance \(\varepsilon\)
\State Initialize \(Y^{0} \gets 0\)
\Repeat
    \State \(X^{k} \gets \mathcal{S}_\tau(Y^{k-1})\)
    \State \(Y^{k} \gets Y^{k-1} + \delta\,\mathcal{P}_\Omega(M - X^{k})\)
\Until{\(\|\mathcal{P}_\Omega(X^{k} - M)\|_F \,/\, \|\mathcal{P}_\Omega(M)\|_F < \varepsilon\)}
\State \Return \(\widehat{M}\) with \(\widehat{M}_{ij} = M_{ij}\) for \((i,j) \in \Omega\), and \(\widehat{M}_{ij} = X^{k}_{ij}\) otherwise
\end{algorithmic}
\end{algorithm}

\begin{algorithm}[H]
\caption{SVT with warm start and adaptive step size}
\label{alg:svt-me}
\begin{algorithmic}[1]
\Require Partially observed matrix \(M \in \mathbb{R}^{n_1 \times n_2}\), observed index set \(\Omega\), threshold \(\tau\), tolerance \(\varepsilon\)
\State Initialize \(X^{0}\) by column-mean imputation; \(Y^{0} \gets X^{0}\)
\State \(\delta_{0} \gets 1.2\,p\), \(\delta_{\max} \gets \min(2p,\, 2)\), where \(p = |\Omega|/(n_1 n_2)\), $e_{-1} \gets \infty$
\Repeat
    \State \(Y^{k+1} \gets Y^{k} + \delta_{k}\,\mathcal{P}_\Omega(M - X^{k})\)
    \State \(X^{k+1} \gets \mathcal{S}_\tau(Y^{k+1})\)
    \State \(e_{k} \gets \|\mathcal{P}_\Omega(X^{k+1} - M)\|_F \,/\, \|\mathcal{P}_\Omega(M)\|_F\)
    \State \(\delta_{k+1} \gets 0.9\,\delta_{k}\) if \(e_{k} > e_{k-1}\), else \(\min(1.05\,\delta_{k},\, \delta_{\max})\)
\Until{\(e_{k} < \varepsilon\)}
\State \Return \(\widehat{M}\) with \(\widehat{M}_{ij} = M_{ij}\) on \(\Omega\), \(\widehat{M}_{ij} = X^{k+1}_{ij}\) otherwise
\end{algorithmic}
\end{algorithm}

\begin{algorithm}[H]
\caption{SoftImpute \citep{mazumder2010spectral}}
\label{alg:softimpute}
\begin{algorithmic}[1]
\Require Partially observed matrix \(M \in \mathbb{R}^{n_1 \times n_2}\), observed index set \(\Omega\), maximum rank \(r_{\max}\), regularization \(\lambda\), tolerance \(\varepsilon\)
\State Initialize \(X^{0} \gets 0\)
\Repeat
    \State \(W^{k} \gets \mathcal{P}_\Omega(M) + \mathcal{P}_\Omega^\perp(X^{k})\)
    \State \(U\Sigma V^\top \gets \mathrm{SVD}(W^{k})\), truncate to top \(r_{\max}\) components
    \State \(X^{k+1} \gets S_\lambda(W^{k})\)
\Until{\(\|X^{k+1} - X^{k}\|_F^2 \,/\, \|X^{k}\|_F^2 < \varepsilon\)}
\State \Return \(\widehat{M}\) with \(\widehat{M}_{ij} = M_{ij}\) on \(\Omega\), \(\widehat{M}_{ij} = X^{k+1}_{ij}\) otherwise
\end{algorithmic}
\end{algorithm}

\clearpage

\section{Experimental details}
\label{app:experimental-details}

This section summarizes the hyperparameters and implementation details used in the experiments.
\begin{table}[H]
\centering
\caption{Hyperparameter settings used in the additional experiments. SVT\_OGpaper follows the fixed step size configuration of \citet{cai2010singular}. SoftImpute follows the soft-thresholded SVD formulation of \citet{mazumder2010spectral}.}
\label{tab:hyperparameters}

\begin{tabular}{p{0.20\linewidth} p{0.73\linewidth}}
\toprule
Method & Configuration / hyperparameters \\
\midrule

Mean &
Column-wise mean imputation. \\

Median &
Column-wise median imputation. \\

Half-min &
Missing entries are replaced by one half of the observed column minimum when the minimum is positive; otherwise a floor value of \(10^{-6}\) is used. \\

kNN &
\(k=10\) nearest neighbors with distance-based weighting. \\

MissForest &
Iterative Random Forest imputation using
100 trees, maximum iterations \(=10\), median initialization, and \texttt{n\_jobs=-1}. \\

SVT\_OGpaper &
The original SVT method with a fixed step size and \(\delta=\min(1.2/p,2)\), zero initialization,
skip-ahead dual initialization, maximum iterations \(=1000\), and tolerance \(=10^{-5}\). \\

SVT &
Proposed adaptive SVT variant with
\(\tau=5\max(n_1,n_2)\), column-mean warm start,
initial step size \(\delta_0=\min(1.2p,2p)\),
step-size cap \(\min(2p,2)\), contraction factor \(0.9\),
expansion factor \(1.05\), maximum iterations \(=1000\), and tolerance \(=10^{-5}\). \\

NF\_OGpaper &
Two-stage method using SVT\_OGpaper for initialization, followed by one
column-wise Random Forest refinement with 100 trees, and \texttt{n\_jobs=-1}. \\

NuclearForest &
Two-stage method using the proposed adaptive SVT variant for initialization, followed by one column-wise Random Forest refinement with 100 trees, and \texttt{n\_jobs=-1}. \\

SoftImpute &
SoftImpute with data-scaled regularization
\(\lambda=0.01\lambda_0\), where \(\lambda_0\) is the largest singular value of the zero-filled matrix; rank cap \(\min(n,p)\), maximum iterations \(=1000\), and convergence threshold \(=10^{-5}\). \\

SoftForest &
Two-stage method using SoftImpute initialization with
\(\lambda=0.01\lambda_0\), rank cap \(\min(n,p)\), maximum iterations \(=1000\), and threshold \(=10^{-5}\), followed by one column-wise Random Forest refinement with 100 trees, and \texttt{n\_jobs=-1}. \\

\bottomrule
\end{tabular}
\end{table}

\begin{table}[H]
\centering
\caption{Global experimental settings}
\label{tab:global-settings}

\begin{tabular}{ll}
\toprule
Setting & Value \\
\midrule
Number of repeated runs & 10 \\
Random seeds & \(42, 43, \ldots, 51\), assigned sequentially to the 10 repeated runs \\

Missingness rates & \(10\%,20\%,30\%,40\%,50\%,60\%,70\%,80\%,90\%\) \\

\bottomrule
\end{tabular}
\end{table}

\begin{table}[H]
\centering
\caption{Total wall-clock runtime.}
\label{tab:additional-runtime-summary}
\begin{tabular}{lcc}
\toprule
Experiment & Total runtime (min) & Total runtime (h) \\
\midrule
Metabolomics dataset under MCAR/MAR missingness & 192.5 & 3.21 \\
Metabolomics dataset under MNAR missingness & 192.0 & 3.20 \\
Housing dataset under MCAR missingness & 12.9 & 0.22 \\
Housing dataset under MAR missingness & 13.0 & 0.22 \\
\bottomrule
\end{tabular}
\end{table}

\section{Additional experimental results}

\label{app:additionalplot}

\subsection{Results including the original SVT method}
\label{app:og-paper-results}

\begin{figure}[H]
    \centering

    \begin{subfigure}[b]{0.49\textwidth}
        \centering
        \includegraphics[width=\textwidth]{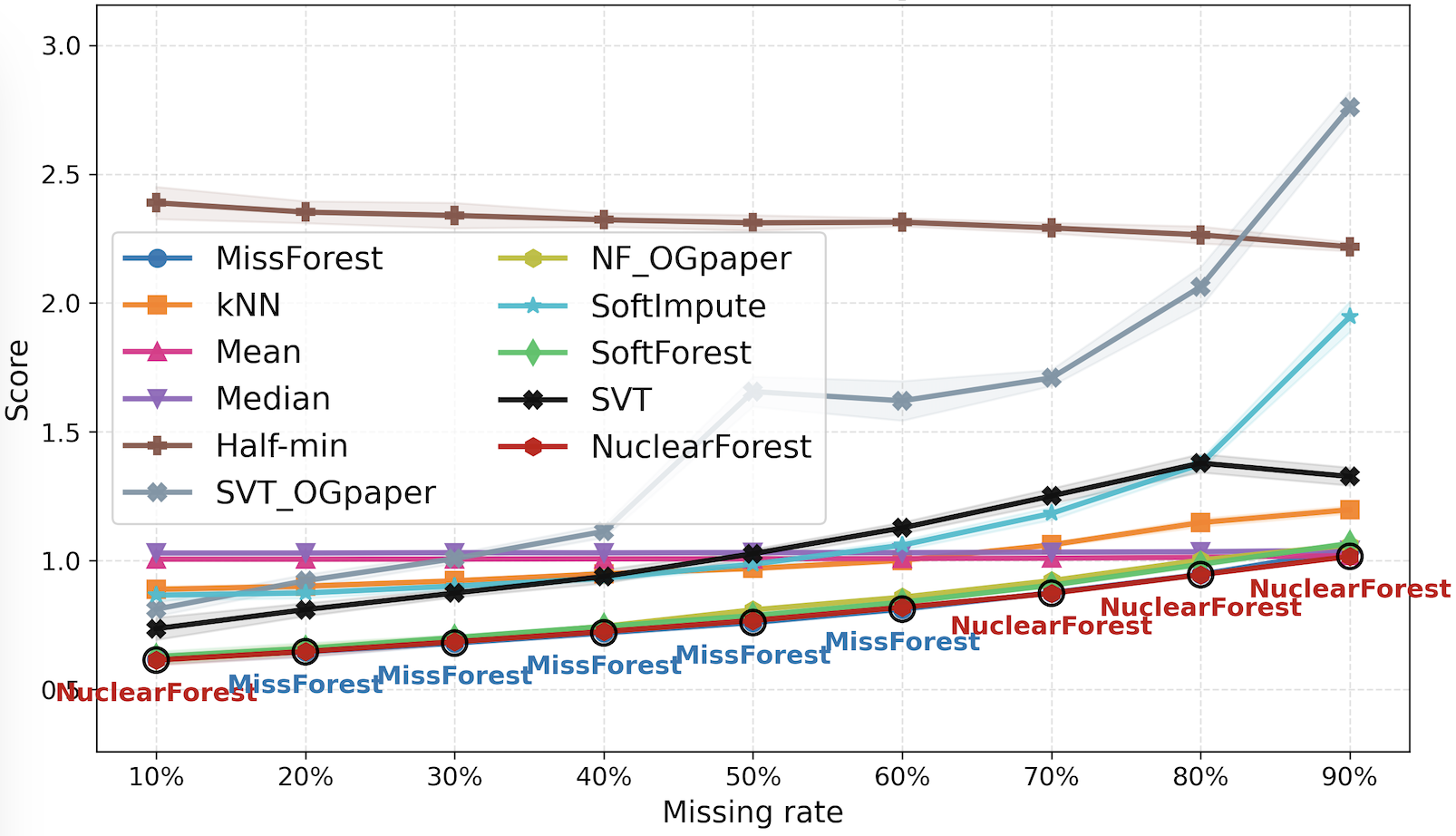}
        \caption{NRMSE}
        \label{fig:subo11}
    \end{subfigure}
    \hfill
    \begin{subfigure}[b]{0.49\textwidth}
        \centering
        \includegraphics[width=\textwidth]{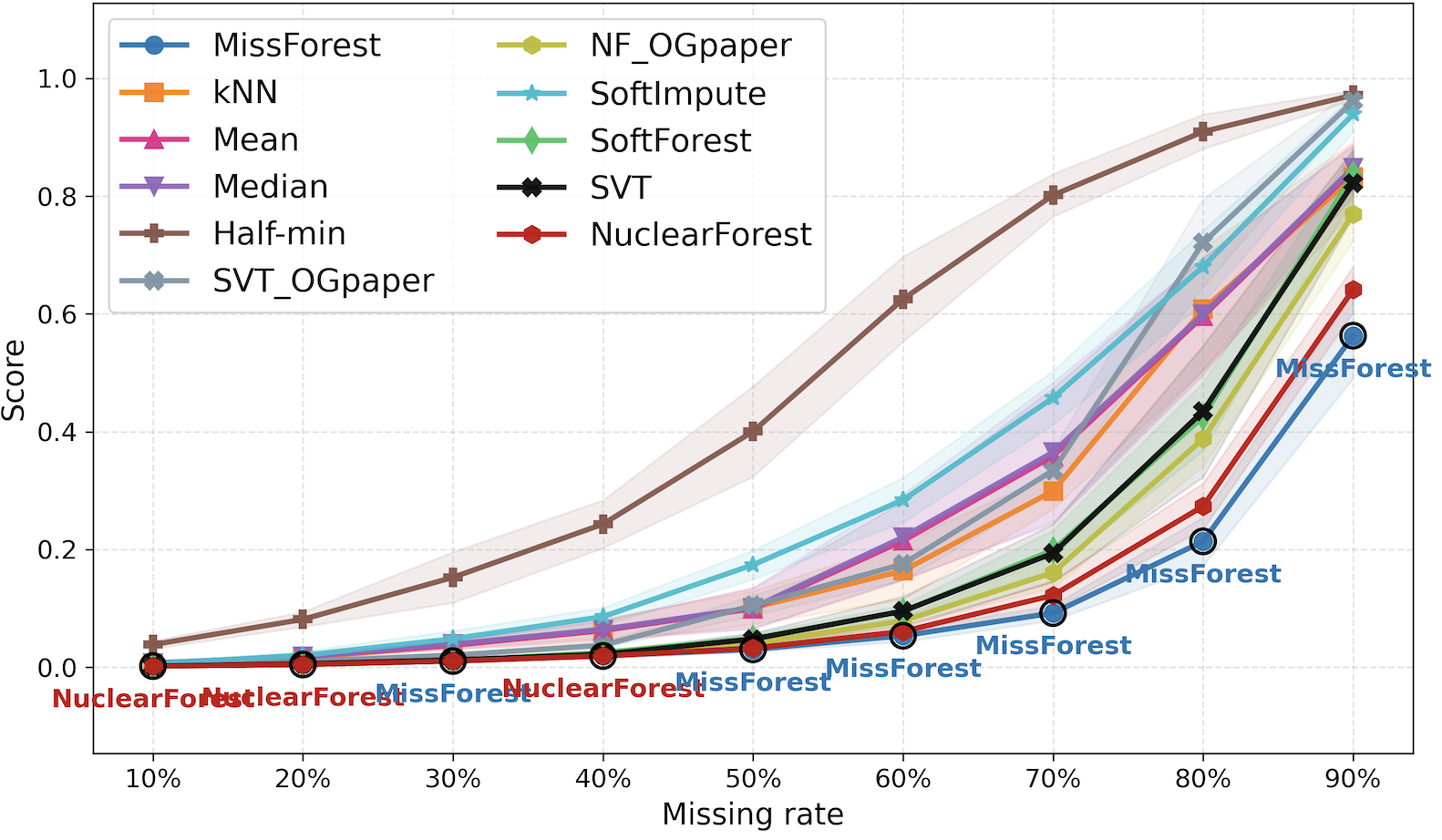}
        \caption{PCA Procrustes distance}
        \label{fig:subo12}
    \end{subfigure}

    \begin{subfigure}[b]{0.49\textwidth}
        \centering
        \includegraphics[width=\textwidth]{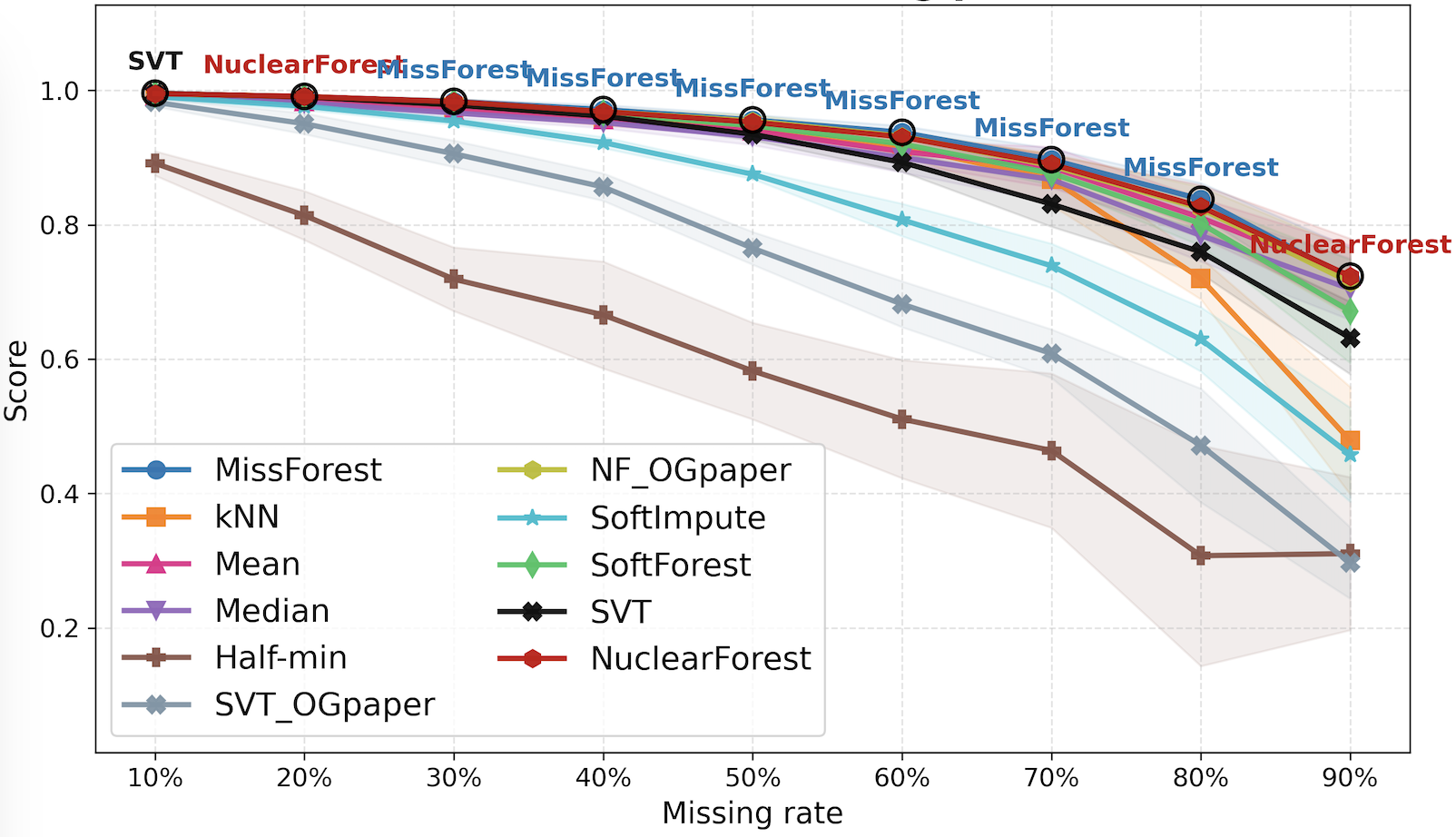}
        \caption{Pearson log-\(p\) correlation}
        \label{fig:subo13}
    \end{subfigure}
    \hfill
    \begin{subfigure}[b]{0.49\textwidth}
        \centering
        \includegraphics[width=\textwidth]{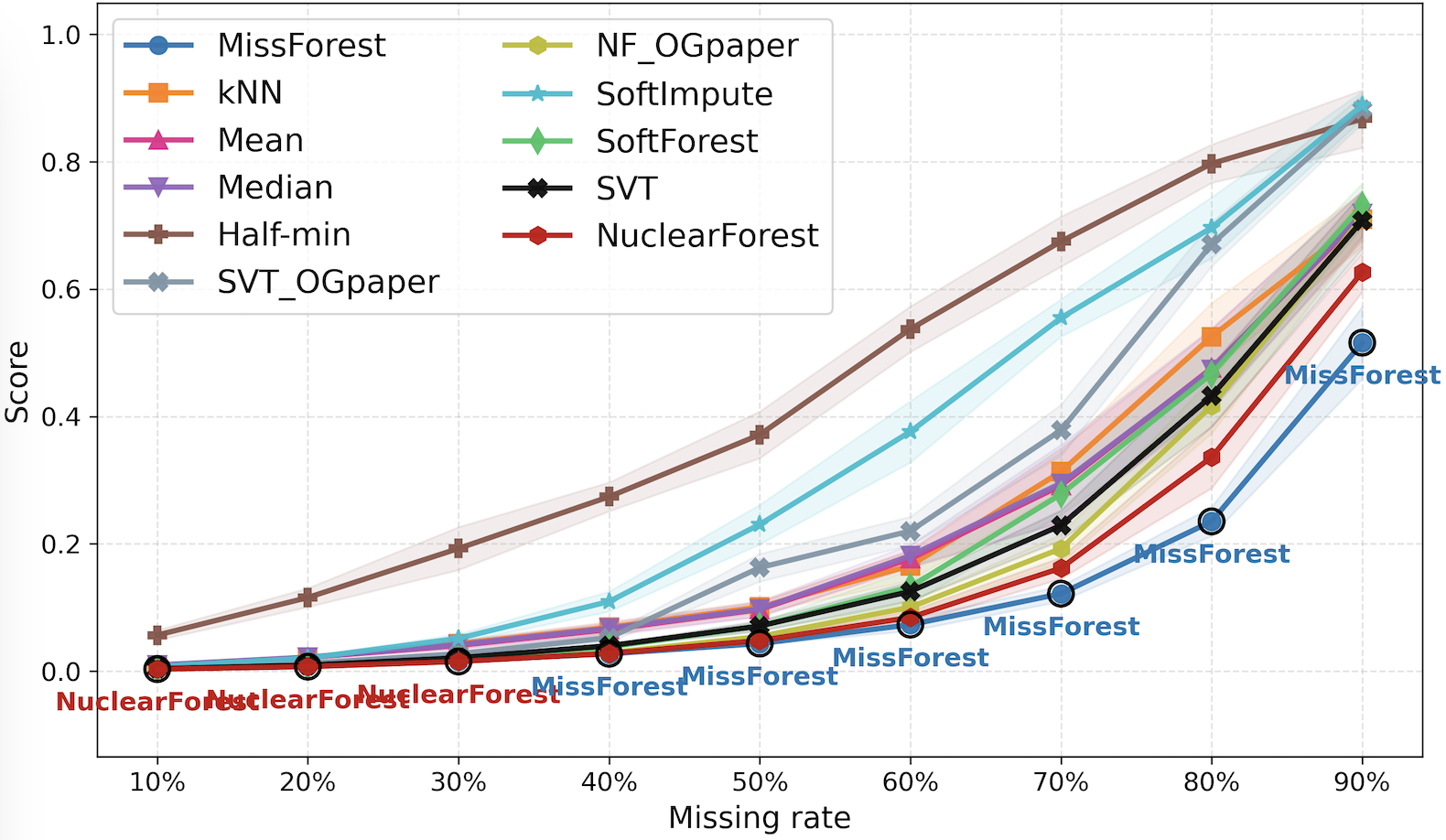}
        \caption{PLS Procrustes distance}
        \label{fig:subo14}
    \end{subfigure}

    \begin{subfigure}[b]{0.49\textwidth}
        \centering
        \includegraphics[width=\textwidth]{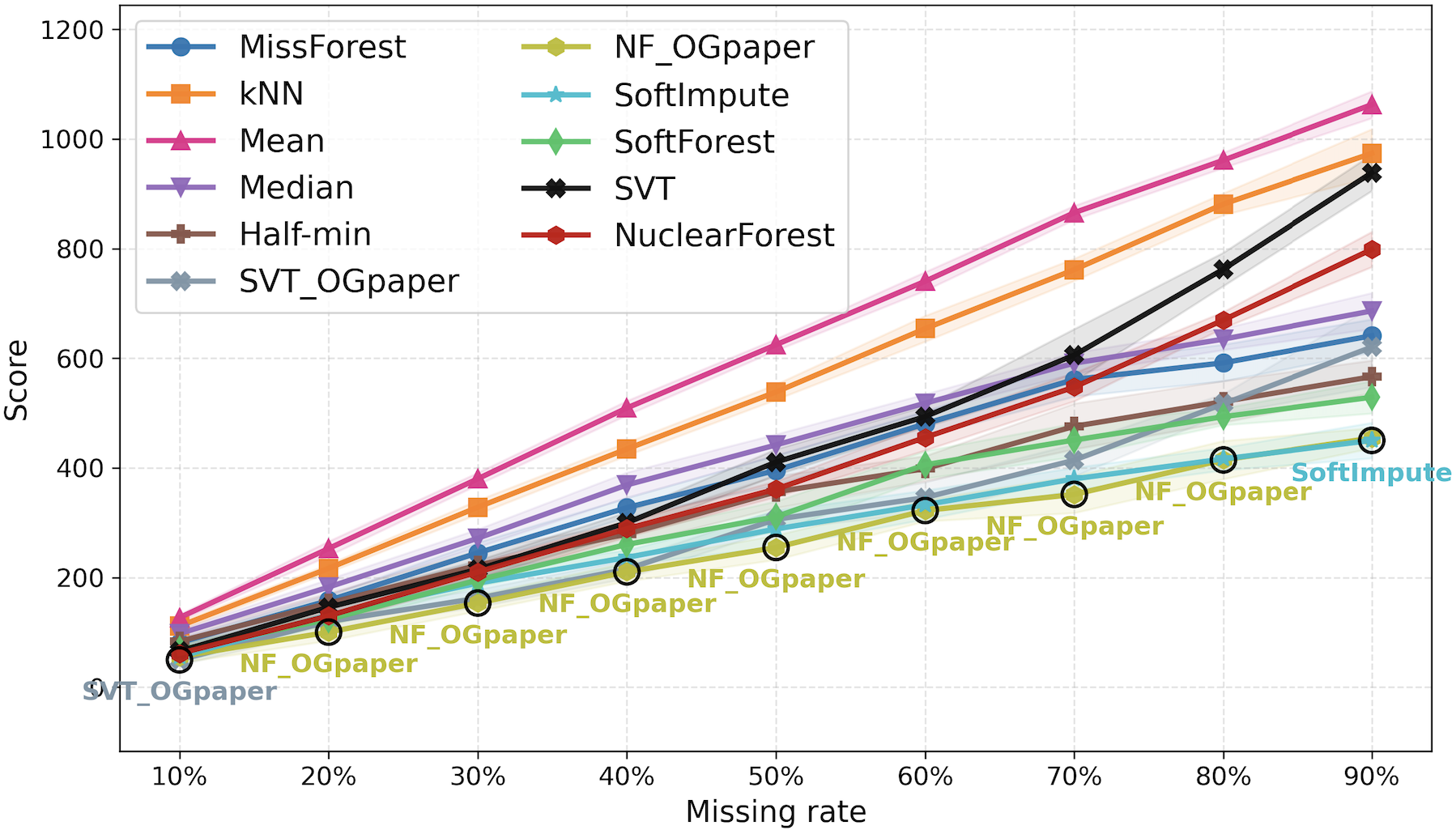}
        \caption{SOR}
        \label{fig:subo21}
    \end{subfigure}
    \hfill
    \begin{subfigure}[b]{0.49\textwidth}
        \centering
        \includegraphics[width=\textwidth]{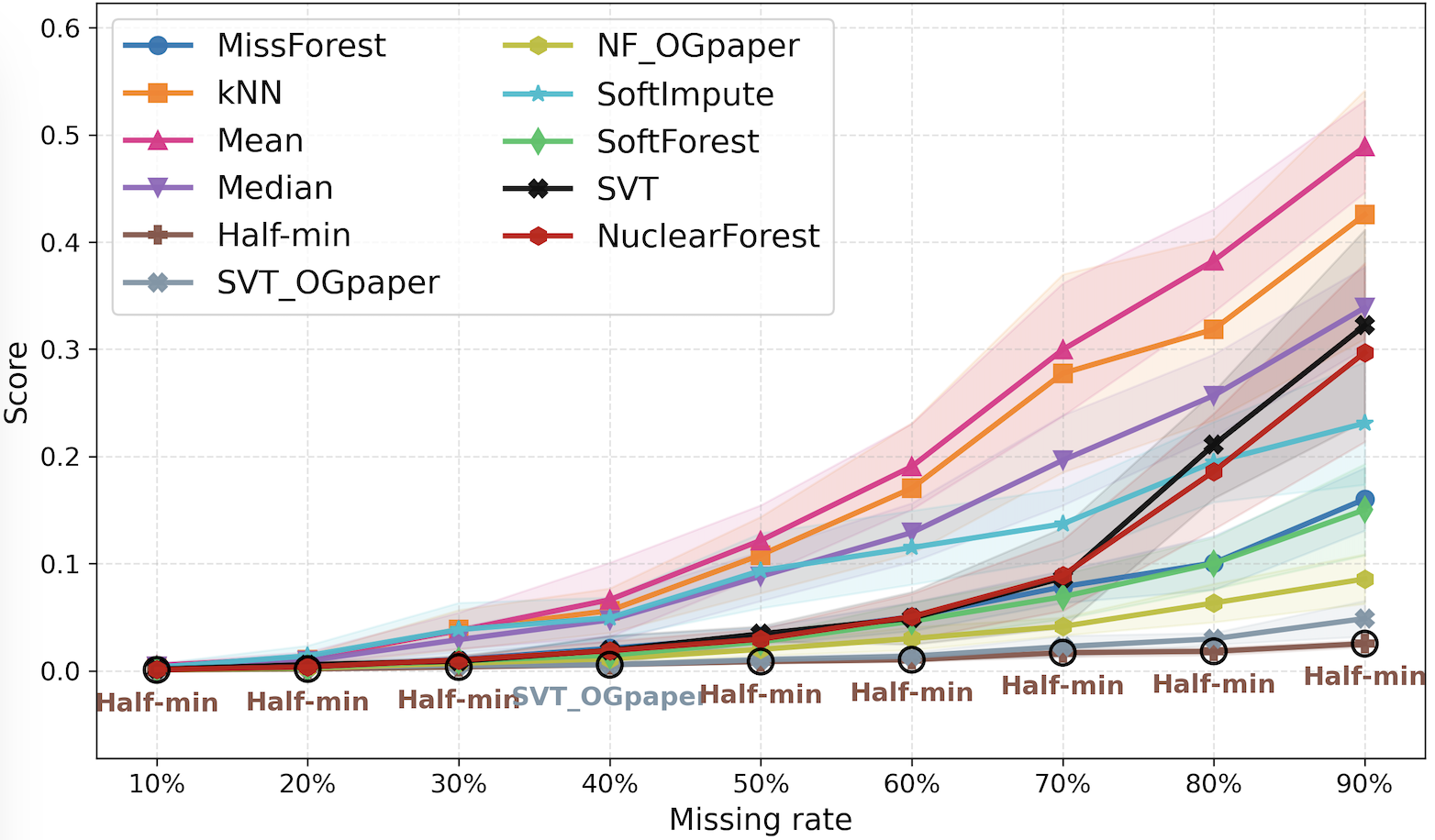}
        \caption{PCA Procrustes distance}
        \label{fig:subo22}
    \end{subfigure}

    \begin{subfigure}[b]{0.49\textwidth}
        \centering
        \includegraphics[width=\textwidth]{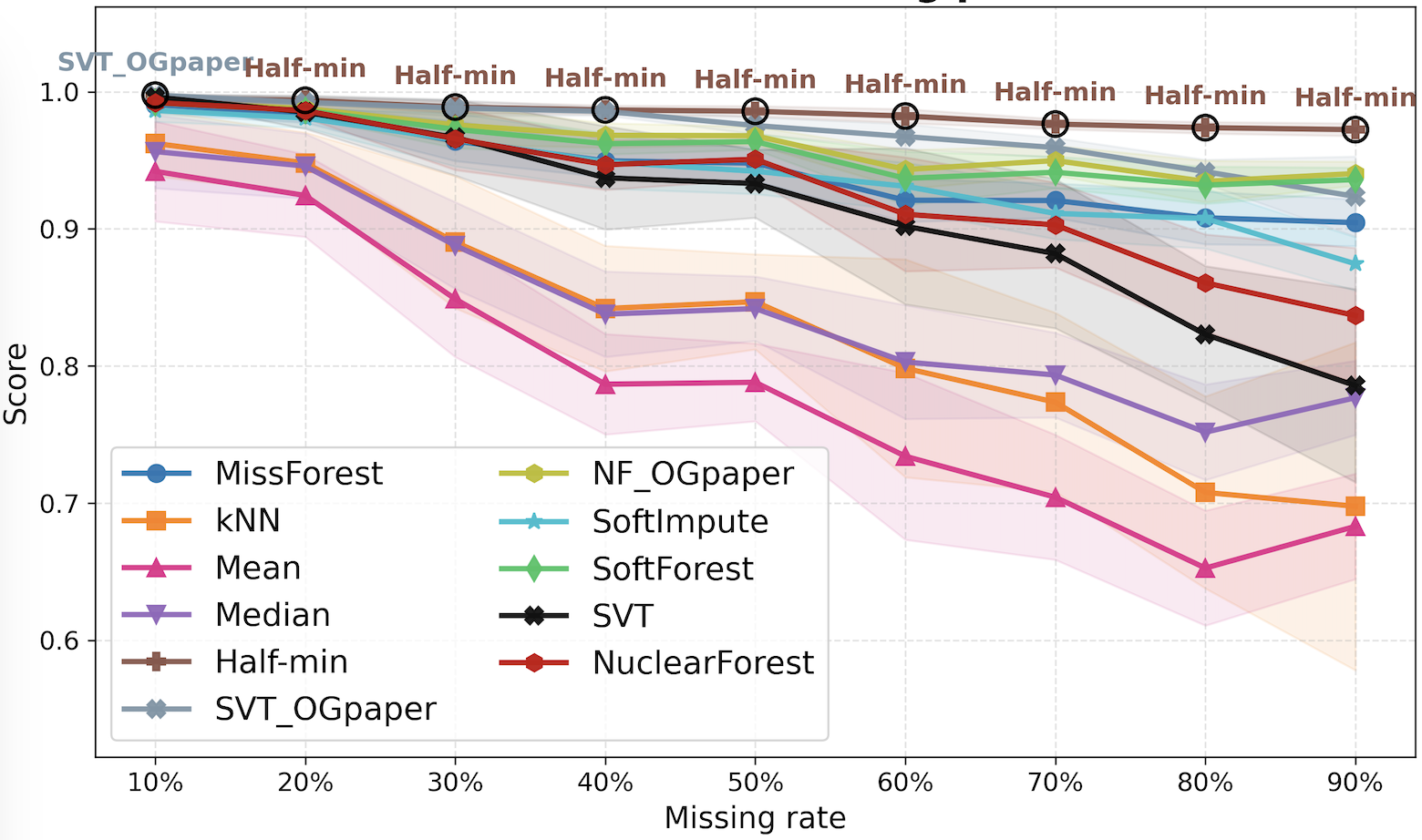}
        \caption{Pearson log-\(p\) correlation}
        \label{fig:subo23}
    \end{subfigure}
    \hfill
    \begin{subfigure}[b]{0.49\textwidth}
        \centering
        \includegraphics[width=\textwidth]{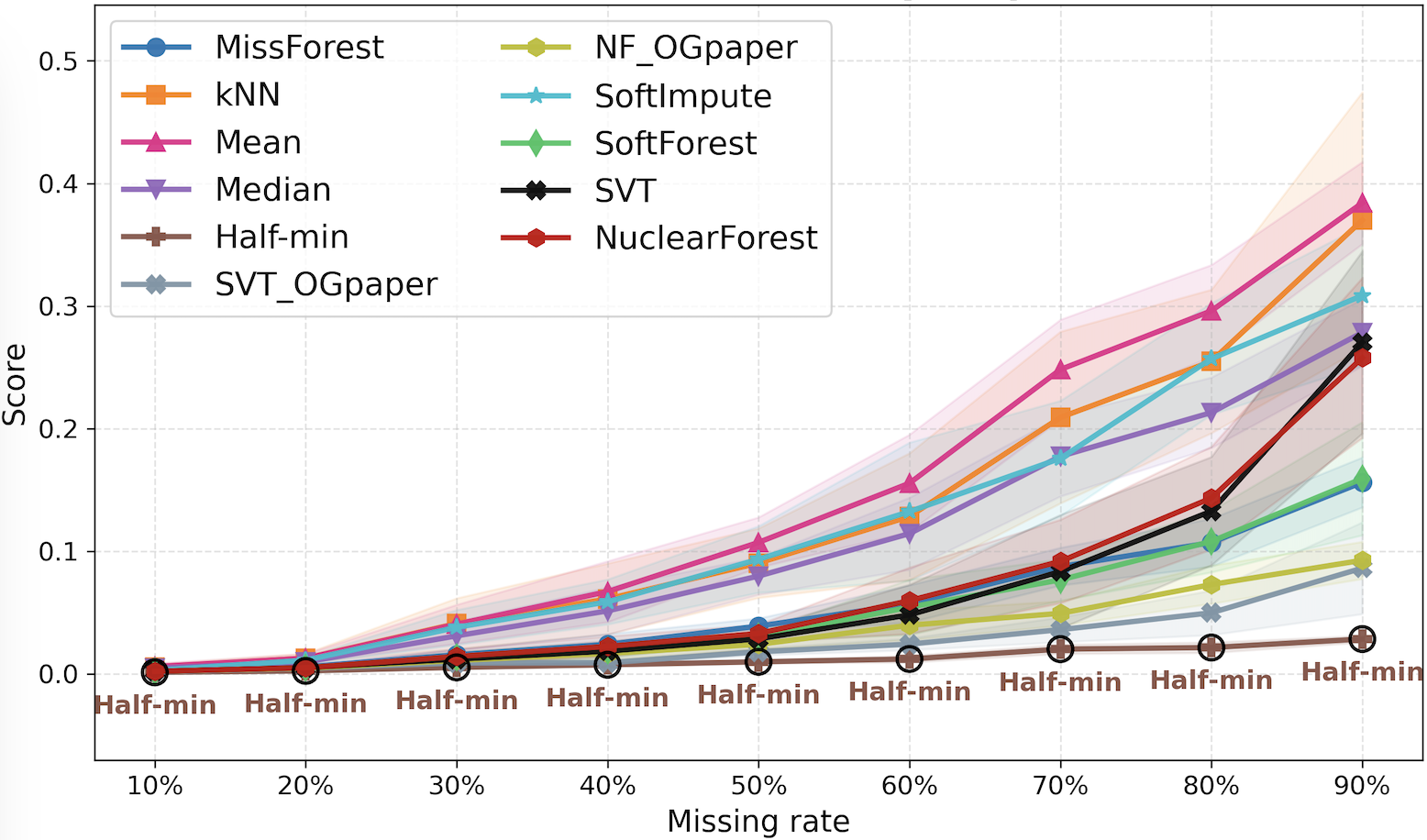}
        \caption{PLS Procrustes distance}
        \label{fig:subo24}
    \end{subfigure}

    \caption{Comparison of SVT\_OGpaper and NF\_OGpaper with the adapted SVT and NuclearForest methods on the metabolomics dataset under MCAR/MAR (subfigures a--d) and MNAR (subfigures e--h).}

    \label{fig:og}
\end{figure}

\begin{figure}[H]
    \centering

    \begin{subfigure}[b]{0.32\textwidth}
        \centering
        \includegraphics[width=\textwidth]{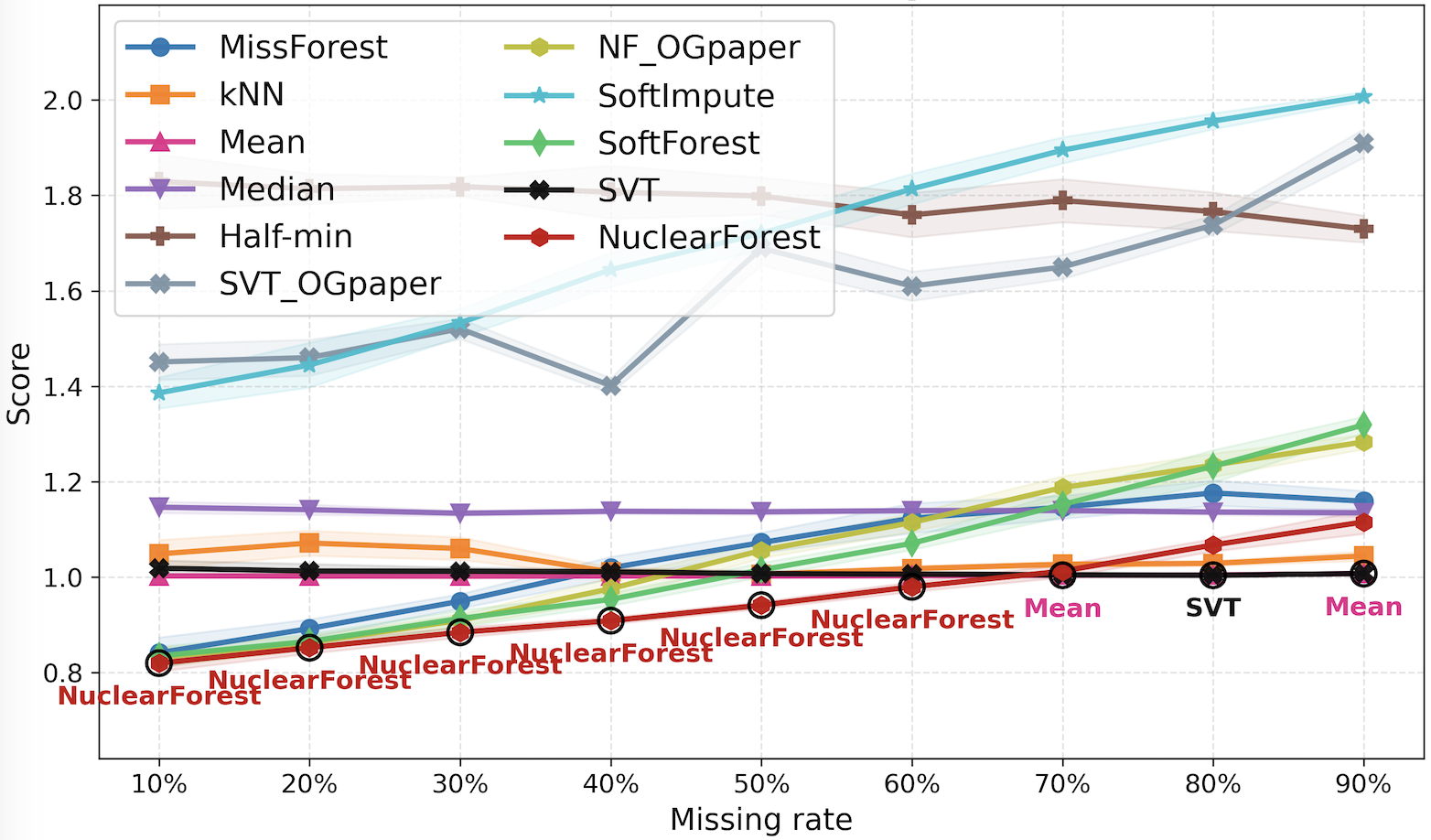}
        \caption{NRMSE}
        \label{fig:subo31}
    \end{subfigure}
    \hfill
    \begin{subfigure}[b]{0.32\textwidth}
        \centering
        \includegraphics[width=\textwidth]{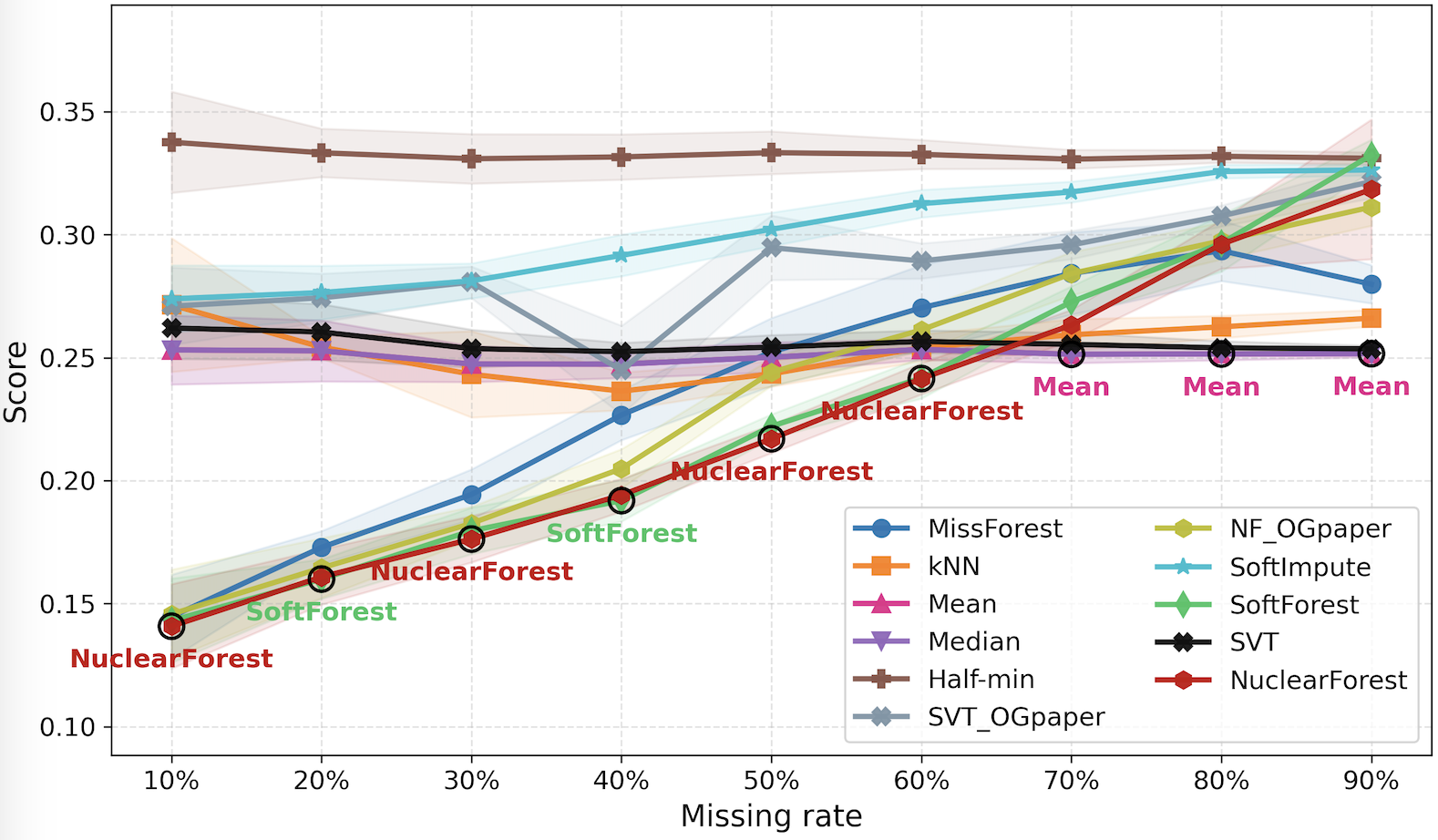}
        \caption{PFC}
        \label{fig:subo32}
    \end{subfigure}
    \hfill
    \begin{subfigure}[b]{0.32\textwidth}
        \centering
        \includegraphics[width=\textwidth]{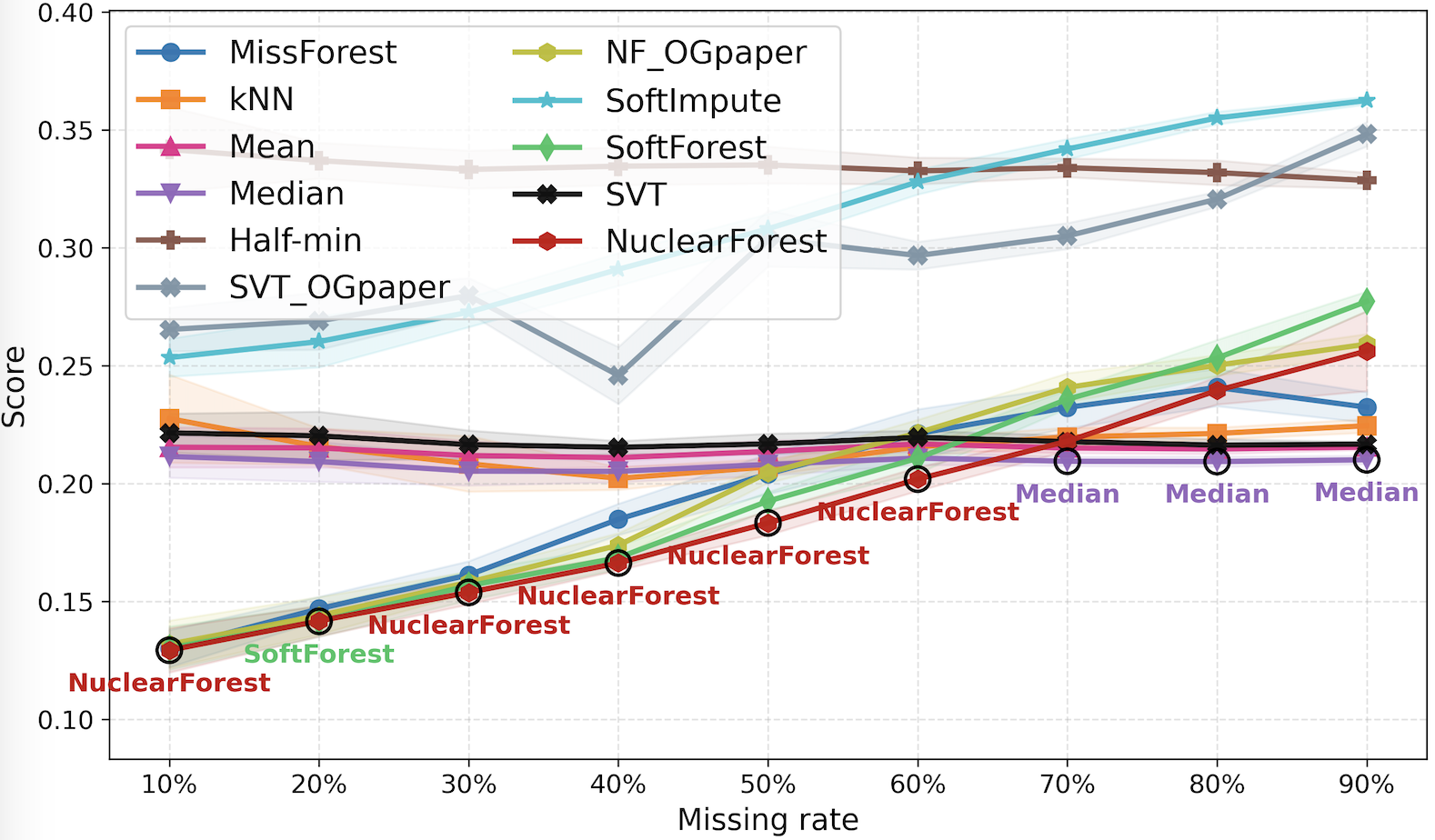}
        \caption{Gower's distance}
        \label{fig:subo33}
    \end{subfigure}

    \begin{subfigure}[b]{0.32\textwidth}
        \centering
        \includegraphics[width=\textwidth]{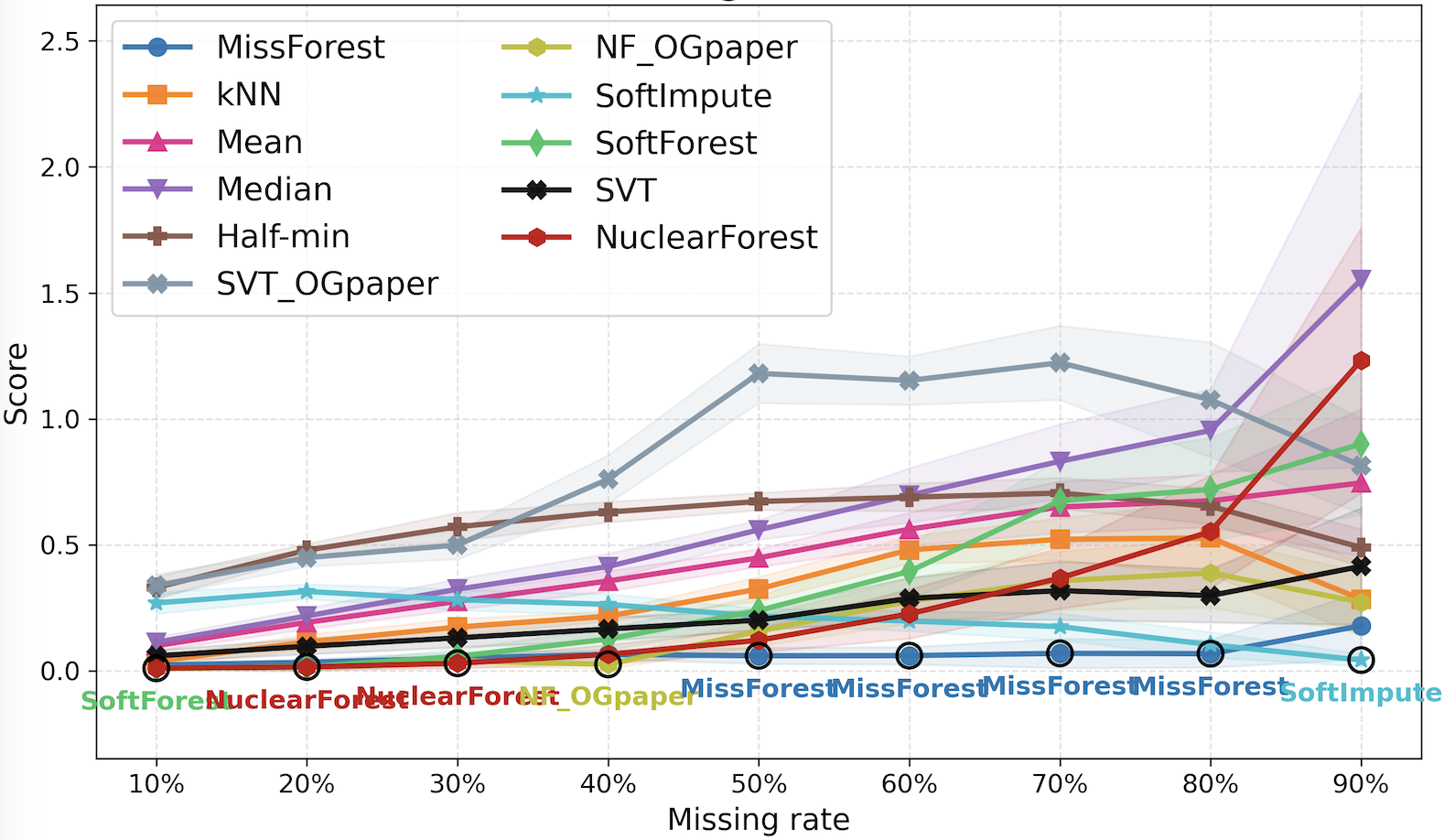}
        \caption{Predictive R\textsuperscript{2} degradation}
        \label{fig:subo34}
    \end{subfigure}
    \hfill
    \begin{subfigure}[b]{0.32\textwidth}
        \centering
        \includegraphics[width=\textwidth]{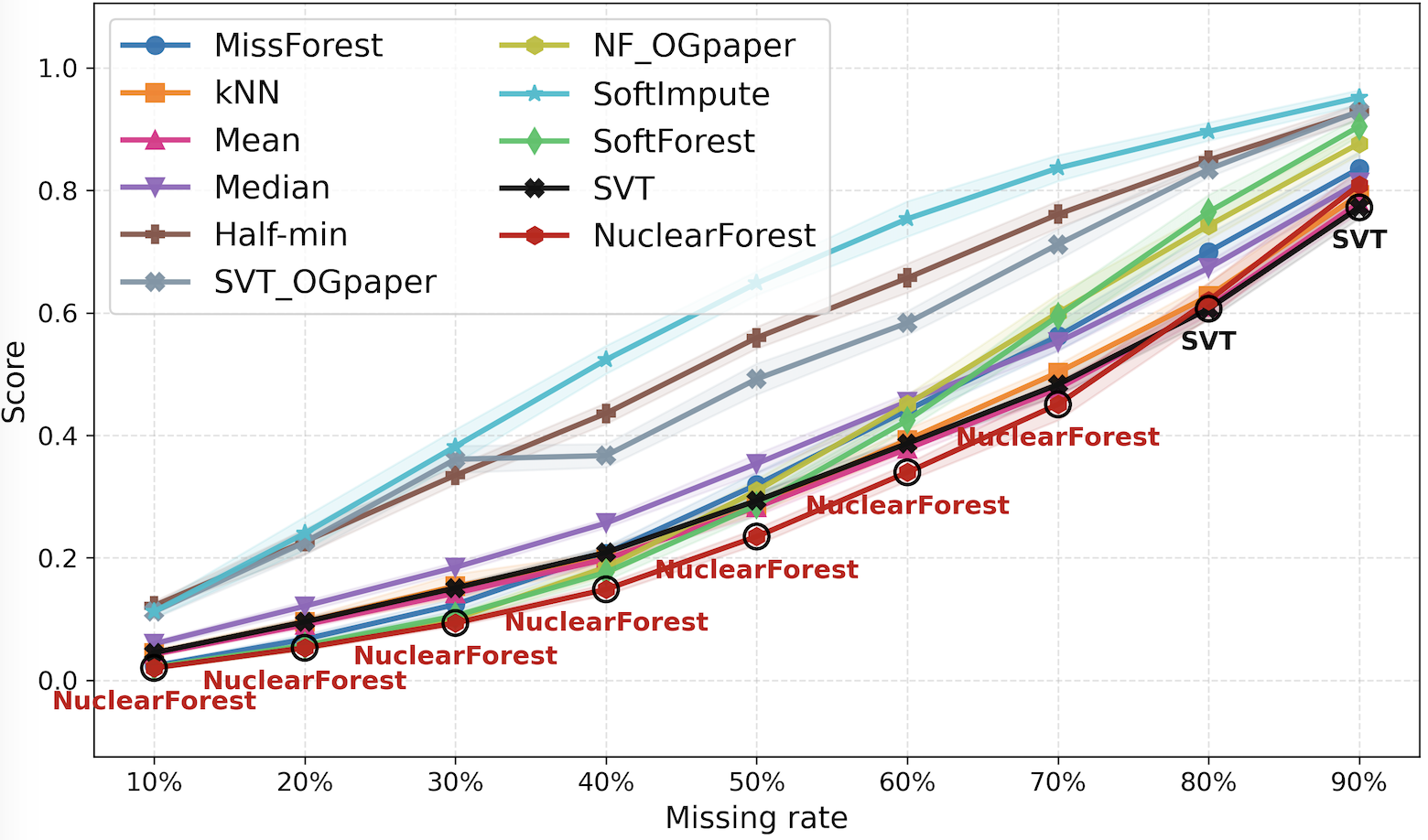}
        \caption{PCA Procrustes distance}
        \label{fig:subo35}
    \end{subfigure}
    \hfill
    \begin{subfigure}[b]{0.32\textwidth}
        \centering
        \includegraphics[width=\textwidth]{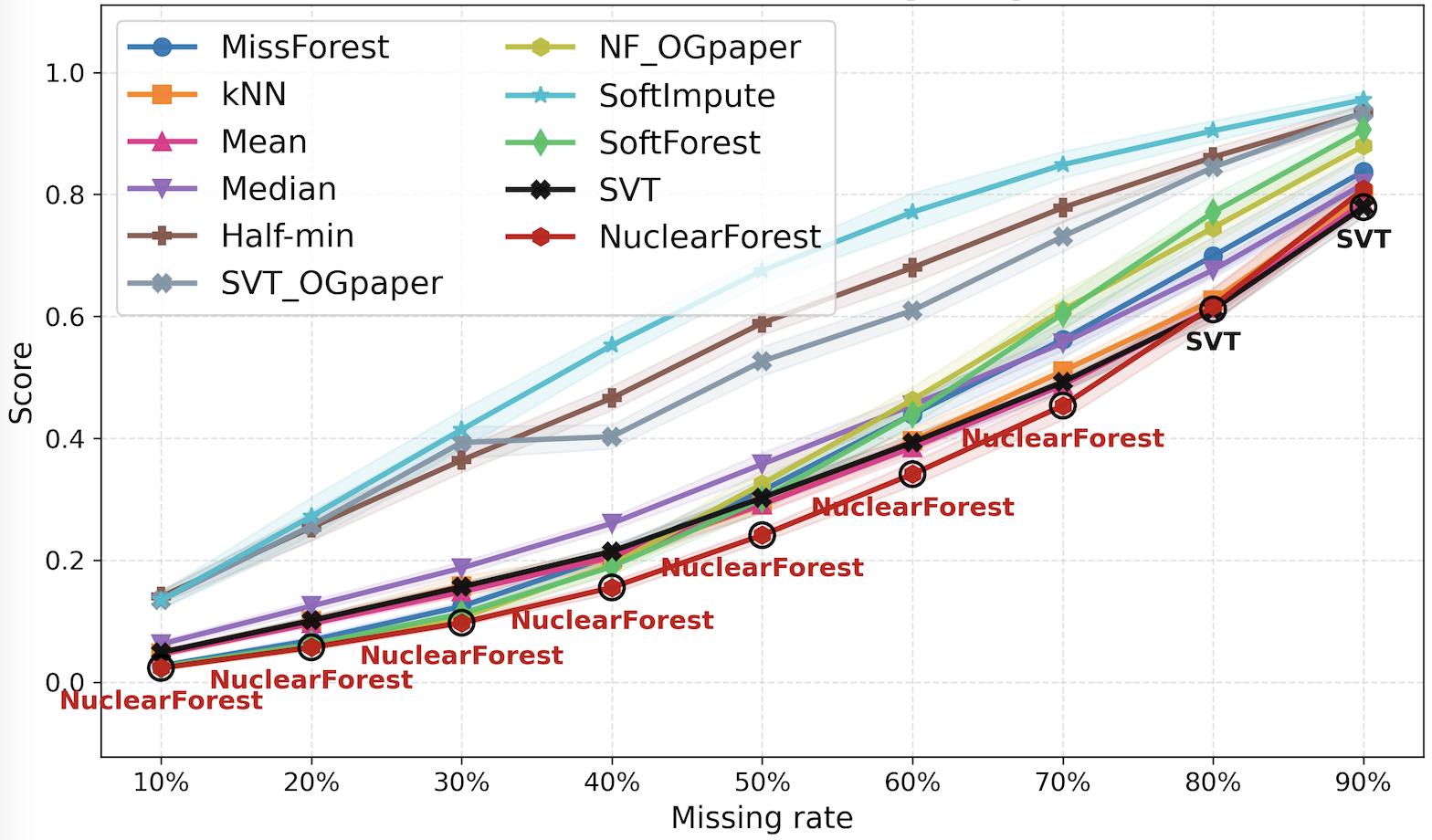}
        \caption{PLS Procrustes distance}
        \label{fig:subo36}
    \end{subfigure}

    \begin{subfigure}[b]{0.32\textwidth}
        \centering
        \includegraphics[width=\textwidth]{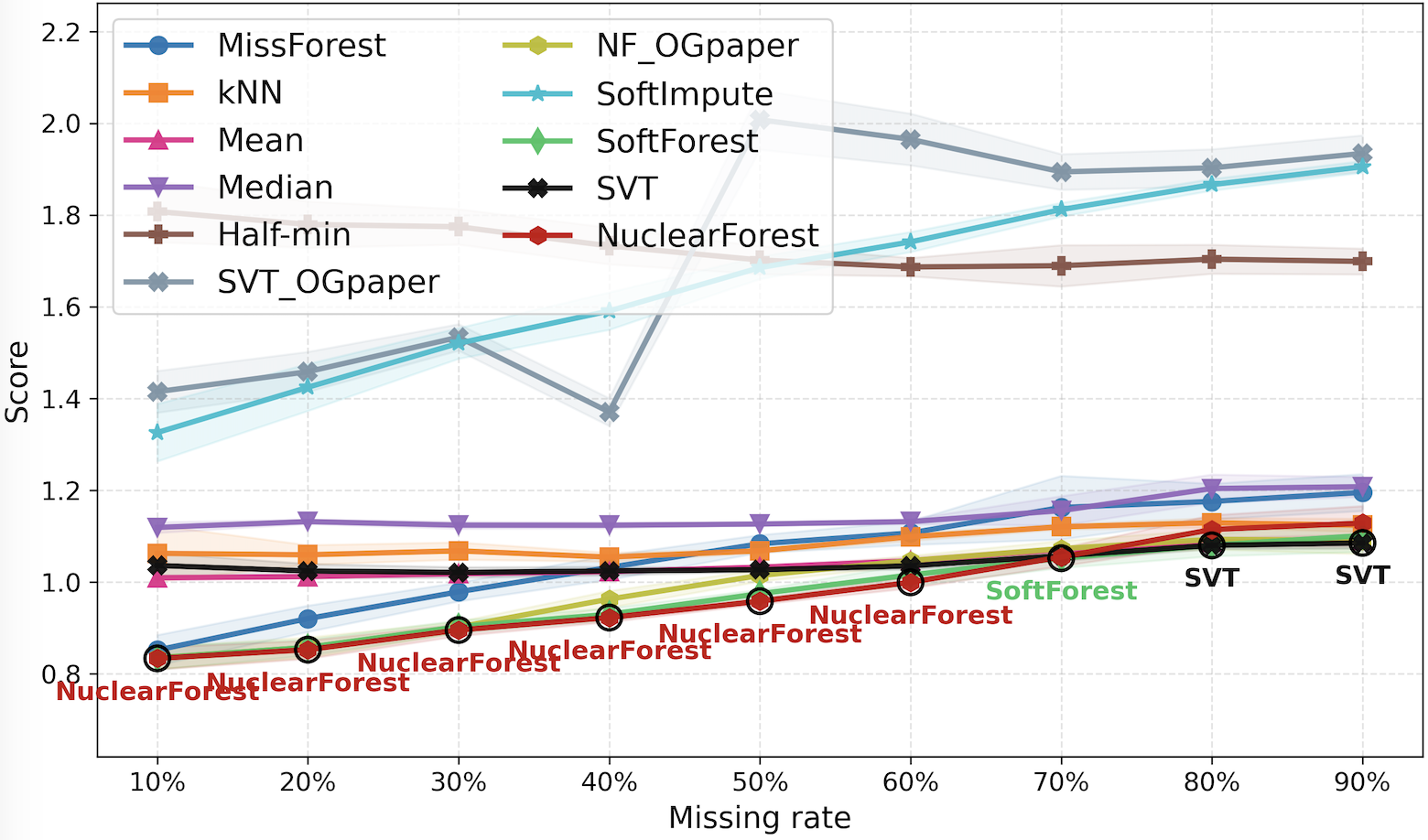}
        \caption{NRMSE}
        \label{fig:subo37}
    \end{subfigure}
    \hfill
    \begin{subfigure}[b]{0.32\textwidth}
        \centering
        \includegraphics[width=\textwidth]{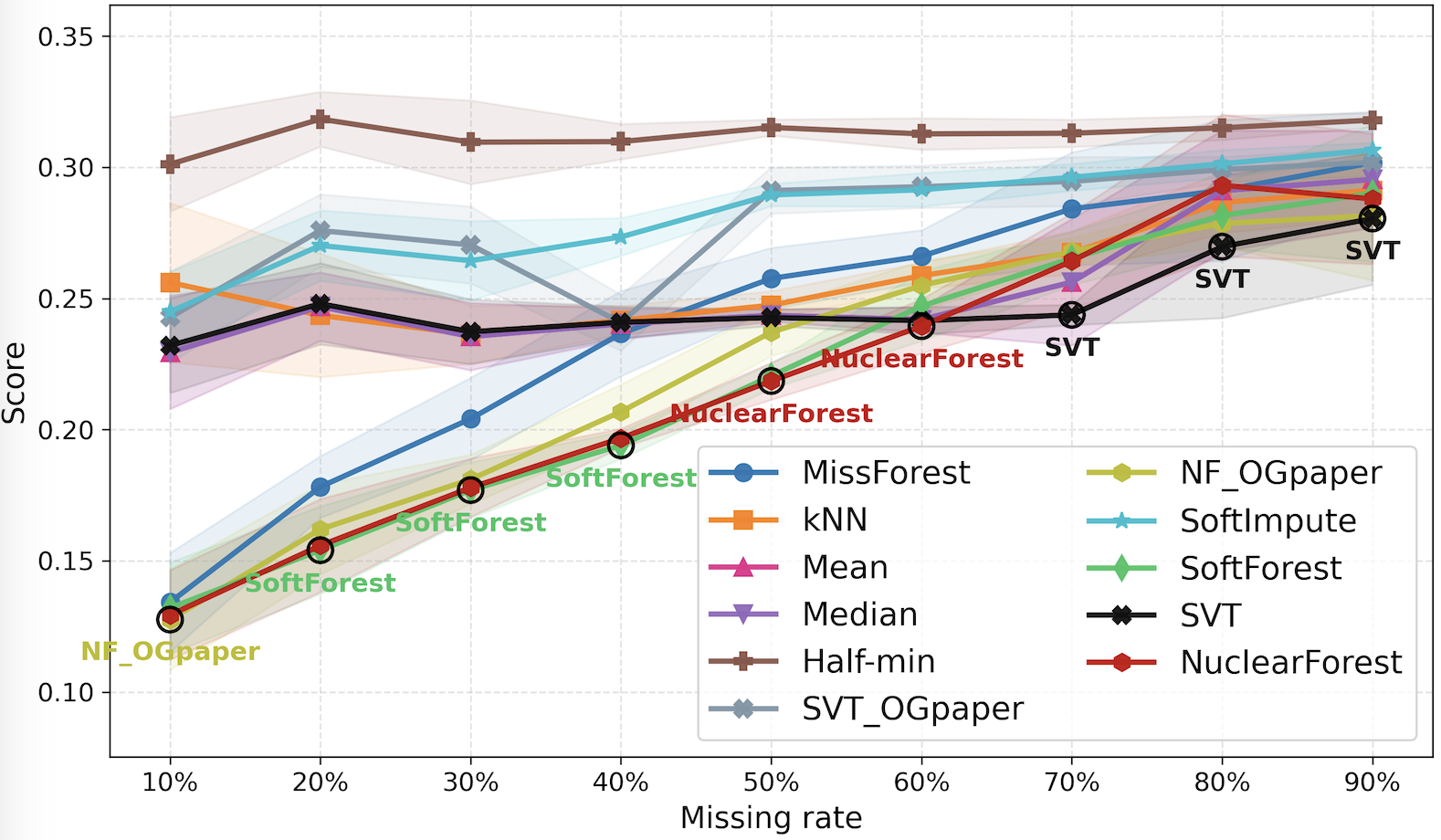}
        \caption{PFC}
        \label{fig:subo38}
    \end{subfigure}
    \hfill
    \begin{subfigure}[b]{0.32\textwidth}
        \centering
        \includegraphics[width=\textwidth]{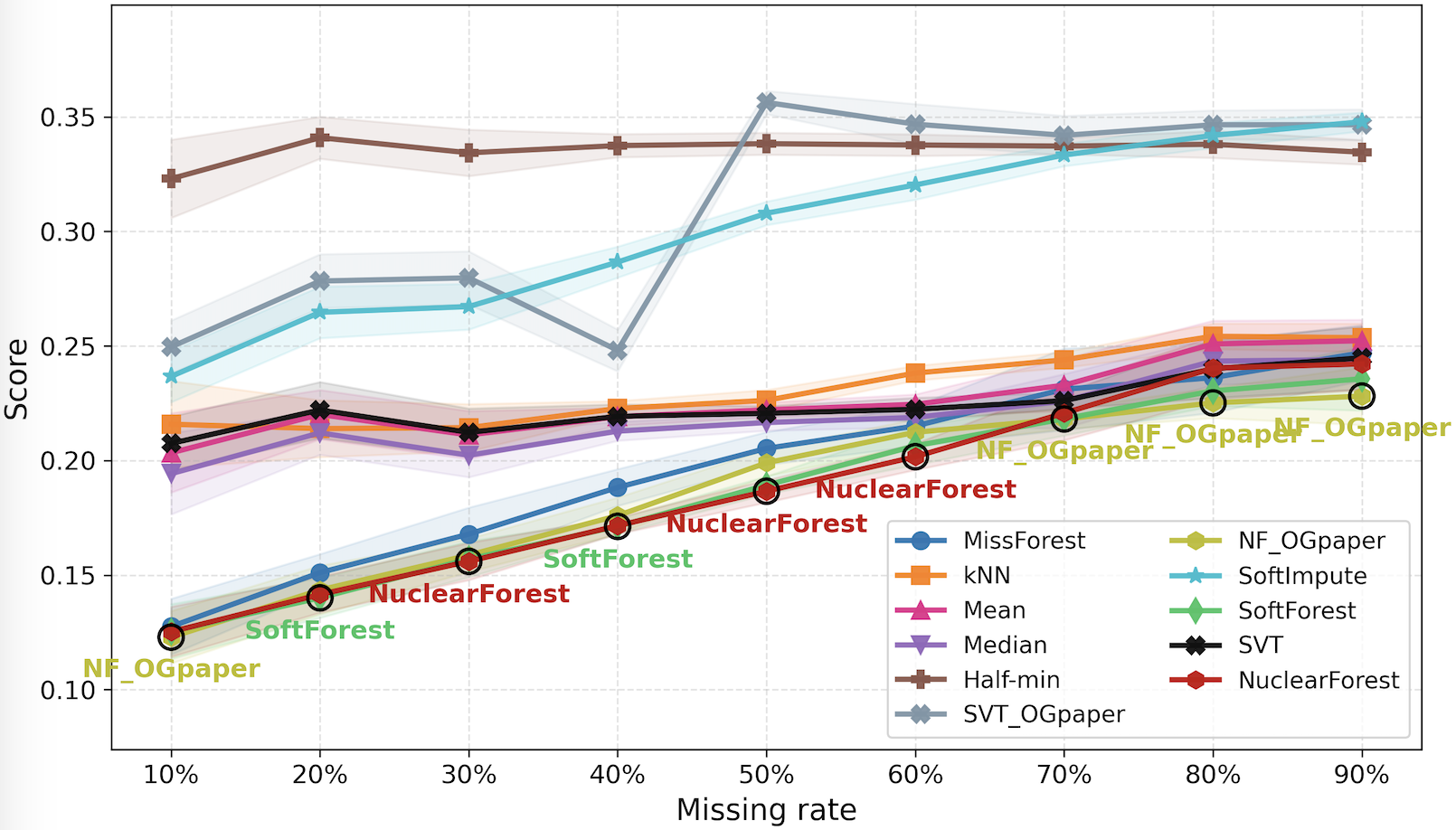}
        \caption{Gower's distance}
        \label{fig:subo39}
    \end{subfigure}

    \begin{subfigure}[b]{0.32\textwidth}
        \centering
        \includegraphics[width=\textwidth]{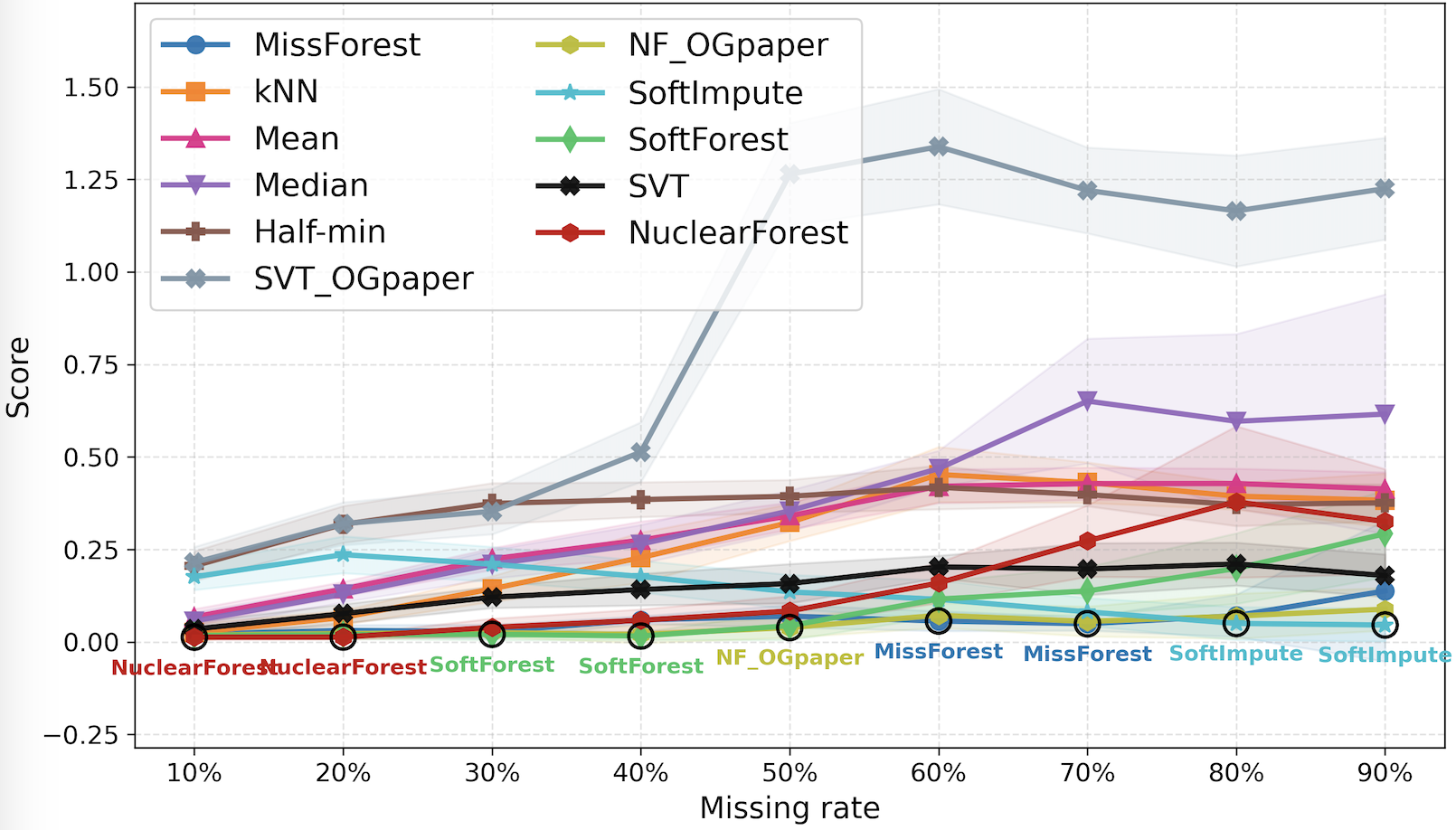}
        \caption{Predictive R\textsuperscript{2} degradation}
        \label{fig:subo40}
    \end{subfigure}
    \hfill
    \begin{subfigure}[b]{0.32\textwidth}
        \centering
        \includegraphics[width=\textwidth]{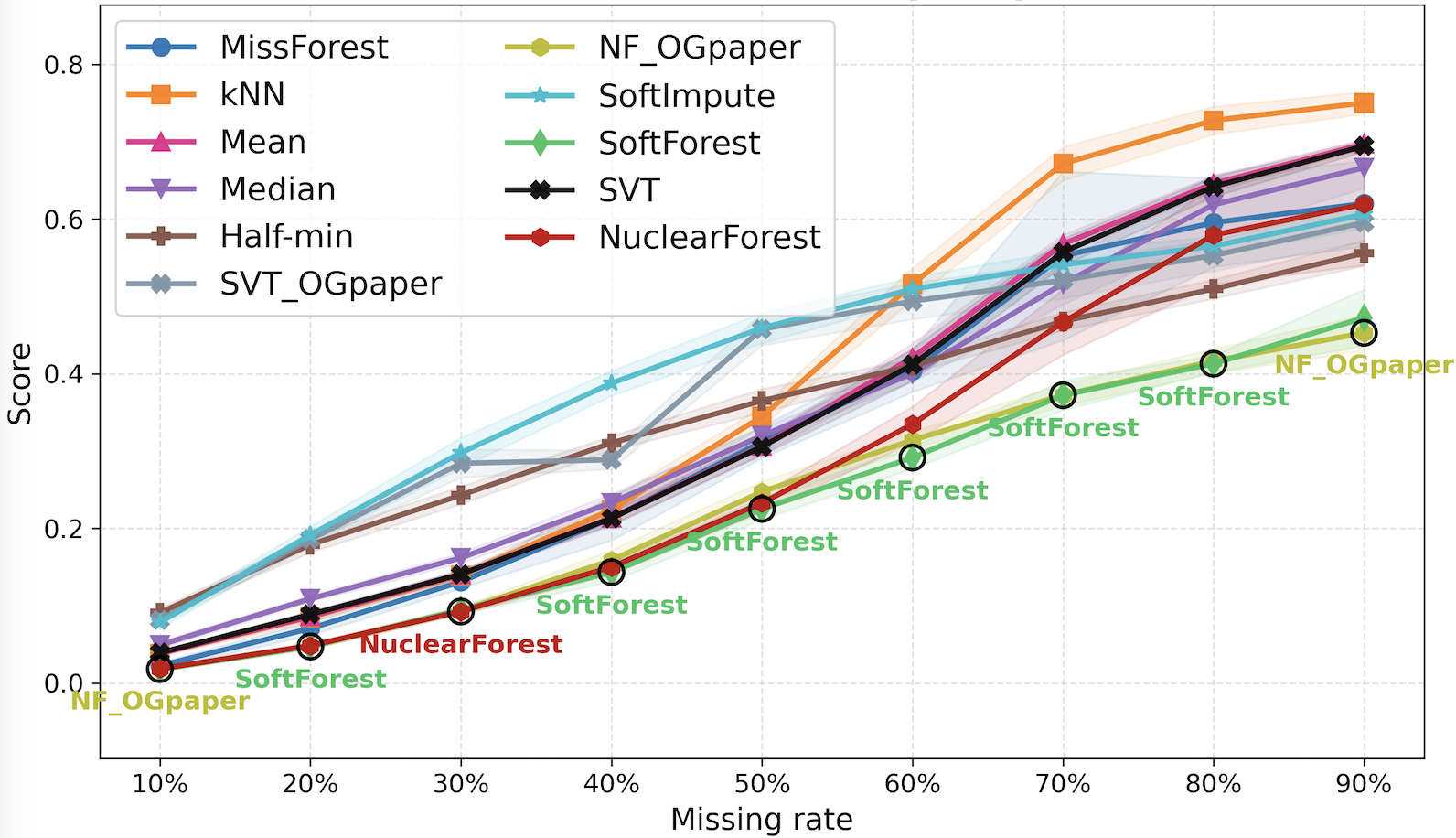}
        \caption{PCA Procrustes distance}
        \label{fig:subo41}
    \end{subfigure}
    \hfill
    \begin{subfigure}[b]{0.32\textwidth}
        \centering
        \includegraphics[width=\textwidth]{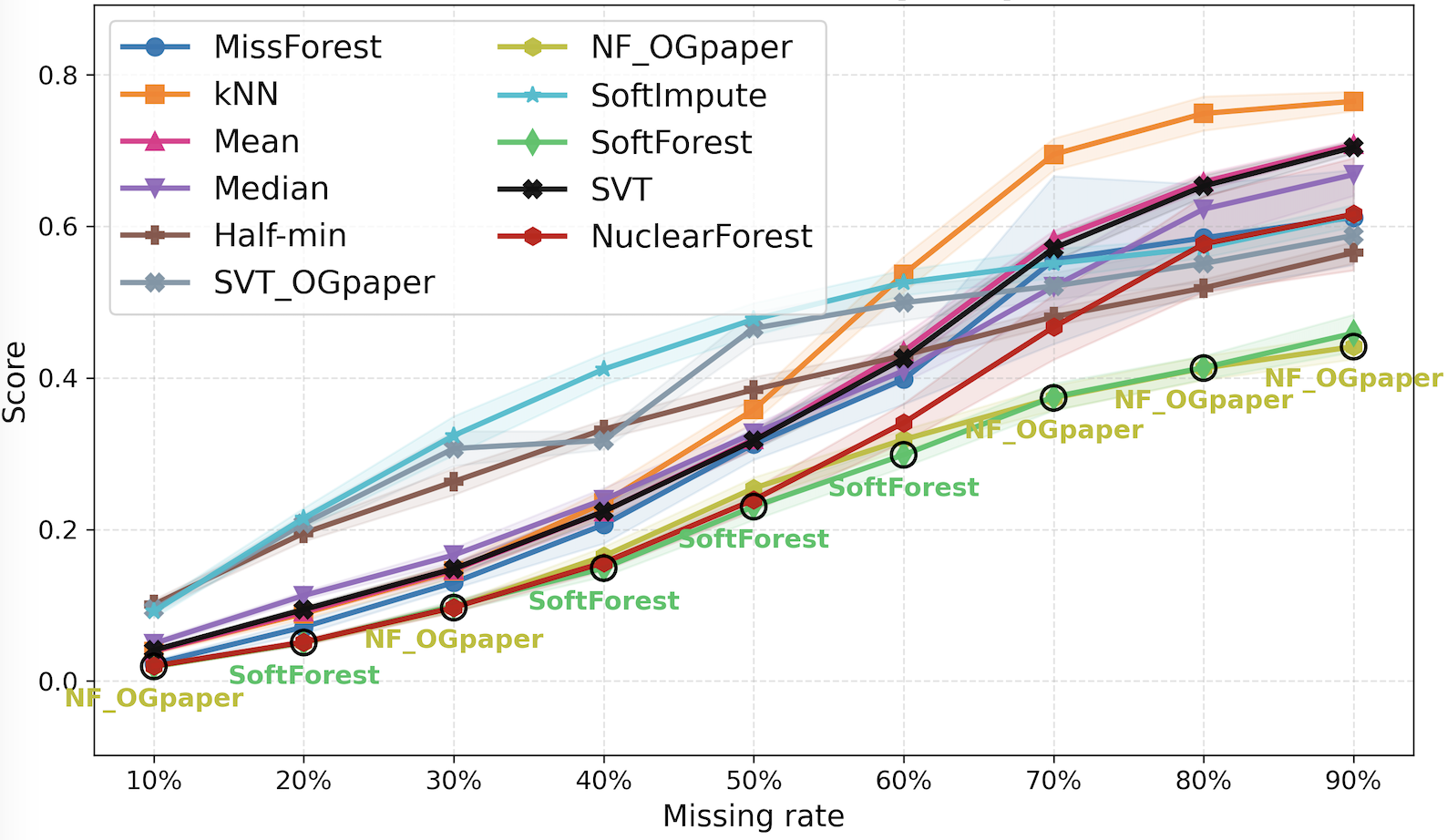}
        \caption{PLS Procrustes distance}
        \label{fig:subo42}
    \end{subfigure}

    \caption{Comparison of SVT\_OGpaper and NF\_OGpaper with the adapted SVT and NuclearForest methods on the housing dataset under MCAR (subfigures a--f) and MAR (subfigures g--l).}

    \label{fig:housingoriginal}
\end{figure}

\begin{figure}[H]
    \centering

    \begin{subfigure}[b]{0.49\textwidth}
        \centering
        \includegraphics[width=\textwidth]{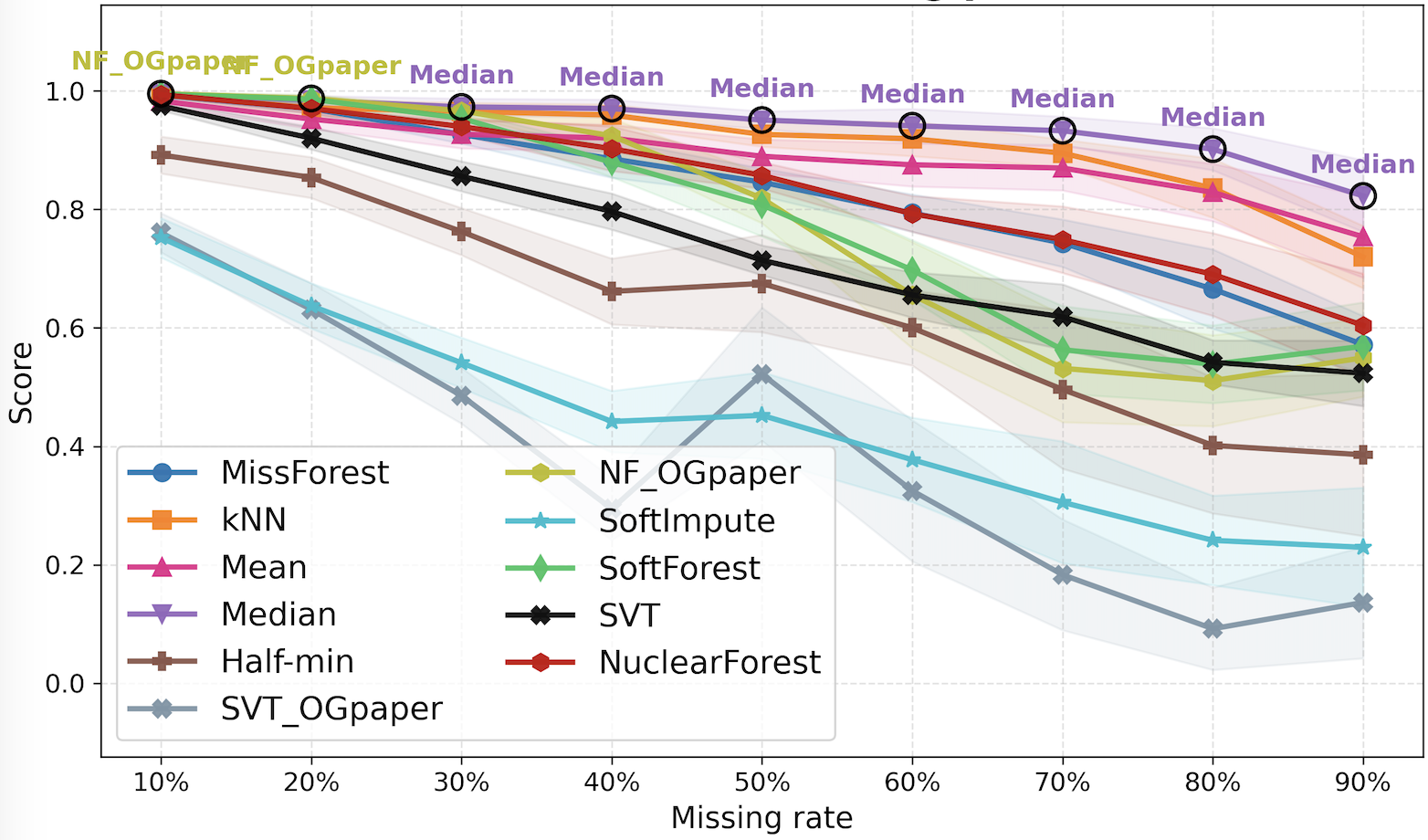}
        \caption{Pearson log-\(p\) correlation}
        \label{fig:subo41}
    \end{subfigure}
    \hfill
    \begin{subfigure}[b]{0.49\textwidth}
        \centering
        \includegraphics[width=\textwidth]{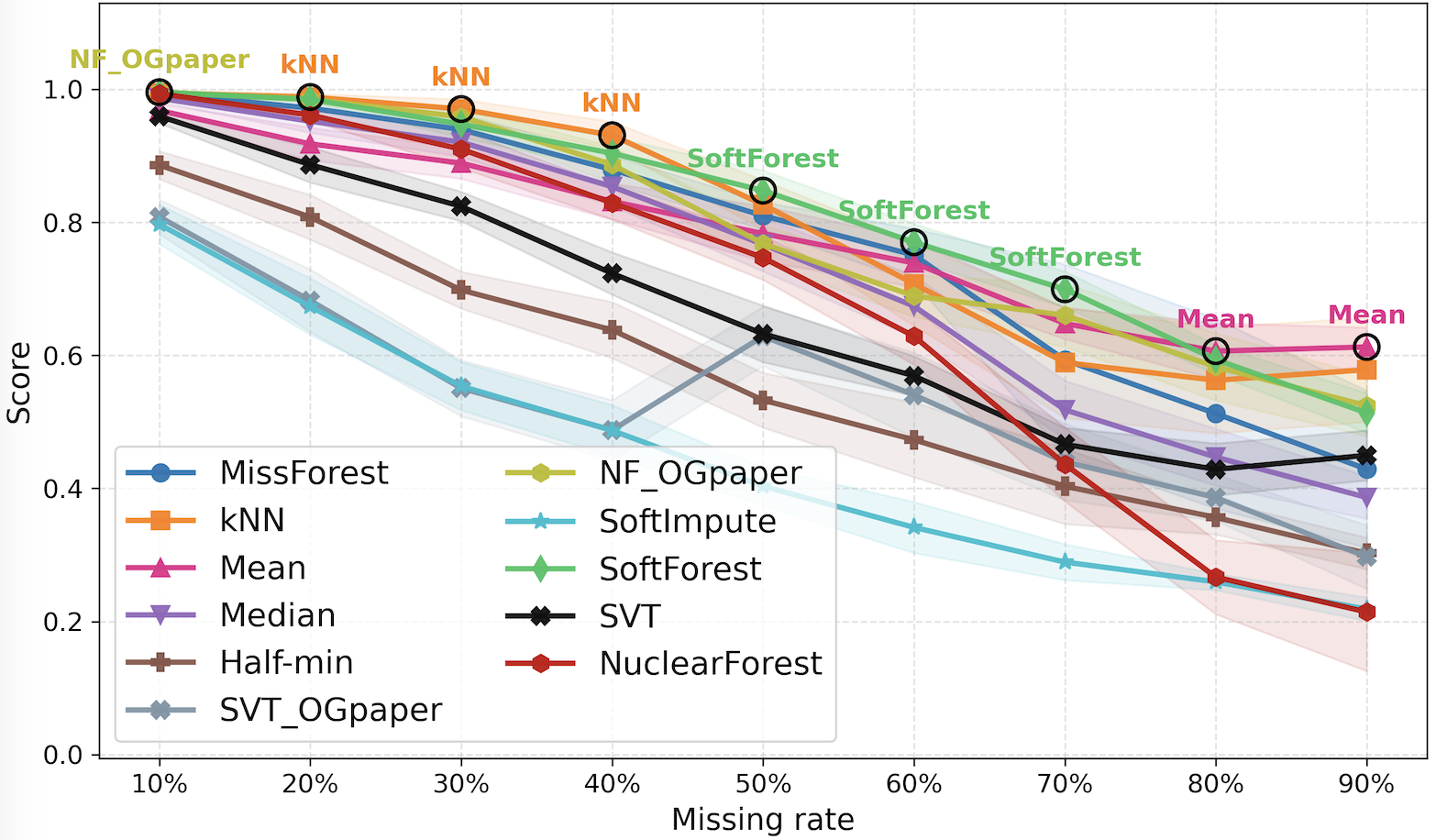}
        \caption{Pearson log-\(p\) correlation}
        \label{fig:subo42}
    \end{subfigure}

    \caption{Comparison of imputation methods on the housing dataset under MCAR (subfigure a) and MAR (subfigure b).}

    \label{fig:housinglogp}
\end{figure}

%\paragraph{Results} As shown in \cref{fig:og}, the adapted versions of SVT and NuclearForest outperform the original implementations under MCAR/MAR. However, under MNAR, the original versions generally achieve better performance than the adapted variants. Under MNAR, the adaptive variants may be disadvantaged because their step-size updates are guided by the observed-entry residual, while the observed entries are themselves biased under informative missingness. Thus, better fitting of the observed set does not necessarily imply better recovery of the missing entries, and may reinforce the missingness-induced bias. The original fixed-step SVT update is less directly coupled to residual-based adaptation, which may explain its stronger performance under MNAR. This is consistent with the strong performance of Half-min, as the MNAR mechanism resembles left-censoring. The Half-min method remains the best-performing imputer for this missingness mechanism overall except SOR metric.

\subsection{Partial SVD versus full SVD in SVT}

\label{app:partial-full-svd}

\ref{app:partial-full-svd} compares partial and full SVD implementations inside SVT. These results assess whether the choice of SVD solver materially affects imputation quality in the small-to-moderate dense matrices considered in this work. SVT\_partial uses a partial SVD implementation, whereas SVT and SVT\_OGpaper use full SVD.

\begin{table}[H]
\centering

\centering

\caption{NRMSE}
\label{tab:particalSVD_nrmse}

\begin{tabular}{lccc}
\toprule
Rate & SVT\_OGpaper & SVT\_partial & SVT \\
\midrule

10\% & 0.806 ± 0.011 & 0.804 ± 0.010 & \textbf{0.740 ± 0.050} \\
20\% & 0.920 ± 0.029 & 0.917 ± 0.028 & \textbf{0.814 ± 0.032} \\
30\% & 0.999 ± 0.012 & 0.995 ± 0.014 & \textbf{0.874 ± 0.015} \\
40\% & 1.104 ± 0.027 & 1.096 ± 0.028 & \textbf{0.938 ± 0.028} \\
50\% & 1.646 ± 0.067 & 1.646 ± 0.067 & \textbf{1.023 ± 0.023} \\
60\% & 1.579 ± 0.014 & 1.579 ± 0.014 & \textbf{1.140 ± 0.013} \\
70\% & 1.717 ± 0.035 & 1.717 ± 0.035 & \textbf{1.252 ± 0.035} \\
80\% & 2.028 ± 0.066 & 2.028 ± 0.066 & \textbf{1.360 ± 0.033} \\
90\% & 2.753 ± 0.069 & 2.753 ± 0.069 & \textbf{1.332 ± 0.028} \\
\bottomrule
\end{tabular}

\end{table}

\begin{table}[H]
\centering

\caption{PCA Procrustes distance}
\label{tab:particalSVD_pca_procrustes}

\begin{tabular}{lccc}
\toprule
Rate & SVT\_OGpaper & SVT\_partial & SVT \\
\midrule
10\% & 0.0028 ± 0.0006 & 0.0028 ± 0.0007 & \textbf{0.0019 ± 0.0007} \\
20\% & 0.0095 ± 0.0026 & 0.0096 ± 0.0027 & \textbf{0.0059 ± 0.0011} \\
30\% & 0.0216 ± 0.0029 & 0.0215 ± 0.0030 & \textbf{0.0119 ± 0.0014} \\
40\% & 0.0393 ± 0.0045 & 0.0385 ± 0.0042 & \textbf{0.0226 ± 0.0048} \\
50\% & 0.1020 ± 0.0178 & 0.1020 ± 0.0178 & \textbf{0.0488 ± 0.0123} \\
60\% & 0.1600 ± 0.0309 & 0.1600 ± 0.0309 & \textbf{0.0952 ± 0.0312} \\
70\% & 0.3231 ± 0.0238 & 0.3231 ± 0.0238 & \textbf{0.1715 ± 0.0343} \\
80\% & 0.7594 ± 0.0723 & 0.7594 ± 0.0723 & \textbf{0.4641 ± 0.1562} \\
90\% & 0.9626 ± 0.0199 & 0.9626 ± 0.0199 & \textbf{0.8080 ± 0.0604} \\
\bottomrule
\end{tabular}
\end{table}

\begin{table}[H]
\centering
\caption{Pearson correlation of \(-\log_{10}p\) }
\label{tab:pearson_logp_mcar}

\begin{tabular}{lccc}
\toprule
Rate & SVT\_OGpaper & SVT\_partial & SVT \\
\midrule
10\% & 0.982 ± 0.004 & 0.982 ± 0.004 & \textbf{0.995 ± 0.003} \\
20\% & 0.956 ± 0.012 & 0.956 ± 0.012 & \textbf{0.991 ± 0.002} \\
30\% & 0.916 ± 0.020 & 0.915 ± 0.020 & \textbf{0.980 ± 0.003} \\
40\% & 0.858 ± 0.015 & 0.858 ± 0.015 & \textbf{0.960 ± 0.006} \\
50\% & 0.770 ± 0.025 & 0.770 ± 0.025 & \textbf{0.938 ± 0.009} \\
60\% & 0.695 ± 0.039 & 0.695 ± 0.039 & \textbf{0.889 ± 0.018} \\
70\% & 0.606 ± 0.042 & 0.606 ± 0.042 & \textbf{0.838 ± 0.030} \\
80\% & 0.478 ± 0.111 & 0.478 ± 0.111 & \textbf{0.757 ± 0.036} \\
90\% & 0.300 ± 0.044 & 0.300 ± 0.044 & \textbf{0.631 ± 0.044} \\
\bottomrule
\end{tabular}
\end{table}

\begin{table}[H]

\centering
\caption{PLS Procrustes distance}
\label{tab:particalSVD_pls_procrustes}

\begin{tabular}{lccc}
\toprule
Rate & SVT\_OGpaper & SVT\_partial & SVT \\
\midrule
10\% & 0.0045 ± 0.0004 & 0.0045 ± 0.0005 & \textbf{0.0041 ± 0.0005} \\
20\% & 0.0124 ± 0.0004 & 0.0125 ± 0.0005 & \textbf{0.0094 ± 0.0012} \\
30\% & 0.0272 ± 0.0017 & 0.0271 ± 0.0015 & \textbf{0.0203 ± 0.0018} \\
40\% & 0.0562 ± 0.0051 & 0.0560 ± 0.0049 & \textbf{0.0403 ± 0.0022} \\
50\% & 0.1630 ± 0.0243 & 0.1630 ± 0.0243 & \textbf{0.0740 ± 0.0089} \\
60\% & 0.2020 ± 0.0048 & 0.2020 ± 0.0048 & \textbf{0.1272 ± 0.0116} \\
70\% & 0.3742 ± 0.0298 & 0.3742 ± 0.0298 & \textbf{0.2281 ± 0.0287} \\
80\% & 0.6624 ± 0.0201 & 0.6624 ± 0.0201 & \textbf{0.4307 ± 0.0637} \\
90\% & 0.8905 ± 0.0149 & 0.8905 ± 0.0149 & \textbf{0.6956 ± 0.0399} \\

\bottomrule
\end{tabular}
\end{table}

\begin{table}[H]
\centering
\caption{Runtime comparison between partial SVD and full SVD in SVT.}
\label{tab:svd-runtime}

\begin{tabular}{lcccc}
\toprule
Method & Mean runtime (s) & Std runtime (s) & Min runtime (s) & Max runtime (s)  \\
\midrule
SVT\_OGpaper & 5.465 ± 1.817 & 1.817 & 3.933 & 14.748  \\
SVT\_partial & 15.810 ± 5.069 & 5.069 & 8.206 & 29.109  \\
SVT & \textbf{5.292 ± 1.267} & \textbf{1.267} & \textbf{3.896} & \textbf{9.304}  \\
\bottomrule
\end{tabular}
\end{table}

\subsection{Ablation Study}
\label{app:ablation}
The ablation isolates the effects of the two modifications to the SVT: column-mean warm starting and the adaptive step-size rule. The plots compare these variants across missingness levels.

\begin{figure}[H]
    \centering
    \includegraphics[width=\textwidth]{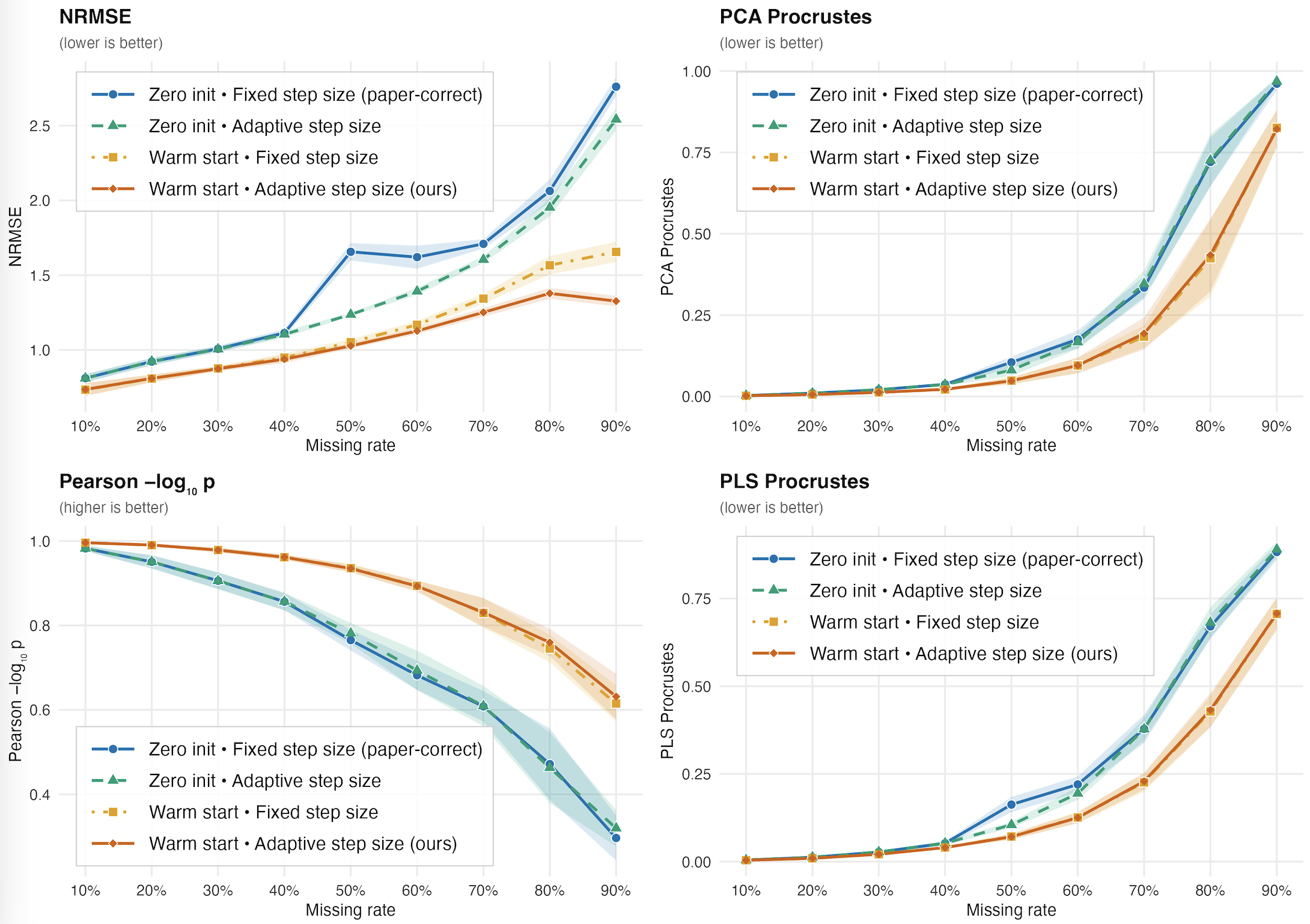}
    \caption{Ablation study on the metabolomics dataset under MCAR/MAR missingness.}
    \label{fig:large_image}
\end{figure}

\begin{table}[H]
\centering
\caption{NRMSE across missing rates}
\small
\begin{tabular}{lcccc}
\toprule
Rate & Zero · fixed & Zero · adaptive & Warm · fixed & Warm · adaptive \\
\midrule
10\% & 0.811 ± 0.027 & 0.811 ± 0.027 & \textbf{0.734 ± 0.042} & 0.736 ± 0.043 \\
20\% & 0.922 ± 0.024 & 0.923 ± 0.024 & \textbf{0.809 ± 0.026} & 0.810 ± 0.025 \\
30\% & 1.008 ± 0.020 & 1.005 ± 0.020 & 0.877 ± 0.017 & \textbf{0.874 ± 0.016} \\
40\% & 1.113 ± 0.025 & 1.104 ± 0.023 & 0.949 ± 0.027 & \textbf{0.938 ± 0.026} \\
50\% & 1.656 ± 0.058 & 1.237 ± 0.011 & 1.052 ± 0.020 & \textbf{1.027 ± 0.017} \\
60\% & 1.620 ± 0.077 & 1.392 ± 0.024 & 1.168 ± 0.025 & \textbf{1.127 ± 0.022} \\
70\% & 1.709 ± 0.032 & 1.604 ± 0.023 & 1.343 ± 0.033 & \textbf{1.251 ± 0.029} \\
80\% & 2.063 ± 0.078 & 1.952 ± 0.059 & 1.566 ± 0.061 & \textbf{1.378 ± 0.036} \\
90\% & 2.760 ± 0.063 & 2.542 ± 0.073 & 1.656 ± 0.069 & \textbf{1.327 ± 0.035} \\
\bottomrule
\end{tabular}
\end{table}

\begin{table}[H]
\centering
\caption{PCA Procrustes distance across missing rates}
\small
\begin{tabular}{lcccc}
\toprule
Rate & Zero · fixed & Zero · adaptive & Warm · fixed & Warm · adaptive \\
\midrule
10\% & 0.003 ± 0.001 & 0.003 ± 0.001 & \textbf{0.002 ± 0.001} & 0.002 ± 0.001 \\
20\% & 0.010 ± 0.002 & 0.010 ± 0.002 & 0.006 ± 0.001 & \textbf{0.006 ± 0.001} \\
30\% & 0.021 ± 0.004 & 0.021 ± 0.004 & 0.012 ± 0.003 & \textbf{0.012 ± 0.003} \\
40\% & 0.037 ± 0.005 & 0.037 ± 0.004 & \textbf{0.022 ± 0.003} & 0.022 ± 0.004 \\
50\% & 0.105 ± 0.017 & 0.081 ± 0.018 & 0.049 ± 0.010 & \textbf{0.048 ± 0.010} \\
60\% & 0.175 ± 0.027 & 0.167 ± 0.021 & \textbf{0.095 ± 0.023} & 0.095 ± 0.025 \\
70\% & 0.334 ± 0.035 & 0.346 ± 0.041 & \textbf{0.184 ± 0.034} & 0.194 ± 0.049 \\
80\% & 0.721 ± 0.075 & 0.726 ± 0.080 & \textbf{0.425 ± 0.118} & 0.435 ± 0.112 \\
90\% & 0.961 ± 0.020 & 0.968 ± 0.011 & 0.826 ± 0.056 & \textbf{0.822 ± 0.057} \\
\bottomrule
\end{tabular}
\end{table}

\begin{table}[H]
\centering
\caption{PLS Procrustes distance across missing rates}
\small
\begin{tabular}{lcccc}
\toprule
Rate & Zero · fixed & Zero · adaptive & Warm · fixed & Warm · adaptive \\
\midrule
10\% & 0.005 ± 0.000 & 0.005 ± 0.000 & \textbf{0.004 ± 0.000} & 0.004 ± 0.000 \\
20\% & 0.012 ± 0.001 & 0.012 ± 0.001 & 0.009 ± 0.001 & \textbf{0.009 ± 0.001} \\
30\% & 0.027 ± 0.003 & 0.027 ± 0.003 & 0.021 ± 0.003 & \textbf{0.021 ± 0.003} \\
40\% & 0.052 ± 0.006 & 0.052 ± 0.006 & 0.040 ± 0.002 & \textbf{0.040 ± 0.002} \\
50\% & 0.163 ± 0.022 & 0.105 ± 0.009 & 0.072 ± 0.009 & \textbf{0.071 ± 0.008} \\
60\% & 0.220 ± 0.023 & 0.194 ± 0.018 & 0.126 ± 0.017 & \textbf{0.125 ± 0.015} \\
70\% & 0.378 ± 0.040 & 0.378 ± 0.032 & \textbf{0.226 ± 0.026} & 0.229 ± 0.024 \\
80\% & 0.670 ± 0.034 & 0.681 ± 0.045 & \textbf{0.428 ± 0.045} & 0.432 ± 0.049 \\
90\% & 0.883 ± 0.022 & 0.890 ± 0.022 & \textbf{0.706 ± 0.048} & 0.707 ± 0.044 \\
\bottomrule
\end{tabular}
\end{table}

\begin{table}[H]
\centering
\caption{Pearson log-\(p\) correlation across missing rates}
\small
\begin{tabular}{lcccc}
\toprule
Rate & Zero · fixed & Zero · adaptive & Warm · fixed & Warm · adaptive \\
\midrule
10\% & 0.982 ± 0.006 & 0.982 ± 0.006 & 0.996 ± 0.002 & \textbf{0.996 ± 0.002} \\
20\% & 0.951 ± 0.016 & 0.951 ± 0.016 & 0.990 ± 0.003 & \textbf{0.990 ± 0.002} \\
30\% & 0.906 ± 0.020 & 0.906 ± 0.020 & \textbf{0.979 ± 0.006} & 0.978 ± 0.005 \\
40\% & 0.856 ± 0.020 & 0.857 ± 0.021 & 0.962 ± 0.006 & \textbf{0.962 ± 0.006} \\
50\% & 0.765 ± 0.024 & 0.781 ± 0.025 & \textbf{0.936 ± 0.010} & 0.935 ± 0.009 \\
60\% & 0.682 ± 0.034 & 0.693 ± 0.047 & \textbf{0.894 ± 0.012} & 0.893 ± 0.014 \\
70\% & 0.608 ± 0.036 & 0.609 ± 0.048 & 0.830 ± 0.036 & \textbf{0.831 ± 0.034} \\
80\% & 0.472 ± 0.085 & 0.463 ± 0.084 & 0.745 ± 0.032 & \textbf{0.760 ± 0.032} \\
90\% & 0.297 ± 0.053 & 0.320 ± 0.043 & 0.615 ± 0.043 & \textbf{0.632 ± 0.054} \\
\bottomrule
\end{tabular}
\end{table}

\begin{table}[H]
\centering

\caption{Runtime comparison of SVT ablation variants. Runtime is reported in seconds as mean ± standard deviation}
\label{tab:svt_ablation_runtime}

\begin{tabular}{lccc}
\toprule
Method & Mean ± Std. & Min. & Max. \\
\midrule
Zero init + fixed step size & 6.440 ± 1.616 & 4.385 & 12.716 \\
Zero init + adaptive step size & \textbf{6.147 ± 1.152} & 4.377 & 11.303 \\
Warm start + fixed step size & 6.378 ± 1.355 & 4.421 & 11.093 \\
Warm start + adaptive step size & 6.459 ± 1.669 & 4.229 & 12.460 \\
\bottomrule
\end{tabular}
\end{table}

\subsection{Runtime Decomposition and Scaling}
\label{app:runtime_scaling}
\begin{table}[H]
\centering

\caption{Stage-decomposed runtime on the metabolomics dataset ($198 \times 130$, MCAR). Times presented are means over three measured runs. Pre denotes Preprocess, SVD denotes one SVD. NF denotes NuclearForest, MF denotes MissForest, and SI denotes SoftImpute.}

\label{tab:runtime_decomposition}
\begin{tabular}{lccccccc}
\toprule
Missing
& \shortstack{Pre.\\(ms)}
& \shortstack{SVD\\(ms)}
& \shortstack{SVT\\(s)}
& \shortstack{RF\\(s)}
& \shortstack{NF\\(s)}
& \shortstack{MF\\(s)}
& \shortstack{SI\\(s)} \\
\midrule
20\%
& 0.18 ms
& 4.09 ms
& 8.97 s
& 10.31 s
& 19.28 s
& 105.50 s
& 0.48 s \\

50\%
& 0.22 ms
& 10.18 ms
& 10.28 s
& 7.75 s
& 18.03 s
& 76.64 s
& 0.80 s \\

80\%
& 0.19 ms
& 12.81 ms
& 8.22 s
& 5.08 s
& 13.30 s
& 50.84 s
& 0.46 s \\
\bottomrule
\end{tabular}
\end{table}

Mask generation is part of the evaluation protocol rather than the imputation
procedure and is therefore excluded from the reported imputation runtime.

We conduct a scaling study using rank-10 Gaussian matrices, with 5\%
noise and 30\% MCAR masking. Matrix size is reported as number of rows by number
of columns. With 130 columns, increasing the number of rows from 1000 to 2000
approximately doubled runtime from 66.0 s to 130.0 s. However, for a matrix with
2000 rows and 1000 columns, runtime reached 6885.7 s, with RF refinement
accounting for 87.0\% of total runtime. 

\begin{table}[H]
\centering

\caption{Runtime scaling study on rank-10 Gaussian matrices with 5\% noise and 30\% MCAR masking.}

\label{tab:runtime_scaling}
\begin{tabular}{lccccc}
\toprule
Matrix size
& Entries
& SVT stage
& RF stage
& Total
& RF share \\
\midrule
$(1000 \times 130)$
& 130,000
& 19.58 s
& 46.44 s
& 66.02 s
& 70.3\% \\

$(2000 \times 130)$
& 260,000
& 25.84 s
& 104.19 s
& 130.02 s
& 80.1\% \\

$(1000 \times 500)$
& 500,000
& 169.73 s
& 653.89 s
& 823.62 s
& 79.4\% \\

$(5000 \times 200)$
& 1,000,000
& 106.67 s
& 727.89 s
& 834.56 s
& 87.2\% \\

$(2000 \times 1000)$
& 2,000,000
& 893.83 s
& 5991.90 s
& 6885.72 s
& 87.0\% \\
\bottomrule
\end{tabular}
\end{table}

\section{Full numerical results}
\label{app:full-numerical-results}

\subsection{MCAR/MAR results on the metabolomics dataset}

\begin{table}[H]
\centering
\caption{NRMSE across missing rates}
\small
\begin{tabular}{lcccccc}
\toprule
Rate & MissForest & kNN & Mean & Median & Half-min & SVT\_OGpaper \\
\midrule
10\% & 0.618 ± 0.020 & 0.888 ± 0.012 & 1.006 ± 0.001 & 1.029 ± 0.002 & 2.389 ± 0.063 & 0.811 ± 0.027 \\
20\% & \textbf{0.646 ± 0.021} & 0.901 ± 0.009 & 1.005 ± 0.001 & 1.029 ± 0.002 & 2.353 ± 0.043 & 0.922 ± 0.024 \\
30\% & \textbf{0.680 ± 0.007} & 0.922 ± 0.005 & 1.006 ± 0.001 & 1.031 ± 0.003 & 2.340 ± 0.049 & 1.008 ± 0.020 \\
40\% & \textbf{0.719 ± 0.009} & 0.949 ± 0.005 & 1.007 ± 0.001 & 1.030 ± 0.002 & 2.323 ± 0.027 & 1.113 ± 0.025 \\
50\% & \textbf{0.760 ± 0.008} & 0.970 ± 0.008 & 1.007 ± 0.001 & 1.031 ± 0.003 & 2.311 ± 0.030 & 1.656 ± 0.058 \\
60\% & \textbf{0.812 ± 0.009} & 1.001 ± 0.010 & 1.008 ± 0.001 & 1.031 ± 0.002 & 2.314 ± 0.016 & 1.620 ± 0.077 \\
70\% & 0.875 ± 0.006 & 1.061 ± 0.008 & 1.010 ± 0.001 & 1.033 ± 0.002 & 2.291 ± 0.021 & 1.709 ± 0.032 \\
\bottomrule
\end{tabular}
\end{table}

\begin{table}[H]
\centering
\caption{PCA Procrustes distance across missing rates}
\small
\begin{tabular}{lcccccc}
\toprule
Rate & MissForest & kNN & Mean & Median & Half-min & SVT\_OGpaper \\
\midrule
10\% & 0.002 ± 0.001 & 0.005 ± 0.001 & 0.007 ± 0.001 & 0.007 ± 0.001 & 0.038 ± 0.007 & 0.003 ± 0.001 \\
20\% & 0.005 ± 0.001 & 0.016 ± 0.003 & 0.018 ± 0.003 & 0.018 ± 0.004 & 0.081 ± 0.013 & 0.010 ± 0.002 \\
30\% & \textbf{0.010 ± 0.002} & 0.037 ± 0.007 & 0.037 ± 0.006 & 0.039 ± 0.007 & 0.153 ± 0.043 & 0.021 ± 0.004 \\
40\% & 0.020 ± 0.003 & 0.062 ± 0.020 & 0.062 ± 0.016 & 0.064 ± 0.015 & 0.243 ± 0.041 & 0.037 ± 0.005 \\
50\% & \textbf{0.030 ± 0.004} & 0.101 ± 0.028 & 0.099 ± 0.035 & 0.101 ± 0.034 & 0.400 ± 0.077 & 0.105 ± 0.017 \\
60\% & \textbf{0.053 ± 0.010} & 0.164 ± 0.049 & 0.215 ± 0.066 & 0.221 ± 0.073 & 0.625 ± 0.072 & 0.175 ± 0.027 \\
70\% & \textbf{0.092 ± 0.013} & 0.299 ± 0.030 & 0.359 ± 0.115 & 0.365 ± 0.118 & 0.802 ± 0.036 & 0.334 ± 0.035 \\
\bottomrule
\end{tabular}
\end{table}

\begin{table}[H]
\centering
\caption{PLS Procrustes distance across missing rates}
\small
\begin{tabular}{lcccccc}
\toprule
Rate & MissForest & kNN & Mean & Median & Half-min & SVT\_OGpaper \\
\midrule
10\% & 0.003 ± 0.001 & 0.009 ± 0.001 & 0.009 ± 0.001 & 0.010 ± 0.001 & 0.057 ± 0.008 & 0.005 ± 0.000 \\
20\% & 0.008 ± 0.001 & 0.021 ± 0.001 & 0.021 ± 0.001 & 0.022 ± 0.001 & 0.116 ± 0.015 & 0.012 ± 0.001 \\
30\% & 0.016 ± 0.002 & 0.043 ± 0.005 & 0.040 ± 0.006 & 0.042 ± 0.006 & 0.193 ± 0.034 & 0.027 ± 0.003 \\
40\% & \textbf{0.027 ± 0.003} & 0.069 ± 0.009 & 0.065 ± 0.009 & 0.067 ± 0.009 & 0.274 ± 0.023 & 0.052 ± 0.006 \\
50\% & \textbf{0.043 ± 0.002} & 0.101 ± 0.009 & 0.097 ± 0.011 & 0.098 ± 0.012 & 0.371 ± 0.037 & 0.163 ± 0.022 \\
60\% & \textbf{0.073 ± 0.010} & 0.166 ± 0.023 & 0.177 ± 0.014 & 0.181 ± 0.017 & 0.537 ± 0.036 & 0.220 ± 0.023 \\
70\% & \textbf{0.122 ± 0.012} & 0.313 ± 0.031 & 0.291 ± 0.058 & 0.295 ± 0.060 & 0.675 ± 0.039 & 0.378 ± 0.040 \\
\bottomrule
\end{tabular}
\end{table}

\begin{table}[H]
\centering
\caption{Pearson log-\(p\) correlation across missing rates}
\small
\begin{tabular}{lcccccc}
\toprule
Rate & MissForest & kNN & Mean & Median & Half-min & SVT\_OGpaper \\
\midrule
10\% & 0.995 ± 0.003 & 0.993 ± 0.003 & 0.992 ± 0.003 & 0.991 ± 0.003 & 0.892 ± 0.018 & 0.982 ± 0.006 \\
20\% & 0.991 ± 0.002 & 0.985 ± 0.002 & 0.983 ± 0.003 & 0.981 ± 0.004 & 0.814 ± 0.037 & 0.951 ± 0.016 \\
30\% & \textbf{0.984 ± 0.004} & 0.975 ± 0.004 & 0.970 ± 0.006 & 0.966 ± 0.007 & 0.719 ± 0.047 & 0.906 ± 0.020 \\
40\% & \textbf{0.972 ± 0.007} & 0.958 ± 0.010 & 0.957 ± 0.008 & 0.952 ± 0.008 & 0.666 ± 0.080 & 0.856 ± 0.020 \\
50\% & \textbf{0.956 ± 0.008} & 0.942 ± 0.008 & 0.940 ± 0.010 & 0.931 ± 0.012 & 0.582 ± 0.072 & 0.765 ± 0.024 \\
60\% & \textbf{0.938 ± 0.010} & 0.916 ± 0.013 & 0.909 ± 0.017 & 0.900 ± 0.022 & 0.511 ± 0.088 & 0.682 ± 0.034 \\
70\% & \textbf{0.897 ± 0.016} & 0.868 ± 0.039 & 0.886 ± 0.030 & 0.867 ± 0.034 & 0.464 ± 0.115 & 0.608 ± 0.036 \\
\bottomrule
\end{tabular}
\end{table}

\subsection{MNAR results on the metabolomics dataset}

\begin{table}[H]
\centering
\caption{SOR across missing rates}
\footnotesize
\begin{tabular}{lcccccc}
\toprule
Rate & MissForest & kNN & Mean & Median & Half-min & SVT\_OGpaper \\
\midrule
10\% & 81.200 ± 7.815 & 111.400 ± 7.575 & 126.900 ± 5.990 & 97.600 ± 7.619 & 83.900 ± 12.991 & \textbf{49.900 ± 9.231} \\
20\% & 158.300 ± 9.031 & 216.800 ± 7.495 & 253.000 ± 13.241 & 182.700 ± 14.244 & 153.300 ± 18.252 & 119.900 ± 14.873 \\
30\% & 245.000 ± 19.067 & 328.300 ± 12.650 & 380.200 ± 11.013 & 271.700 ± 16.207 & 222.600 ± 27.240 & 162.300 ± 15.805 \\
40\% & 327.600 ± 17.494 & 434.600 ± 12.011 & 509.400 ± 13.032 & 369.000 ± 22.784 & 279.400 ± 19.873 & 215.300 ± 16.687 \\
50\% & 395.600 ± 27.293 & 538.500 ± 13.938 & 624.700 ± 10.605 & 441.900 ± 19.451 & 356.300 ± 29.208 & 306.200 ± 17.415 \\
60\% & 480.600 ± 23.114 & 654.100 ± 23.473 & 740.500 ± 16.939 & 518.100 ± 16.462 & 398.600 ± 23.177 & 345.900 ± 28.136 \\
70\% & 561.400 ± 30.837 & 761.700 ± 20.205 & 865.600 ± 12.703 & 591.100 ± 19.140 & 475.900 ± 41.757 & 414.200 ± 29.506 \\
\bottomrule
\end{tabular}
\end{table}
\begin{table}[H]
\centering
\caption{PCA Procrustes distance across missing rates}
\small
\begin{tabular}{lcccccc}
\toprule
Rate & MissForest & kNN & Mean & Median & Half-min & SVT\_OGpaper \\
\midrule
10\% & 0.001 ± 0.001 & 0.004 ± 0.001 & 0.005 ± 0.002 & 0.004 ± 0.002 & \textbf{0.001 ± 0.000} & 0.001 ± 0.001 \\
20\% & 0.003 ± 0.001 & 0.010 ± 0.005 & 0.012 ± 0.005 & 0.009 ± 0.005 & \textbf{0.002 ± 0.001} & 0.003 ± 0.002 \\
30\% & 0.011 ± 0.003 & 0.038 ± 0.018 & 0.037 ± 0.017 & 0.029 ± 0.012 & \textbf{0.004 ± 0.001} & 0.005 ± 0.002 \\
40\% & 0.021 ± 0.011 & 0.056 ± 0.021 & 0.066 ± 0.035 & 0.047 ± 0.020 & 0.006 ± 0.001 & \textbf{0.006 ± 0.002} \\
50\% & 0.032 ± 0.008 & 0.108 ± 0.036 & 0.122 ± 0.033 & 0.088 ± 0.023 & \textbf{0.009 ± 0.001} & 0.010 ± 0.003 \\
60\% & 0.050 ± 0.013 & 0.170 ± 0.061 & 0.191 ± 0.040 & 0.129 ± 0.027 & \textbf{0.010 ± 0.002} & 0.014 ± 0.003 \\
70\% & 0.078 ± 0.015 & 0.278 ± 0.092 & 0.300 ± 0.062 & 0.196 ± 0.042 & \textbf{0.017 ± 0.004} & 0.023 ± 0.010 \\
\bottomrule
\end{tabular}
\end{table}
\begin{table}[H]
\centering
\caption{PLS Procrustes distance across missing rates}
\small
\begin{tabular}{lcccccc}
\toprule
Rate & MissForest & kNN & Mean & Median & Half-min & SVT\_OGpaper \\
\midrule
10\% & 0.002 ± 0.001 & 0.005 ± 0.003 & 0.006 ± 0.003 & 0.005 ± 0.002 & \textbf{0.001 ± 0.000} & 0.002 ± 0.001 \\
20\% & 0.006 ± 0.002 & 0.013 ± 0.004 & 0.013 ± 0.004 & 0.009 ± 0.003 & \textbf{0.002 ± 0.000} & 0.004 ± 0.001 \\
30\% & 0.015 ± 0.005 & 0.041 ± 0.021 & 0.040 ± 0.016 & 0.031 ± 0.012 & \textbf{0.005 ± 0.001} & 0.009 ± 0.004 \\
40\% & 0.024 ± 0.008 & 0.061 ± 0.029 & 0.067 ± 0.025 & 0.051 ± 0.019 & \textbf{0.007 ± 0.002} & 0.009 ± 0.002 \\
50\% & 0.039 ± 0.006 & 0.090 ± 0.028 & 0.107 ± 0.020 & 0.080 ± 0.015 & \textbf{0.010 ± 0.001} & 0.018 ± 0.003 \\
60\% & 0.057 ± 0.015 & 0.128 ± 0.052 & 0.155 ± 0.040 & 0.114 ± 0.029 & \textbf{0.012 ± 0.003} & 0.024 ± 0.005 \\
70\% & 0.087 ± 0.015 & 0.209 ± 0.070 & 0.248 ± 0.041 & 0.177 ± 0.033 & \textbf{0.020 ± 0.004} & 0.036 ± 0.010 \\
\bottomrule
\end{tabular}
\end{table}
\begin{table}[H]
\centering
\caption{Pearson log-\(p\) correlation across missing rates}
\small
\begin{tabular}{lcccccc}
\toprule
Rate & MissForest & kNN & Mean & Median & Half-min & SVT\_OGpaper \\
\midrule
10\% & 0.991 ± 0.007 & 0.962 ± 0.024 & 0.942 ± 0.036 & 0.956 ± 0.027 & 0.995 ± 0.003 & \textbf{0.997 ± 0.001} \\
20\% & 0.983 ± 0.006 & 0.948 ± 0.022 & 0.924 ± 0.030 & 0.946 ± 0.024 & \textbf{0.994 ± 0.002} & 0.992 ± 0.004 \\
30\% & 0.964 ± 0.015 & 0.890 ± 0.048 & 0.849 ± 0.042 & 0.888 ± 0.032 & \textbf{0.989 ± 0.004} & 0.988 ± 0.006 \\
40\% & 0.950 ± 0.012 & 0.842 ± 0.046 & 0.787 ± 0.037 & 0.838 ± 0.031 & \textbf{0.987 ± 0.003} & 0.986 ± 0.003 \\
50\% & 0.948 ± 0.013 & 0.847 ± 0.035 & 0.788 ± 0.028 & 0.842 ± 0.023 & \textbf{0.986 ± 0.004} & 0.975 ± 0.006 \\
60\% & 0.921 ± 0.016 & 0.798 ± 0.079 & 0.734 ± 0.061 & 0.803 ± 0.042 & \textbf{0.982 ± 0.005} & 0.967 ± 0.009 \\
70\% & 0.921 ± 0.013 & 0.774 ± 0.065 & 0.704 ± 0.046 & 0.793 ± 0.031 & \textbf{0.976 ± 0.004} & 0.959 ± 0.007 \\
\bottomrule
\end{tabular}
\end{table}
\subsection{MCAR results on the housing dataset}

\begin{table}[H]
\centering
\caption{NRMSE across missing rates}
\small
\begin{tabular}{lcccccc}
\toprule
Rate & MissForest & kNN & Mean & Median & Half-min & SVT\_OGpaper \\
\midrule
10\% & 0.842 ± 0.032 & 1.049 ± 0.029 & 1.003 ± 0.001 & 1.147 ± 0.012 & 1.829 ± 0.057 & 1.451 ± 0.037 \\
20\% & 0.892 ± 0.020 & 1.072 ± 0.026 & 1.003 ± 0.001 & 1.141 ± 0.011 & 1.814 ± 0.032 & 1.460 ± 0.038 \\
30\% & 0.950 ± 0.015 & 1.060 ± 0.024 & 1.002 ± 0.001 & 1.134 ± 0.005 & 1.818 ± 0.020 & 1.520 ± 0.021 \\
40\% & 1.020 ± 0.023 & 1.011 ± 0.010 & 1.003 ± 0.001 & 1.138 ± 0.005 & 1.807 ± 0.056 & 1.401 ± 0.016 \\
50\% & 1.072 ± 0.021 & 1.006 ± 0.006 & 1.002 ± 0.001 & 1.137 ± 0.006 & 1.799 ± 0.039 & 1.689 ± 0.036 \\
60\% & 1.124 ± 0.031 & 1.018 ± 0.005 & 1.003 ± 0.001 & 1.140 ± 0.004 & 1.759 ± 0.047 & 1.610 ± 0.031 \\
70\% & 1.147 ± 0.024 & 1.027 ± 0.006 & \textbf{1.004 ± 0.001} & 1.140 ± 0.004 & 1.789 ± 0.045 & 1.650 ± 0.025 \\
\bottomrule
\end{tabular}
\end{table}

\begin{table}[H]
\centering
\caption{PFC across missing rates}
\small
\begin{tabular}{lcccccc}
\toprule
Rate & MissForest & kNN & Mean & Median & Half-min & SVT\_OGpaper \\
\midrule
10\% & 0.145 ± 0.017 & 0.271 ± 0.027 & 0.253 ± 0.014 & 0.253 ± 0.014 & 0.338 ± 0.021 & 0.271 ± 0.016 \\
20\% & 0.173 ± 0.007 & 0.254 ± 0.004 & 0.253 ± 0.012 & 0.253 ± 0.012 & 0.333 ± 0.010 & 0.274 ± 0.010 \\
30\% & 0.194 ± 0.010 & 0.243 ± 0.018 & 0.247 ± 0.007 & 0.247 ± 0.007 & 0.331 ± 0.010 & 0.281 ± 0.006 \\
40\% & 0.227 ± 0.010 & 0.236 ± 0.008 & 0.247 ± 0.006 & 0.247 ± 0.006 & 0.332 ± 0.009 & 0.245 ± 0.018 \\
50\% & 0.252 ± 0.014 & 0.243 ± 0.005 & 0.250 ± 0.006 & 0.250 ± 0.006 & 0.333 ± 0.009 & 0.295 ± 0.013 \\
60\% & 0.270 ± 0.017 & 0.254 ± 0.005 & 0.253 ± 0.004 & 0.253 ± 0.004 & 0.333 ± 0.006 & 0.289 ± 0.007 \\
70\% & 0.284 ± 0.017 & 0.259 ± 0.006 & \textbf{0.251 ± 0.004} & \textbf{0.251 ± 0.004} & 0.331 ± 0.004 & 0.296 ± 0.006 \\
\bottomrule
\end{tabular}
\end{table}

\begin{table}[H]
\centering
\caption{Gower's distance across missing rates}
\small
\begin{tabular}{lcccccc}
\toprule
Rate & MissForest & kNN & Mean & Median & Half-min & SVT\_OGpaper \\
\midrule
10\% & 0.130 ± 0.008 & 0.228 ± 0.019 & 0.215 ± 0.009 & 0.212 ± 0.009 & 0.342 ± 0.018 & 0.265 ± 0.009 \\
20\% & 0.147 ± 0.005 & 0.216 ± 0.008 & 0.215 ± 0.008 & 0.209 ± 0.009 & 0.337 ± 0.008 & 0.269 ± 0.012 \\
30\% & 0.161 ± 0.006 & 0.208 ± 0.012 & 0.212 ± 0.006 & 0.205 ± 0.006 & 0.333 ± 0.008 & 0.280 ± 0.007 \\
40\% & 0.185 ± 0.007 & 0.202 ± 0.005 & 0.211 ± 0.004 & 0.205 ± 0.005 & 0.335 ± 0.008 & 0.246 ± 0.012 \\
50\% & 0.204 ± 0.008 & 0.207 ± 0.003 & 0.214 ± 0.005 & 0.208 ± 0.005 & 0.335 ± 0.008 & 0.304 ± 0.012 \\
60\% & 0.221 ± 0.010 & 0.216 ± 0.004 & 0.217 ± 0.003 & 0.211 ± 0.003 & 0.333 ± 0.006 & 0.297 ± 0.006 \\
70\% & 0.232 ± 0.009 & 0.220 ± 0.004 & 0.215 ± 0.002 & \textbf{0.209 ± 0.002} & 0.334 ± 0.004 & 0.305 ± 0.005 \\
\bottomrule
\end{tabular}
\end{table}

\begin{table}[H]
\centering
\caption{Downstream predictive \(R^2\) degradation across missing rates}
\small
\begin{tabular}{lcccccc}
\toprule
Rate & MissForest & kNN & Mean & Median & Half-min & SVT\_OGpaper \\
\midrule
10\% & 0.023 ± 0.008 & 0.038 ± 0.017 & 0.100 ± 0.021 & 0.113 ± 0.020 & 0.330 ± 0.044 & 0.339 ± 0.045 \\
20\% & 0.034 ± 0.012 & 0.114 ± 0.029 & 0.192 ± 0.026 & 0.217 ± 0.030 & 0.479 ± 0.026 & 0.449 ± 0.034 \\
30\% & 0.050 ± 0.023 & 0.173 ± 0.033 & 0.276 ± 0.033 & 0.324 ± 0.041 & 0.571 ± 0.056 & 0.499 ± 0.056 \\
40\% & 0.065 ± 0.017 & 0.217 ± 0.020 & 0.356 ± 0.040 & 0.414 ± 0.052 & 0.630 ± 0.041 & 0.760 ± 0.096 \\
50\% & \textbf{0.059 ± 0.032} & 0.325 ± 0.048 & 0.448 ± 0.035 & 0.559 ± 0.038 & 0.672 ± 0.034 & 1.180 ± 0.118 \\
60\% & \textbf{0.060 ± 0.034} & 0.480 ± 0.054 & 0.562 ± 0.066 & 0.697 ± 0.109 & 0.689 ± 0.055 & 1.152 ± 0.096 \\
70\% & \textbf{0.069 ± 0.058} & 0.522 ± 0.084 & 0.650 ± 0.102 & 0.831 ± 0.148 & 0.705 ± 0.061 & 1.223 ± 0.147 \\
\bottomrule
\end{tabular}
\end{table}

\begin{table}[H]
\centering
\caption{PCA Procrustes distance across missing rates}
\small
\begin{tabular}{lcccccc}
\toprule
Rate & MissForest & kNN & Mean & Median & Half-min & SVT\_OGpaper \\
\midrule
10\% & 0.024 ± 0.004 & 0.045 ± 0.006 & 0.043 ± 0.003 & 0.060 ± 0.005 & 0.122 ± 0.007 & 0.112 ± 0.011 \\
20\% & 0.067 ± 0.008 & 0.096 ± 0.007 & 0.091 ± 0.006 & 0.121 ± 0.008 & 0.227 ± 0.020 & 0.225 ± 0.020 \\
30\% & 0.124 ± 0.010 & 0.154 ± 0.020 & 0.142 ± 0.005 & 0.185 ± 0.009 & 0.335 ± 0.015 & 0.361 ± 0.022 \\
40\% & 0.208 ± 0.011 & 0.198 ± 0.011 & 0.198 ± 0.007 & 0.257 ± 0.008 & 0.436 ± 0.017 & 0.367 ± 0.019 \\
50\% & 0.319 ± 0.021 & 0.284 ± 0.011 & 0.282 ± 0.014 & 0.354 ± 0.016 & 0.559 ± 0.018 & 0.492 ± 0.025 \\
60\% & 0.441 ± 0.015 & 0.393 ± 0.018 & 0.378 ± 0.011 & 0.455 ± 0.013 & 0.657 ± 0.024 & 0.584 ± 0.020 \\
70\% & 0.562 ± 0.024 & 0.503 ± 0.013 & 0.476 ± 0.011 & 0.552 ± 0.014 & 0.761 ± 0.022 & 0.712 ± 0.025 \\
\bottomrule
\end{tabular}
\end{table}

\begin{table}[H]
\centering
\caption{PLS Procrustes distance across missing rates}
\small
\begin{tabular}{lcccccc}
\toprule
Rate & MissForest & kNN & Mean & Median & Half-min & SVT\_OGpaper \\
\midrule
10\% & 0.027 ± 0.004 & 0.047 ± 0.007 & 0.046 ± 0.004 & 0.062 ± 0.006 & 0.140 ± 0.011 & 0.134 ± 0.014 \\
20\% & 0.069 ± 0.007 & 0.100 ± 0.006 & 0.096 ± 0.006 & 0.125 ± 0.009 & 0.254 ± 0.022 & 0.257 ± 0.024 \\
30\% & 0.124 ± 0.007 & 0.157 ± 0.017 & 0.148 ± 0.007 & 0.187 ± 0.010 & 0.364 ± 0.020 & 0.393 ± 0.026 \\
40\% & 0.207 ± 0.012 & 0.203 ± 0.013 & 0.206 ± 0.009 & 0.260 ± 0.010 & 0.466 ± 0.020 & 0.403 ± 0.020 \\
50\% & 0.314 ± 0.020 & 0.290 ± 0.012 & 0.292 ± 0.014 & 0.358 ± 0.019 & 0.589 ± 0.020 & 0.526 ± 0.024 \\
60\% & 0.439 ± 0.021 & 0.396 ± 0.020 & 0.385 ± 0.016 & 0.455 ± 0.016 & 0.680 ± 0.025 & 0.610 ± 0.023 \\
70\% & 0.562 ± 0.029 & 0.510 ± 0.011 & 0.487 ± 0.009 & 0.556 ± 0.013 & 0.778 ± 0.023 & 0.731 ± 0.025 \\
\bottomrule
\end{tabular}
\end{table}

\begin{table}[H]
\centering
\caption{Pearson log-\(p\) correlation across missing rates}
\small
\begin{tabular}{lcccccc}
\toprule
Rate & MissForest & kNN & Mean & Median & Half-min & SVT\_OGpaper \\
\midrule
10\% & 0.993 ± 0.004 & 0.993 ± 0.004 & 0.983 ± 0.008 & 0.995 ± 0.002 & 0.892 ± 0.031 & 0.761 ± 0.033 \\
20\% & 0.970 ± 0.008 & 0.971 ± 0.015 & 0.952 ± 0.016 & 0.984 ± 0.008 & 0.854 ± 0.035 & 0.631 ± 0.045 \\
30\% & 0.926 ± 0.015 & 0.966 ± 0.015 & 0.927 ± 0.024 & \textbf{0.973 ± 0.014} & 0.763 ± 0.040 & 0.486 ± 0.047 \\
40\% & 0.885 ± 0.032 & 0.959 ± 0.018 & 0.920 ± 0.023 & \textbf{0.970 ± 0.015} & 0.661 ± 0.056 & 0.294 ± 0.053 \\
50\% & 0.846 ± 0.023 & 0.927 ± 0.021 & 0.890 ± 0.029 & \textbf{0.951 ± 0.015} & 0.675 ± 0.082 & 0.522 ± 0.113 \\
60\% & 0.794 ± 0.032 & 0.919 ± 0.029 & 0.875 ± 0.036 & \textbf{0.941 ± 0.029} & 0.600 ± 0.064 & 0.325 ± 0.119 \\
70\% & 0.743 ± 0.040 & 0.895 ± 0.023 & 0.870 ± 0.038 & \textbf{0.933 ± 0.024} & 0.496 ± 0.134 & 0.183 ± 0.094 \\
\bottomrule
\end{tabular}
\end{table}

\subsection{MAR results on the housing dataset}

\begin{table}[H]
\centering
\caption{NRMSE across missing rates}
\small
\begin{tabular}{lcccccc}
\toprule
Rate & MissForest & kNN & Mean & Median & Half-min & SVT\_OGpaper \\
\midrule
10\% & 0.852 ± 0.033 & 1.063 ± 0.061 & 1.009 ± 0.005 & 1.119 ± 0.012 & 1.808 ± 0.066 & 1.415 ± 0.046 \\
20\% & 0.920 ± 0.029 & 1.060 ± 0.021 & 1.012 ± 0.003 & 1.132 ± 0.005 & 1.779 ± 0.050 & 1.458 ± 0.044 \\
30\% & 0.979 ± 0.020 & 1.068 ± 0.019 & 1.018 ± 0.003 & 1.124 ± 0.008 & 1.774 ± 0.039 & 1.534 ± 0.029 \\
40\% & 1.032 ± 0.027 & 1.055 ± 0.010 & 1.023 ± 0.003 & 1.124 ± 0.007 & 1.733 ± 0.040 & 1.371 ± 0.030 \\
50\% & 1.084 ± 0.019 & 1.068 ± 0.004 & 1.033 ± 0.004 & 1.127 ± 0.002 & 1.702 ± 0.031 & 2.007 ± 0.065 \\
60\% & 1.107 ± 0.026 & 1.099 ± 0.008 & 1.046 ± 0.008 & 1.132 ± 0.006 & 1.687 ± 0.020 & 1.965 ± 0.056 \\
70\% & 1.163 ± 0.069 & 1.121 ± 0.006 & 1.071 ± 0.006 & 1.156 ± 0.032 & 1.689 ± 0.045 & 1.894 ± 0.039 \\
\bottomrule
\end{tabular}
\end{table}

\begin{table}[H]
\centering
\caption{PFC across missing rates}
\small
\begin{tabular}{lcccccc}
\toprule
Rate & MissForest & kNN & Mean & Median & Half-min & SVT\_OGpaper \\
\midrule
10\% & 0.134 ± 0.019 & 0.256 ± 0.030 & 0.229 ± 0.022 & 0.229 ± 0.022 & 0.301 ± 0.018 & 0.243 ± 0.017 \\
20\% & 0.178 ± 0.012 & 0.244 ± 0.024 & 0.247 ± 0.013 & 0.247 ± 0.013 & 0.318 ± 0.010 & 0.276 ± 0.014 \\
30\% & 0.204 ± 0.016 & 0.236 ± 0.011 & 0.235 ± 0.013 & 0.235 ± 0.013 & 0.310 ± 0.016 & 0.270 ± 0.015 \\
40\% & 0.237 ± 0.016 & 0.242 ± 0.006 & 0.240 ± 0.006 & 0.240 ± 0.006 & 0.310 ± 0.007 & 0.241 ± 0.011 \\
50\% & 0.258 ± 0.012 & 0.247 ± 0.005 & 0.244 ± 0.002 & 0.244 ± 0.002 & 0.315 ± 0.003 & 0.291 ± 0.009 \\
60\% & 0.266 ± 0.010 & 0.259 ± 0.005 & 0.242 ± 0.005 & 0.242 ± 0.005 & 0.313 ± 0.006 & 0.293 ± 0.008 \\
70\% & 0.284 ± 0.022 & 0.267 ± 0.008 & 0.256 ± 0.024 & 0.256 ± 0.024 & 0.313 ± 0.005 & 0.295 ± 0.009 \\
\bottomrule
\end{tabular}
\end{table}

\begin{table}[H]
\centering
\caption{Gower's distance across missing rates}
\small
\begin{tabular}{lcccccc}
\toprule
Rate & MissForest & kNN & Mean & Median & Half-min & SVT\_OGpaper \\
\midrule
10\% & 0.128 ± 0.012 & 0.216 ± 0.019 & 0.203 ± 0.017 & 0.194 ± 0.018 & 0.323 ± 0.017 & 0.250 ± 0.012 \\
20\% & 0.151 ± 0.008 & 0.214 ± 0.012 & 0.220 ± 0.011 & 0.212 ± 0.010 & 0.341 ± 0.009 & 0.278 ± 0.012 \\
30\% & 0.168 ± 0.012 & 0.214 ± 0.010 & 0.211 ± 0.010 & 0.202 ± 0.010 & 0.334 ± 0.010 & 0.280 ± 0.012 \\
40\% & 0.188 ± 0.008 & 0.223 ± 0.003 & 0.219 ± 0.004 & 0.213 ± 0.004 & 0.337 ± 0.005 & 0.248 ± 0.009 \\
50\% & 0.205 ± 0.007 & 0.226 ± 0.005 & 0.222 ± 0.003 & 0.217 ± 0.004 & 0.338 ± 0.005 & 0.356 ± 0.005 \\
60\% & 0.215 ± 0.009 & 0.238 ± 0.003 & 0.225 ± 0.004 & 0.219 ± 0.004 & 0.338 ± 0.005 & 0.347 ± 0.009 \\
70\% & 0.231 ± 0.018 & 0.244 ± 0.003 & 0.233 ± 0.012 & 0.226 ± 0.012 & 0.337 ± 0.004 & 0.342 ± 0.009 \\
\bottomrule
\end{tabular}
\end{table}

\begin{table}[H]
\centering
\caption{Downstream predictive \(R^2\) degradation across missing rates}
\small
\begin{tabular}{lcccccc}
\toprule
Rate & MissForest & kNN & Mean & Median & Half-min & SVT\_OGpaper \\
\midrule
10\% & 0.021 ± 0.008 & 0.026 ± 0.024 & 0.068 ± 0.022 & 0.057 ± 0.023 & 0.206 ± 0.040 & 0.215 ± 0.042 \\
20\% & 0.031 ± 0.015 & 0.066 ± 0.019 & 0.143 ± 0.019 & 0.129 ± 0.022 & 0.318 ± 0.050 & 0.320 ± 0.058 \\
30\% & 0.028 ± 0.019 & 0.144 ± 0.034 & 0.225 ± 0.030 & 0.210 ± 0.039 & 0.374 ± 0.055 & 0.353 ± 0.062 \\
40\% & 0.060 ± 0.010 & 0.228 ± 0.068 & 0.275 ± 0.051 & 0.264 ± 0.052 & 0.385 ± 0.047 & 0.513 ± 0.081 \\
50\% & 0.070 ± 0.031 & 0.323 ± 0.050 & 0.341 ± 0.036 & 0.354 ± 0.056 & 0.394 ± 0.044 & 1.264 ± 0.138 \\
60\% & \textbf{0.057 ± 0.027} & 0.453 ± 0.075 & 0.421 ± 0.045 & 0.469 ± 0.049 & 0.418 ± 0.059 & 1.339 ± 0.156 \\
70\% & \textbf{0.049 ± 0.018} & 0.431 ± 0.055 & 0.428 ± 0.044 & 0.651 ± 0.169 & 0.398 ± 0.031 & 1.220 ± 0.116 \\
\bottomrule
\end{tabular}
\end{table}

\begin{table}[H]
\centering
\caption{PCA Procrustes distance across missing rates}
\small
\begin{tabular}{lcccccc}
\toprule
Rate & MissForest & kNN & Mean & Median & Half-min & SVT\_OGpaper \\
\midrule
10\% & 0.022 ± 0.004 & 0.038 ± 0.004 & 0.038 ± 0.003 & 0.049 ± 0.005 & 0.091 ± 0.006 & 0.081 ± 0.005 \\
20\% & 0.070 ± 0.009 & 0.085 ± 0.006 & 0.086 ± 0.004 & 0.109 ± 0.004 & 0.179 ± 0.010 & 0.186 ± 0.010 \\
30\% & 0.131 ± 0.008 & 0.139 ± 0.005 & 0.141 ± 0.005 & 0.162 ± 0.006 & 0.243 ± 0.011 & 0.284 ± 0.018 \\
40\% & 0.213 ± 0.029 & 0.224 ± 0.015 & 0.214 ± 0.010 & 0.233 ± 0.011 & 0.310 ± 0.009 & 0.289 ± 0.012 \\
50\% & 0.312 ± 0.021 & 0.345 ± 0.012 & 0.307 ± 0.011 & 0.320 ± 0.010 & 0.365 ± 0.015 & 0.458 ± 0.021 \\
60\% & 0.404 ± 0.027 & 0.516 ± 0.021 & 0.421 ± 0.018 & 0.400 ± 0.010 & 0.410 ± 0.009 & 0.494 ± 0.024 \\
70\% & 0.552 ± 0.109 & 0.672 ± 0.022 & 0.568 ± 0.012 & 0.516 ± 0.027 & 0.468 ± 0.012 & 0.521 ± 0.023 \\
\bottomrule
\end{tabular}
\end{table}
\begin{table}[H]
\centering
\caption{PLS Procrustes distance across missing rates}
\small
\begin{tabular}{lcccccc}
\toprule
Rate & MissForest & kNN & Mean & Median & Half-min & SVT\_OGpaper \\
\midrule
10\% & 0.023 ± 0.004 & 0.039 ± 0.003 & 0.039 ± 0.003 & 0.049 ± 0.005 & 0.101 ± 0.006 & 0.095 ± 0.006 \\
20\% & 0.071 ± 0.009 & 0.089 ± 0.007 & 0.091 ± 0.004 & 0.113 ± 0.005 & 0.195 ± 0.011 & 0.208 ± 0.012 \\
30\% & 0.130 ± 0.008 & 0.145 ± 0.009 & 0.147 ± 0.007 & 0.166 ± 0.009 & 0.263 ± 0.018 & 0.307 ± 0.023 \\
40\% & 0.206 ± 0.025 & 0.235 ± 0.018 & 0.224 ± 0.013 & 0.239 ± 0.016 & 0.332 ± 0.012 & 0.318 ± 0.012 \\
50\% & 0.313 ± 0.021 & 0.358 ± 0.016 & 0.319 ± 0.013 & 0.327 ± 0.011 & 0.385 ± 0.017 & 0.465 ± 0.021 \\
60\% & 0.398 ± 0.033 & 0.537 ± 0.024 & 0.435 ± 0.021 & 0.409 ± 0.015 & 0.430 ± 0.010 & 0.499 ± 0.023 \\
70\% & 0.555 ± 0.111 & 0.695 ± 0.021 & 0.582 ± 0.012 & 0.521 ± 0.030 & 0.481 ± 0.012 & 0.521 ± 0.018 \\
\bottomrule
\end{tabular}
\end{table}

\begin{table}[H]
\centering
\caption{Pearson log-\(p\) correlation across missing rates}
\small
\begin{tabular}{lcccccc}
\toprule
Rate & MissForest & kNN & Mean & Median & Half-min & SVT\_OGpaper \\
\midrule
10\% & 0.994 ± 0.003 & 0.995 ± 0.004 & 0.968 ± 0.010 & 0.987 ± 0.006 & 0.886 ± 0.021 & 0.808 ± 0.026 \\
20\% & 0.972 ± 0.011 & \textbf{0.988 ± 0.005} & 0.918 ± 0.025 & 0.952 ± 0.017 & 0.808 ± 0.034 & 0.682 ± 0.047 \\
30\% & 0.940 ± 0.023 & \textbf{0.970 ± 0.015} & 0.889 ± 0.023 & 0.921 ± 0.022 & 0.698 ± 0.028 & 0.551 ± 0.042 \\
40\% & 0.879 ± 0.021 & \textbf{0.931 ± 0.020} & 0.832 ± 0.027 & 0.854 ± 0.029 & 0.638 ± 0.042 & 0.488 ± 0.045 \\
50\% & 0.810 ± 0.021 & 0.827 ± 0.037 & 0.783 ± 0.042 & 0.767 ± 0.041 & 0.533 ± 0.041 & 0.629 ± 0.045 \\
60\% & 0.751 ± 0.040 & 0.707 ± 0.063 & 0.739 ± 0.036 & 0.674 ± 0.046 & 0.473 ± 0.056 & 0.541 ± 0.064 \\
70\% & 0.593 ± 0.144 & 0.589 ± 0.083 & 0.648 ± 0.023 & 0.518 ± 0.044 & 0.403 ± 0.057 & 0.440 ± 0.058 \\
\bottomrule
\end{tabular}
\end{table}

\subsection{Runtime comparison}

\begin{table}[H]
\centering
\caption{Runtime comparison of imputation methods on the metabolomics dataset under MCAR/MAR missingness. Runtime is reported in seconds.}
\small
\begin{tabular}{lcccc}
\toprule
Method & Mean (s) & Std (s) & Min (s) & Max (s) \\
\midrule
MissForest & 77.787 & 21.426 & 49.290 & 116.269 \\
kNN & 0.040 & 0.009 & 0.021 & 0.063 \\
Mean & 0.001 & 0.001 & 0.000 & 0.004 \\
Median & 0.002 & 0.001 & 0.001 & 0.006 \\
Half-min & \textbf{0.000} & 0.000 & 0.000 & 0.002 \\
SVT\_OGpaper & 5.889 & 1.261 & 4.180 & 12.519 \\
NF\_OGpaper & 13.581 & 2.343 & 9.826 & 20.209 \\
SoftImpute & 0.384 & 0.144 & 0.111 & 0.895 \\
SoftForest & 8.171 & 1.931 & 5.264 & 11.753 \\
SVT & 5.660 & 0.956 & 4.408 & 9.867 \\
NuclearForest & 13.389 & 2.365 & 9.772 & 21.601 \\
\bottomrule
\end{tabular}
\end{table}

\begin{table}[H]
\centering
\caption{Runtime comparison of imputation methods on the metabolomics dataset under MNAR missingness. Runtime is reported in seconds.}
\small
\begin{tabular}{lccccc}
\toprule
Method & Mean (s) & Std (s) & Min (s) & Max (s) \\
\midrule
MissForest & 89.438 & 22.845 & 11.206 & 111.187  \\
kNN & 0.025 & 0.012 & 0.006 & 0.060  \\
Mean & 0.001 & 0.001 & 0.000 & 0.005  \\
Median & 0.002 & 0.001 & 0.001 & 0.005  \\
Half-min & \textbf{0.000} & 0.000 & 0.000 & 0.001  \\
SVT\_OGpaper & 5.611 & 0.769 & 4.911 & 9.745  \\
NF\_OGpaper & 9.751 & 2.228 & 5.806 & 13.438  \\
SoftImpute & 0.125 & 0.085 & 0.013 & 0.562  \\
SoftForest & 4.344 & 2.137 & 0.946 & 7.712  \\
SVT & 5.446 & 0.448 & 4.851 & 7.636  \\
NuclearForest & 9.797 & 2.381 & 5.936 & 16.075 \\
\bottomrule
\end{tabular}
\end{table}

\begin{table}[H]
\centering
\caption{Runtime comparison of imputation methods on the housing dataset under MCAR missingness. Runtime is reported in seconds.}
\small
\begin{tabular}{lcccc}
\toprule
Method & Mean (s) & Std (s) & Min (s) & Max (s) \\
\midrule
MissForest & 5.934 & 0.187 & 5.516 & 6.454 \\
kNN & 0.034 & 0.024 & 0.010 & 0.144 \\
Mean & 0.001 & 0.000 & 0.000 & 0.004 \\
Median & 0.001 & 0.000 & 0.001 & 0.003 \\
Half-min & \textbf{0.000} & 0.000 & 0.000 & 0.003 \\
SVT\_OGpaper & 0.119 & 0.096 & 0.002 & 0.285 \\
NF\_OGpaper & 0.750 & 0.062 & 0.650 & 0.926 \\
SoftImpute & 0.001 & 0.000 & 0.000 & 0.001 \\
SoftForest & 0.664 & 0.045 & 0.573 & 0.845 \\
SVT & 0.003 & 0.001 & 0.002 & 0.007 \\
NuclearForest & 0.659 & 0.047 & 0.576 & 0.882 \\
\bottomrule
\end{tabular}
\end{table}

\begin{table}[H]
\centering
\caption{Runtime comparison of imputation methods on the housing dataset under MAR missingness. Runtime is reported in seconds.}
\small
\begin{tabular}{lcccc}
\toprule
Method & Mean (s) & Std (s) & Min (s) & Max (s) \\
\midrule
MissForest & 5.973 & 0.198 & 5.719 & 6.754 \\
kNN & 0.030 & 0.018 & 0.008 & 0.140 \\
Mean & 0.001 & 0.001 & 0.000 & 0.007 \\
Median & 0.001 & 0.000 & 0.001 & 0.003 \\
Half-min & \textbf{0.000} & 0.000 & 0.000 & 0.001 \\
SVT\_OGpaper & 0.120 & 0.096 & 0.002 & 0.305 \\
NF\_OGpaper & 0.761 & 0.075 & 0.647 & 1.121 \\
SoftImpute & 0.001 & 0.000 & 0.000 & 0.002 \\
SoftForest & 0.661 & 0.039 & 0.601 & 0.820 \\
SVT & 0.003 & 0.001 & 0.002 & 0.007 \\
NuclearForest & 0.661 & 0.034 & 0.604 & 0.766 \\
\bottomrule
\end{tabular}
\end{table}

\end{document}